\documentclass{article}

\usepackage[nonatbib, final]{neurips_2024}

\usepackage[utf8]{inputenc}
\usepackage[T1]{fontenc}
\usepackage{hyperref}
\usepackage{url}
\usepackage{booktabs}
\usepackage{amsfonts}
\usepackage{nicefrac}
\usepackage{microtype}
\usepackage{xcolor}
\usepackage[nottoc,notlot,notlof]{tocbibind}
\usepackage{multirow}
\usepackage{thmtools}
\usepackage{mathtools}
\usepackage{amsmath}
\usepackage{amssymb}
\usepackage{amsthm}
\usepackage{thmtools, thm-restate}
\usepackage{dsfont}
\usepackage{algorithm2e}
\usepackage{enumitem}
\usepackage{upgreek}
\usepackage{booktabs}
\usepackage{multirow}
\usepackage{apptools}
\usepackage{setspace}
\usepackage{pgfplots}
\usepackage{tikz}
\usepackage{wrapfig}
\usepackage{setspace}

\AtAppendix{\counterwithin{lemma}{section}}

\definecolor{primary}{HTML}{e41a1c}
\definecolor{secondary}{HTML}{4daf4a}

\hypersetup{
    colorlinks,
    linkcolor=primary,
    citecolor=secondary,
    urlcolor=secondary
}

\mathtoolsset{showonlyrefs}
\allowdisplaybreaks

\DeclareMathOperator*{\argmin}{arg \, min}

\let\originalleft\left
\let\originalright\right
\renewcommand{\left}{\mathopen{}\mathclose\bgroup\originalleft}
\renewcommand{\right}{\aftergroup\egroup\originalright}

\newcommand*{\eqdef}{\ensuremath{\overset{\mathclap{\text{\fontsize{4pt}{4pt}\selectfont def}}}{=}}}
\newcommand*{\vb}[1]{\boldsymbol{#1}}
\newcommand*{\transpose}{\top}
\newcommand*{\expected}{\mathbb{E}}

\newcommand*{\ev}[1]{\mathbb{E}\left[#1 \right] }

\newcommand*{\grad}[2]{\vb{g}_{#2}({#1})}
\newcommand*{\F}[1]{f({#1})}
\newcommand*{\Fi}[2]{f_{#2}({#1})}
\newcommand*{\Fstar}{f_{\star}}
\newcommand*{\Finfimum}{f_{\mathrm{inf}}}

\newcommand*{\fake}[1]{\mathrel{\phantom{#1}}}
\newcommand{\overbar}[1]{\mkern 1.5mu\overline{\mkern-1.5mu#1\mkern-1.5mu}\mkern 1.5mu}

\makeatletter
\def\thm@space@setup{%
  \thm@preskip=0.25cm
  \thm@postskip=0.1cm
}
\makeatother

\newtheorem{lemma}{Lemma}
\newtheorem{definition}{Definition}
\newtheorem{assumption}{Assumption}

\title{
Aiding Global Convergence in Federated Learning via Local Perturbation and Mutual Similarity Information}

\author{%
  Emanuel Buttaci\thanks{This work is based on the master's thesis of the author.} \\
  Politecnico di Torino, Italy \\
  \texttt{emanuel.buttaci@studenti.polito.it}
  \And
  Giuseppe Carlo Calafiore \\
  Politecnico di Torino, Italy \\
  \texttt{giuseppe.calafiore@polito.it}
}

\begin{document}

\addtocontents{toc}{\protect\setcounter{tocdepth}{0}}

\maketitle

\setstretch{1.25}

\begin{abstract}
Federated learning has emerged in the last decade as a distributed optimization paradigm due to the rapidly increasing number of portable devices able to support the heavy computational needs related to the training of machine learning models. 
Federated learning utilizes gradient-based optimization to minimize a loss objective shared across participating agents.
To the best of our knowledge, the literature mostly lacks elegant solutions that naturally harness the reciprocal statistical similarity between clients to redesign the optimization procedure.
To address this gap, by conceiving the federated network as a similarity graph, we propose a novel modified framework wherein each client locally performs a perturbed gradient step leveraging prior information about other statistically affine clients.
We theoretically prove that our procedure, due to a suitably introduced adaptation in the update rule, achieves a quantifiable speedup concerning the exponential contraction factor in the strongly convex case compared with popular algorithms \textsc{FedAvg} and \textsc{FedProx}, here analyzed as baselines.
Lastly, we legitimize our conclusions through experimental results on the CIFAR10 and FEMNIST datasets, where we show that our algorithm speeds convergence up to a margin of 30 global rounds compared with \textsc{FedAvg} while modestly improving generalization on unseen data in heterogeneous settings.
\end{abstract}

\section{Introduction}

The large availability of computationally capable devices, such as smartphones or data centers, made the advent of distributed machine learning possible. Federated learning is 
an optimization paradigm enabling a distributed approach to artificial intelligence where a central server coordinates the training of a shared statistical model, such as a convolutional neural network classifier, across several clients whilst preserving their privacy constraints. 
No samples from clients are shared with the server and only locally-computed updates are communicated at each round to optimize the centralized model. 

Formally, a central server coordinates the training of a model $\vb{w} \in \mathbb{R}^D$ across $C$ clients, where each holds its dataset $\mathcal{D}_i$ with samples $\xi_1^i, \xi_2^i, \ldots, \xi_{N_i}^i$ drawn from distribution $\mathcal{P}_i(\vb{x}, \vb{y})$. Given the sample loss $\ell(\vb{w}; \xi)$, each client $i$ solves the local problem $\min_{\vb{w}} \, \mathbb{E}_{\xi \, \sim \, \mathcal{P}_i} [ \ell (\vb{w}; \xi) ]$, expressed by the empirical risk minimization objective $\Fi{\vb{w}}{i}$. 
The server globally minimizes the aggregation function
\begin{align}
    \F{\vb{w}} \eqdef \sum_{i = 1}^{C} p_i \Fi{\vb{w}}{i} \, \textrm{ where each } \, \Fi{\vb{w}}{i} \eqdef \frac{1}{N_i} \sum_{n = 1}^{N_i} \ell (\vb{w}; \xi_n^i) \, \textrm{ and } \, \sum_{i = 1}^C p_i = 1.
\end{align}
The optimization is performed during $T$ training rounds, and participating clients receive the shared model $\overbar{\vb{w}}_{t, 0}$ and complete $E$ steps of stochastic gradient descent on their respective datasets in each round. The local update rule for client $i$ at round $t \ge 0$ and step $k \in \{ \, 0, \ldots, E - 1\, \}$ is $\vb{w}_{t, k + 1}^i = \vb{w}_{t, k}^i - \gamma_{t}  \grad{\vb{w}_{t, k}^i}{i} \label{eq:fed_avg_update_rule}$ where $\grad{\vb{w}_{t, k}^i}{i} = \nabla \Fi{{\vb{w}_{t, k}^i; \xi_{t, k}^i}}{i}$ is the stochastic gradient and $\gamma_{t} $ is the step size. From \cite{local_sgd_average_iterate_sequence}, it is common to study the convergence using the average iterative rule
\begin{align}
    \overbar{\vb{w}}_{t, k + 1} = \overbar{\vb{w}}_{t, k} - \gamma_{t} \sum_{i \, \in \, \mathcal{S}_t} p_{i} \grad{\vb{w}_{t, k}^i}{i}.
    \label{eq:fed_avg_update_rule_averaged}
\end{align} 
The average iterate is defined as $\overbar{\vb{w}} \eqdef \sum_{i = 1}^C p_{i} \vb{w}_i$. The set $\mathcal{S}_t$ denotes the clients participating in server round $t$. Globally, this results in inexact stochastic gradient descent, because client $i$ computes the stochastic gradient in the local iterate $\vb{w}_{t, k}^i$ rather than $\overbar{\vb{w}}_{t, k}$. We remind that $\vb{w}_{t, 0}^i = \overbar{\vb{w}}_{t, 0}$ holds for every client at the beginning of any round $t$. In practice, the server collects the local updates and then aggregates them to update the centralized model at the end of every global round.

\paragraph{Contributions}
\label{sec:contribution}

Motivated by the scarcity of literature that questions whether convergence and generalization might benefit from integrating mutual client similarity information in the standard federated training process, we simultaneously answer the following unrelated topic questions.
\begin{enumerate}[leftmargin=*]
\item \textit{Can we improve generalization by leveraging local information about statistically similar clients?}
\item \textit{Can we devise a locally perturbed scheme to directly influence the global convergence speed?}
\end{enumerate}
\begin{wrapfigure}{R}{0.49\textwidth}
    \vspace*{-3mm}
    \centering
\resizebox{0.48 \textwidth}{!}{
\begin{tikzpicture}
    \definecolor{primary}{HTML}{e41a1c} 
    \definecolor{secondary}{HTML}{377eb8}
    
    \usetikzlibrary{shapes.geometric, arrows, automata, positioning, fadings, through}
    \tikzstyle{c1} = [rectangle, rounded corners, minimum width = 1.5cm, minimum height = 1.5cm, text centered, draw = black, line width = 0.5mm, left color = primary!20, right color = secondary!20]

    \tikzstyle{c2} = [rectangle, rounded corners, minimum width = 1.5cm, minimum height = 1.5cm, text centered, draw = black, line width = 0.5mm, fill = primary!20]

    \tikzstyle{c3} = [rectangle, rounded corners, minimum width = 1.5cm, minimum height = 1.5cm, text centered, draw = black, line width = 0.5mm, fill = secondary!20]

    \draw[line width = 0.5mm] (0, 0) -- (7, 0) node [midway, below, sloped, yshift = -2] {$p_{12} \propto [\boldsymbol{A}]_{12} \gg 0$};

    \draw[line width = 0.5mm, draw opacity = 0.3] (0, 0.75) edge node[sloped, anchor = center, above, yshift = 2] (arrow13) {$p_{13} \propto [\boldsymbol{A}]_{13} > 0$} (2.75, 4.5);

    \draw[line width = 0.5mm, draw opacity = 0.10] (7, 0.75) edge node[sloped, anchor = center, above, yshift = 2] (arrow23) {$p_{23} \propto [\boldsymbol{A}]_{23} \approx 0$} (4.25, 4.5);

    \draw[dashed, line width = 0.5mm, ->] (1, 0.75) -- (2.8, 1.5) node [midway, above, sloped] {$\vb{w}_{t - 1, E}^1$};

    \draw[dashed, line width = 0.5mm, <-] (1, 0.5) -- (2.8, 1.25) node [midway, below, sloped] (arrow12) {$\vb{u}_{t}^1$, $\overline{\vb{w}}_{t, 0}$};

    \draw[dashed, line width = 0.5mm, <-] (4.2, 1.5) -- (6, 0.75) node [midway, above, sloped] {$\vb{w}_{t - 1, E}^2$};

    \draw[dashed, line width = 0.5mm, ->] (4.2, 1.25) -- (6, 0.5) node [midway, below, sloped] (arrow12) {$\vb{u}_{t}^2$, $\overline{\vb{w}}_{t, 0}$};

    \draw[dashed, line width = 0.5mm, <-] (3.375, 2.2) -- (3.375, 3.5) node [midway, above, sloped, yshift = 2] {$\vb{w}_{t - 1, E}^3$};

    \draw[dashed, line width = 0.5mm, ->] (3.625, 2.2) -- (3.625, 3.5) node [midway, below, sloped, yshift = -2] {$\vb{u}_{t}^3$, $\overline{\vb{w}}_{t, 0}$};

    \node[draw] at (0, 0) (client1) [c1] {};
    \node[draw] at (7, 0) (client2) [c2] {};
    \node[draw] at (3.5, 4.5) (client3) [c3] {};
    \draw[line width = 0.5mm, fill = white] (3.5, 1.5) circle (0.5) {};
    \draw[fill = black] (3.5, 1.5) circle (0.25) {};

    \node at (-1.25, 0.5) {$\mathcal{D}_1$};
    \node at (8.25, 0.5) {$\mathcal{D}_2$};
    \node at (4.75, 5) {$\mathcal{D}_3$};

    \draw[fill = secondary, rounded corners = 2] (3, 4) rectangle (3.35, 4.35);
    \draw[fill = secondary, rounded corners = 2] (3, 4.6) rectangle (3.35, 4.95);
    \draw[fill = secondary, rounded corners = 2] (3.65, 4) rectangle (4, 4.35);
    \draw[fill = secondary, rounded corners = 2] (3.65, 4.6) rectangle (4, 4.95);

    \draw[fill = secondary, rounded corners = 2] (0.1, 0.1) rectangle (0.45, 0.45);
    \draw[fill = primary] (-0.3, 0.25) circle (0.175);
    \draw[fill = primary] (-0.3, -0.25) circle (0.175);
    \draw[fill = primary] (0.3, -0.25) circle (0.175);

    \draw[fill = primary] (6.7, 0.25) circle (0.175);
    \draw[fill = primary] (7.3, 0.25) circle (0.175);
    \draw[fill = primary] (6.7, -0.25) circle (0.175);
    \draw[fill = primary] (7.3, -0.25) circle (0.175);
\end{tikzpicture}
}
    \label{fig:sample-federated-network-representation}
    \caption{Illustration of our framework with three clients having binary class samples, and the server computing $\overline{\vb{w}}_{t, 0}$ and each $\vb{u}_t^i$ at every round $t$.}
    \vspace*{-5mm}
\end{wrapfigure}
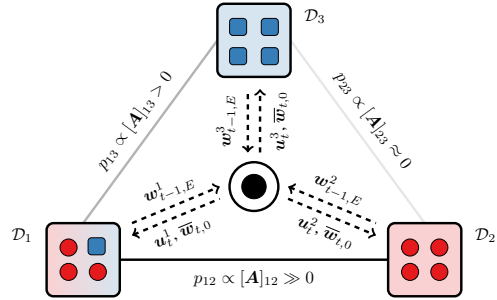
We take inspiration from the concept of node similarity from spectral graph theory to rethink a federated network as a weighted graph based on the statistical similarity between clients' data. Hence, our goal is to devise an elementary procedure that capitalizes on the aforementioned information to carry out a perturbed optimization step that could improve the overall convergence speed as well as the generalization capacity of the learned model.
Our main contributions are summarized as follows.

To establish a baseline, we provide an alternative analysis (see \ref{sec:analysis_fedprox_baseline}) of known algorithm \textsc{FedProx} (\cite{fedprox}) as a generalization of \textsc{FedAvg} with proximal updates. In the strongly convex scenario, we show that \textsc{FedProx} has a slower contraction rate at iteration $t$ with respect to \textsc{FedAvg} for comparable step sizes and any $\alpha > 0$, namely the weight of the proximal term. For nonconvex losses, \textsc{FedProx} presents two extra terms compared with \textsc{FedAvg} of order $\mathcal{O}(G^2 /\sqrt{T})$ and $\mathcal{O}(G^2 / T^{3/2})$, where the former may be a reason of deceleration (see \ref{sec:nonconvex_scenario_discussion}).

In Section \ref{sec:our_algorithm_framework_definition}, we conceive our algorithm based on making inexact gradient steps. Our idea of inexactness is connected with choosing an ideal perturbed iterate that locally minimizes the distance from statistically similar clients (neighboring nodes of the graph representation of the network as in Figure \ref{fig:sample-federated-network-representation}). For strongly convex losses (see \ref{sec:contraction_rate_discussion}), its contraction term has provably faster decay than \textsc{FedAvg} (and \textsc{FedProx}) for comparable step sizes and any choice of our perturbation parameter $\beta \in (0, 1)$.
Our algorithm has similar asymptotic terms to $\textsc{FedAvg}$ even without convexity guarantees, and thus performs comparably. However, regardless of convexity, we show that our framework presents a limitation due to unavoidable term $\mathcal{O}(G^2 / \beta^2)$ rapidly growing as we decrease $\beta$ (see \ref{sec:convergence_analysis_discussion}). 

To support our claims on the robustness of our method, we offer empirical evidence on widely used benchmark datasets (see \ref{sec:experimentation}): FEMNIST from \cite{leaf} and CIFAR10 from \cite{cifar10}. We compare its performance with \textsc{FedAvg} as the specific case of \textsc{FedProx} when $\alpha = 0$.

\section{Related works}
\label{sec:related_works}

The pioneering algorithm \textsc{FedAvg} was introduced by \cite{fedaveraging} to classify written digits using a CNN (\cite{cnn}) and predict the next word in a sentence with an LSTM (\cite{lstm}). It was shown to largely outperform single epoch decentralized stochastic gradient descent (\textsc{FedSGD}). \textsc{FedAvg}, also known as \textsc{LocalSGD} and initially proposed by \cite{parallelized_sgd} as SGD (\cite{sgd}) with periodic model averaging, has since been studied to assess its behavior under different assumptions. 
Most works, including ours, assume the local objective losses to be smooth and analyze the convergence of \textsc{FedAvg} leveraging distinct prior information. For instance, \cite{fed_averaging_niid_convergence, local_sgd_average_iterate_sequence} derived the convergence rates of \textsc{FedAvg} for strongly convex losses and diminishing step sizes while bounding the norm of gradients. \cite{fed_averaging_niid_convergence} used the gap between the global optimal value $\Fstar$ and the expectation of local optimal values $\expected \, \Fi{\vb{w}_{\star}^i}{i}$ as heterogeneity measure. On the other hand, \cite{local_sgd_tighter_theory, scaffold} extended the known analysis to generally convex and nonconvex functions. 
The diverse nature of clients' losses, exacerbated by the statistical heterogeneity, induces locally-optimized models to diverge from the global model. This phenomenon, called client drift (or local divergence), has been tackled by multiple studies. In this respect, \cite{fedprox} introduced a generalization of \textsc{FedAvg} named \textsc{FedProx}, which exploits a proximal term in clients' objectives to constrain the local iterates to stay close to the global one as optimization takes place. Due to its elementary nature, we choose \textsc{FedProx} as our baseline algorithm for future comparisons. Furthermore, \cite{scaffold} suggested SCAFFOLD, a framework that employs control variates to reduce the extent of client drift.

\paragraph{Personalized federated learning}

In personalized federated learning, instead of learning a single global model that does not account for the different distributions from which data samples are drawn, each client learns a tailored model that better fits the nature of its dataset. \cite{L2GD} initially proposed algorithm L2GD that mixes local and global models while reducing the overall communication. \cite{maml} suggested \textsc{Per-FedAvg} as a personalized version of \textsc{FedAvg} that easily alters the global model to suit local datasets, and \cite{personalized_moreau_envelopes} conceived \textsc{pFedMe}, which regularizes local losses using Moreau envelopes. In particular, our work shares a similarity with algorithm \textsc{FedU}, introduced by \cite{fed_u_laplacian_regularization} and proposed as a generalization of several works made in the direction of personalized federated learning. Specifically, \textsc{FedU} performs a regularization step at the end of each round that uses a generic laplacian graph representation of the federated network to smooth local iterates. 
On the other hand, by relying on a specific graph representation based on statistical clients' similarities (Section \ref{sec:our_algorithm_framework_definition}), we disrupt the local optimization structure by exploiting a perturbed gradient update whose argument minimizes the local variability relative to the iterates of neighboring clients, and we show that this strategy expedites convergence by some margin.
Lastly, \textsc{FedU} is entirely decentralized, and clients independently optimize their local objectives. In contrast, our approach is fully centralized.

\paragraph{Multi-task federated learning} 

Multi-task federated learning aims to learn multiple models concomitantly where each corresponds to a task (node in a network). This strategy leverages existing relationships between the nodes of a network, such as statistical affinity or availability. For instance, \cite{mocha} introduced the multi-task framework \textsc{MOCHA} that accounts for issues related to communication expense or partial participation. \cite{multi_task_mixture_distribution} proposed an EM-based algorithm that assumes that local samples belong to a mixture of unknown data distributions. In relation to multi-task learning, our novel framework explores how convergence and generalization benefit from defining node relationships as mutual statistical similarities based on an inherent graph structure.

\section{Our local variability-based algorithmic framework}
\label{sec:our_algorithm_framework_definition}

Before introducing the idea behind our algorithm, we present a simple model of a federated network as a similarity graph, where nodes are represented by clients and each edge quantifies the statistical similarity between their local datasets.

\subsection{A graph-based model to quantify inter-client similarity}
\label{sec:graph_based_model}

The concept of a relationship between clients heavily translates to the idea of a discrepancy between their data distributions. Since we cannot directly work on probabilistic data distributions, we must infer such statistical properties from their datasets. In our scenario, we allow each client $i$ to share an initial message $\vb{m}_i$, coded as a unitary vector, that embodies information about his local data distribution $\mathcal{P}_i(\vb{x}, \vb{y})$, computed through his dataset $\mathcal{D}_i$.
\begin{definition}
\label{eq:client_misalignment}   
Let $i, n \in \mathcal{C}$ be two clients with datasets $\mathcal{D}_i, \mathcal{D}_n$, respectively. The client misalignment $\mathrm{mis}: \mathbb{R}^D \times \mathbb{R}^D \rightarrow [0, 1]$ is defined as the distance measure $\mathrm{mis}(i, n) \eqdef (1 - \vb{m}_i^{\transpose} \vb{m}_n) / 2 
$ where unitary vectors $\vb{m}_i, \vb{m}_n$ are messages initially shared from clients $i$ and $n$.
\end{definition}
Such a metric is devised to gauge inter-client dissimilarity efficiently. Given that messages $\vb{m}_i, \vb{m}_n$ are unitary vectors, computing client misalignment is equivalent to the squared norm of the distance $\|\vb{m}_i - \vb{m}_n\|^2 / 4$. which is strictly related to the cosine dissimilarity between the given messages. 

However, which message should clients exchange? For simplicity, we consider the first principal component of each client's local dataset as the message that encodes information about the statistical distribution of the data. In this regard, on each local dataset, we compute an unscaled and uncentered version of the principal component analysis to capture information about the scale and intercept of the first principal component, which represents the direction of the highest variation in the data.

The construction of a similarity graph naturally depends on the idea of the relationship between the nodes. Common approaches, such as \cite{laplacian_eigenmaps}, involve building the $\varepsilon$-similarities graph assigning connection weights which are $0/1$ or computed through a Gaussian kernel. Our formulation of the graph representation of a federated network is expressed through its adjacency matrix.
\begin{definition}
\label{eq:federated_network_adjacency_matrix}
Given a federated network with clients $\mathcal{C} = \{ \, 1, \ldots, C \, \}$, we define the adjacency matrix $\vb{A}$ associated to its (undirected and complete) graph representation as $[\vb{A}]_{in} \eqdef -\ln(\mathrm{mis}(i, n)) \cdot \mathds{1}_{i \ne n}$.
\end{definition}
The chosen definition of connection strength between nodes, namely clients, leverages the logarithmic scale to have a better distinction between couples of very similar clients and very dissimilar ones.

\subsection{Intuitive algorithm idea}

The key principle of our algorithm is to complete an inexact local optimization round on each client. This translates to making perturbed moves at each local iteration $k$. Specifically, the update step is performed by taking into consideration the minimization of the variability against other neighboring clients. The idea of neighborhood refers to those clients who share a remarkable statistical similarity. When updating the current local iterate, each client $i$ computes the local stochastic gradient in a shifted coordinate $\widetilde{\vb{w}}_{t, k}^i$. This encodes information about the current local iterate $\vb{w}_{t, k}^i$ and the latest progress made by each neighboring client $n \in \mathcal{N}_i$ in the previous $(t - 1)$-th round.

\subsection{Formulation}
\label{sec:our_algorithm_formulation}

Formally, we are interested in implementing an inexact local update rule of the form
\begin{align}
    \vb{w}_{t, k + 1}^i &= \vb{w}_{t, k}^i - \gamma_{t} \grad{\widetilde{\vb{w}}_{t, k}^i}{i}
    \label{ex:our_algorithm_local_update_rule}
\end{align}
at step $k$ of round $t$. Variable $\widetilde{\vb{w}}_{t, k}^i$ is the perturbed iterate in which the stochastic gradient is evaluated. This forces the update that minimizes $\Fi{\vb{w}}{i}$ to be executed in an imprecise direction in relation to the starting point $\vb{w}_{t, k}^i$. We carry out this investigation to comprehend whether this would benefit or harm the convergence to a stationary point. Particularly, the nature of $\widetilde{\vb{w}}_{t, k}^i$ is fundamental and determines the properties of our algorithm. In this regard, we consider the graph-based representation of a federated network from paragraph \ref{sec:graph_based_model}, and we wish to choose $\widetilde{\vb{w}}_{t, k}^i$ as the solution of the problem
\begin{align}
    \min_{\vb{w} \, \in \, \mathbb{R}^D} \, \left\{ \, \frac{\beta}{2}\left\|\vb{w} - \vb{w}_{t, k}^i\right\|^2 + \frac{1 - \beta}{2 p_i} \sum_{n \, \in \, \mathcal{N}_i} p_{in} \left\| \vb{w} - \vb{w}_{t - 1, E}^n \right\|^2 \, \right\}
    \label{sec:our_algorithm_w_tilde_minimization_problem}
\end{align}
where $\vb{w}_{t - 1, E}^n$ is the last iterate of neighboring client $n$ from the previous round. The solution to this formulation minimizes the distance from the exact iterate $\vb{w}_{t, k}^i$ as well as the local variation, namely the sum of squared deviations from the models of neighbors. In this respect, each iterate $\vb{w}_{t - 1, E}^n$ is weighted according to the similarity measure $p_{in} \propto  [\vb{A}]_{in}$ between $i$ and $n$. However, such weights are normalized by $p_i \eqdef \sum_{n \, \in \, \mathcal{N}_i} p_{in}$, which directly corresponds to the concept of degree of client $i$, when interpreted as a graph-node. Interestingly, we also choose $p_i$ as the weighting factor for client $i$ during aggregation. This favors clients that have a higher degree, namely those who share many statistically similar neighbors. Additionally, those who are generally dissimilar and are not representative of the majority will be given less importance. As a solution to problem \eqref{sec:our_algorithm_w_tilde_minimization_problem}, we obtain
\begin{align}
    \widetilde{\vb{w}}_{t, k}^i = \beta \vb{w}_{t, k}^i + (1 - \beta) \vb{u}_{t}^i  \quad \mathrm{where} \quad \vb{u}_{t}^i \eqdef \frac{1}{p_i} \sum_{n \, \in \, \mathcal{N}_i}^C p_{in} \vb{w}_{t - 1, E}^n. 
    \label{eq:our_algorithm_perturbed_iterate_definition}
\end{align}
The central server sends $\vb{u}_{t}^i$ as well as $\vb{w}_{t, 0}^i = \overbar{\vb{w}}_{t, 0} = \overbar{\vb{w}}_{t - 1, E}$ to client $i$ at the beginning of global round $t$. Concerning \eqref{eq:our_algorithm_perturbed_iterate_definition}, while $\vb{u}_{t}^i$ remains fixed across the round, iterate $\widetilde{\vb{w}}_{t, k}^i$ is updated at every local step $k$ due to its dependence on $\vb{w}_{t, k}^i$. In another perspective, $\widetilde{\vb{w}}_{t, k}^i$ is the mean between the current iterate and the weighted average of the latest updates from neighbors. Clearly, by setting $\beta = 1$, we recover the iterative rule of \textsc{FedAvg}. However, by picking $\beta < 1$, we purposely contaminate the progress made in each step in order to be inexact.

\IncMargin{2em}
\begin{algorithm}[htb]
\label{alg:our_algorithm}

\DontPrintSemicolon

\SetArgSty{textnormal}
\SetNlSty{textbf}{\color{gray}}{}
\newcommand*{\mycommentfont}[1]{\textcolor{gray}{#1}}
\SetCommentSty{mycommentfont}
\SetKwComment{Comment}{$\triangleright$\ }{}
\renewcommand\AlCapFnt{\normalfont}
\renewcommand\AlCapNameFnt{\normalfont}

\definecolor{highlighted}{HTML}{377eb8}

\begingroup
\color{black}
\BlankLine
$\overbar{\vb{w}}_{0, 0} \gets \text{random weights initialization}$ \Comment*[r]{global model}
\ForEach{client $i \in \mathcal{C}$ \bf{in parallel}}{
    \begingroup
    \color{highlighted}
    $\vb{m}_i \gets \text{statistically-significant message as \ref{eq:client_misalignment}}$ \Comment*[r]{client sends his message vector}
    $\vb{u}_{0}^i \gets \overbar{\vb{w}}_{0, 0}$ \Comment*[r]{initialization of local averages}
    \endgroup
}
$[\vb{A}]_{in} \gets -\ln(\mathrm{mis}(i, n)) \cdot \mathds{1}_{i \ne n} \quad (\forall i, n \in \mathcal{C})$ \Comment*[r]{messages-based adjacency matrix}
$p_{in} \gets [\vb{A}]_{in} / (\vb{1}^{\transpose} \vb{A} \vb{1}) \quad (\forall i, n \in \mathcal{C})$ \Comment*[r]{mutual similarity weight initialization}
$p_{i} \gets \sum_{n \, \in \, \mathcal{N}_i} p_{in} \quad (\forall i \in \mathcal{C})$ \Comment*[r]{aggregation weight initialization}
\ForEach{\text{round} $t = 0$ \bf{to} $T - 1$}{
    $\mathcal{S}_t \gets \text{random sample of } M \text{ clients from } \mathcal{C}$ \Comment*[r]{clients selection}
    \ForEach{client $i \in \mathcal{S}_t$ \bf{in parallel}}{
        \begingroup
        \color{highlighted}
        $\vb{w}_{t, 0}^i \gets \overbar{\vb{w}}_{t, 0}$ \Comment*[r]{client receives model}
        $\left\{ \xi_{t, 0}^i, \ldots, \xi_{t, E - 1}^i \right\} \gets \text{partition } \mathcal{D}_i \text{ in } E \text{ mini-batches}$ \;
        \ForEach{local step $k = 1$ \bf{to} $E$}{
            $\widetilde{\vb{w}}_{t, k - 1}^i \gets \beta \vb{w}_{t, k - 1}^i + (1 - \beta) \vb{u}_{t}^i$ \Comment*[r]{perturbed iterate}
            $\grad{\widetilde{\vb{w}}_{t, k - 1}^i}{i} \gets \nabla \Fi{\widetilde{\vb{w}}_{t, k - 1}^i;\xi_{t, k - 1}^i}{i}$ \Comment*[r]{perturbed gradient}
            $\vb{w}_{t, k}^i \gets \vb{w}_{t, k - 1}^i - \gamma_{t} \grad{\widetilde{\vb{w}}_{t, k - 1}^i}{i}$ \Comment*[r]{local optimization}
        }
        \endgroup
    }
    $\vb{u}_{t + 1}^i \gets p_i^{-1} \sum_{n \, \in \, \mathcal{N}_i} p_{in} \vb{w}_{t, E}^n \quad (\forall i \in \mathcal{C})$ \Comment*[r]{server updates local averages}
    $\overbar{\vb{w}}_{t + 1, 0}\gets \sum_{i \, \in \, \mathcal{S}_t} p_i \vb{w}_{t, E}^i$ \Comment*[r]{global aggregation}
}
\endgroup
\BlankLine
\BlankLine
\caption{Pseudocode of our algorithm. Colored instructions are executed on each client.}
\end{algorithm}
\DecMargin{2em}

Without loss of generality, we relax expression \eqref{eq:our_algorithm_perturbed_iterate_definition} so that each neighborhood contains all the clients, that is $\mathcal{N}_i \equiv \mathcal{C}$. Then, we set $p_{in} = 0$ whenever $i$ equals $n$ or $i, n$ are not statistically affine. Lemma \ref{lemma:our_algorithm_average_perturbed_iterate} characterizes the average perturbed iterate in relation to current iterate $\overbar{\vb{w}}_{t, k}$ and initial one $\overbar{\vb{w}}_{t, 0}$.
\begin{restatable}{lemma}{OurAlgorithmAveragePerturbedIterate}
    \label{lemma:our_algorithm_average_perturbed_iterate}
    \label{eq:alternative_characterization_average_iterate}
    At local step $k$ of global round $t$, we have that $\sum_{i = 1}^C p_{i} \widetilde{\vb{w}}_{t, k}^i = \beta \overbar{\vb{w}}_{t, k} + (1 - \beta) \overbar{\vb{w}}_{t, 0}$.
\end{restatable}

\section{Theoretical analysis}
\label{sec:theoretical_analysis}

We adopt the following theoretical assumptions which are widely used in federated optimization.

\begin{assumption}[Full participation]
    \label{ass:full_participation}
    In each server round $t$, all clients of the network take part in the training process and communicate their updates, namely $\mathcal{S}_t = \left\{ 1, \ldots, C \right\}$.
\end{assumption}

\begin{assumption}[Bounded variance]
    \label{ass:bounded_variance}
    Stochastic gradients are unbiased and have bounded variance
    \begin{align}
        \expected \, \grad{\vb{w}_{t, k}^i}{i} = \nabla \Fi{\vb{w}_{t, k}^i}{i} \quad \textrm{and} \quad \expected \, \left\| \grad{\vb{w}_{t, k}^i}{i} - \nabla \Fi{\vb{w}_{t, k}^i}{i} \right\|^2 \le \sigma^2
        \label{eq:bounded_variance}
    \end{align}
    in expectation within client $i \in \{ \, 1, \ldots \, C \, \}$ where $\sigma > 0$.
\end{assumption}

\begin{assumption}[Bounded stochastic gradient norm]
    \label{ass:bounded_stochastic_norm}
    The expected norm of stochastic gradients is bounded as $\expected \, \| \grad{\vb{w}_{t, k}^i}{i} \|^2  \le G^2$ for any $i \in \{ \, 1, \ldots \, C \, \}$ and $(t, k) \in \{ \, 0, \ldots, T - 1 \,\} \times \{ \, 0, \ldots, E \,\}$.
\end{assumption}

\begin{assumption}[Smoothness]
    \label{ass:smoothness}
    Each local objective is $L$-smooth, namely, for any client $i \in \{ \, 1, \ldots \, C \, \}$ and $\vb{v}, \vb{w} \in \mathbb{R}^D$, we have that $\left\|\nabla \Fi{\vb{w}}{i} - \nabla \Fi{\vb{v}}{i}\right\| \le L\left\|\vb{w} - \vb{v}\right\|$ where $L > 0$.
\end{assumption}

\begin{assumption}[Strong convexity]
    \label{ass:strong_convexity}
    Each local objective is $\mu$-strongly convex, namely, for any client $i \in \{ \, 1, \ldots \, C \, \}$ and $\vb{v}, \vb{w} \in \mathbb{R}^D$, we have that $\Fi{\vb{w}}{i} \ge \Fi{\vb{v}}{i} + \nabla \Fi{\vb{v}}{i}^{\transpose}\left(\vb{w} - \vb{v}\right) + (\mu / 2) \|\vb{w} - \vb{v}\|^2 \label{eq:convexity_differentiable_definition}$.
\end{assumption}

$\ev{\cdot}$ denotes the total expectation. We further use a common definition to quantify the heterogeneity.

\begin{definition}[Statistical heterogeneity from \cite{fed_averaging_niid_convergence}]
    \label{ass:statistical_heterogeneity}
    The statistical heterogeneity is defined as $\Gamma \eqdef \Fstar - \sum_{i = 1}^C p_{i} \Fi{\vb{w}_{\star}^i}{i}$ where $\Fstar \eqdef \F{\vb{w}_{\star}}$ and $\vb{w}_{\star}^i\eqdef \argmin_{\vb{w}} \Fi{\vb{w}}{i}$ for any $i \in \{ \, 1, \ldots \, C \, \}$.
\end{definition}

This notion of statistical heterogeneity has already been widely adopted in the literature. For instance, \cite{fed_learning_power_of_choice} employed this definition to study the impact of biased selection strategies.

\subsection{Convergence analysis of \textsc{FedProx} as our baseline}
\label{sec:analysis_fedprox_baseline}

Originally, \cite{fedprox} introduced \textsc{FedProx} and its investigation under different premises. To conduct our analysis in the strongly convex scenario, we take inspiration from the study of \cite{sgd_guide_optimization} on stochastic gradient descent. However, our technical approach has other influences that include \cite{fed_averaging_niid_convergence} and \cite{fed_learning_power_of_choice}.

\begin{restatable}[Convergence of \textsc{FedProx} for strongly convex loss]{theorem}{ConvergenceFedproxStronglyConvex}
    \label{theorem:convergence_fedprox_strongly_convex}
    Under Assumptions \ref{ass:full_participation} to \ref{ass:strong_convexity}, we run \textsc{FedProx} with $\alpha > 0$. When using step size $\gamma = [2 E L_{\alpha}]^{-1}$ for $t \ge 0$, the algorithm satisfies
    \begin{align}
        \expected \, \F{\overbar{\vb{w}}_{t, 0}} - \Fstar  \le \, &\frac{L\Delta}{\mu} \left[1 - \frac{\mu}{3(L + \alpha)}\right]^t + \frac{L}{L_{\alpha}}\left[\frac{S \sigma^2}{4\mu} + \frac{3 L \Gamma}{2\mu} + \frac{2 E^2 G^2}{\mu}\right].
    \end{align}
    Lastly, we define $\Delta \eqdef \F{\overbar{\vb{w}}_{0, 0}} - \Fstar$, $S \eqdef \sum_{i = 1}^C p_{i}^2$ and $L_{\alpha} \eqdef \alpha + L$.
\end{restatable}

As in ordinary SGD, the usage of a fixed step size does not guarantee the convergence to the minimum $\vb{w}_{\star}$, but to a neighborhood whose size depends on the extent of stochasticity and statistical heterogeneity. On the other hand, decreasing the step size slows down convergence but ensures landing exactly on the critical point, as shown by \cite{fed_averaging_niid_convergence}.

In our analysis for nonconvex losses, we follow the example of \cite{sgd_guide_optimization} on the convergence of stochastic gradient descent, and by \cite{parallel_sgd_guide_optimization} on the performance of \textsc{FedAvg}. In the absence of convexity, however, we only replace Assumption \ref{ass:strong_convexity} with \ref{ass:lower_bounded_objective}.

\begin{assumption}[Lower bounded objective]
    \label{ass:lower_bounded_objective}
    The global objective loss $\F{\vb{w}}$ is lower bounded by $\Finfimum$.
\end{assumption}

This assumption has been already used by \cite{sgd_guide_optimization}, and prevents our analysis from requiring the existence of a global minimum. Differently from convex analysis, the convergence will be expressed through the average of the squared norms of the global gradient in different instants.

\begin{restatable}[Convergence of \textsc{FedProx} for nonconvex loss]{theorem}{ConvergenceFedproxNonConvex}
    \label{theorem:convergence_fedprox_nonconvex}
    Suppose  Assumption~\ref{ass:full_participation} to \ref{ass:smoothness} and \ref{ass:lower_bounded_objective} hold. Running algorithm \textsc{FedProx} with $\alpha > 0$ and step size $\gamma = [2 L_{\alpha}\sqrt{TE}]^{-1}$ for $T \ge 1$ rounds yields
    \begin{align*}
        \expected \, \left\|\nabla \F{\widehat{\vb{w}}_{T}}\right\|^2  \le \frac{1}{\sqrt{T}} \left[\frac{8 L_{\alpha} \Delta}{\sqrt{E}} + \frac{ L S \sigma^2}{L_{\alpha}\sqrt{E}} + \frac{\alpha E^{3 / 2} G^2}{L_{\alpha}}\right] + \frac{2 L^2 E G^2}{L_{\alpha}^2 T} + \frac{\alpha^2 L \sqrt{E} G^2}{16 L_{\alpha}^3 T^{3 / 2}}
    \end{align*}
    where we define $\Delta \eqdef \F{\overbar{\vb{w}}_{0, 0}} - \Finfimum$, $S \eqdef \sum_{i = 1}^C p_i^2$ and $L_{\alpha} \eqdef \alpha + L$. Furthermore, we uniformly sample $\widehat{\vb{w}}_{T}$ from $\{ \, \overbar{\vb{w}}_{t, k} \, \}_{t, k}$ for any combination of $0 \le t \le T - 1$ and $0 \le k \le E - 1$.
\end{restatable}

\subsection{Convergence analysis of our algorithm}
\label{sec:analysis_our_algorithm}

We examine our method. Likewise, Assumptions \ref{ass:bounded_variance} and \ref{ass:bounded_stochastic_norm} hold for the perturbed gradient $\grad{\widetilde{\vb{w}}_{t, k}^i}{i}$.

\begin{restatable}[Convergence of our algorithm for strongly convex loss]{theorem}{OurAlgorithmConvergenceStronglyConvex}
 \label{theorem:convergence_our_algorithm_strongly_convex}
    Let Assumptions \ref{ass:full_participation} to \ref{ass:strong_convexity} hold. We run our algorithm with $\beta \in (0, 1)$ and step size $\gamma = [2LE]^{-1}$. Then, for $t \ge 0$, we have the rate
    \begin{align}
        \expected \, \F{\overbar{\vb{w}}_{t, 0}} - \Fstar \le \, &\frac{L\Delta}{\mu} \left[1 - \frac{\mu}{(\beta + 2)L}\right]^t + \frac{S\sigma^2}{4\mu} + \frac{3L\Gamma}{2\mu} + \frac{2 A E^2 G^2}{\mu} + \frac{\beta(1 - \beta) E G^2}{8L}
    \end{align}
    where we denote $A \eqdef 4 + (1 - \beta)^2 + 8 \left(1 - 1/\beta\right)^2$, $\Delta \eqdef \F{\overbar{\vb{w}}_{0, 0}} - \Fstar$ and $S \eqdef \sum_{i = 1}^C p_{i}^2$.
\end{restatable}

The error term of our algorithm has an evident dependence on $(1 - 1/\beta)^2$. This forces the error to grow at a rate of $\mathcal{O}(1 /\beta^2)$ as $\beta$ becomes smaller. This reveals a potential limitation of our approach.

\begin{restatable}[Convergence of our algorithm for nonconvex loss]{theorem}{OurAlgorithmConvergenceNonConvex}
    \label{theorem:convergence_our_algorithm_nonconvex}
    Supposing that  \ref{ass:full_participation} to \ref{ass:smoothness} and \ref{ass:lower_bounded_objective} hold, we run our algorithm with $\beta \in (0, 1)$ for $T \ge 1$ rounds. Choosing step size $\gamma = [2L\sqrt{TE}]^{-1}$, we have
    \begin{align*}
        \expected \, \left\|\nabla \F{\widehat{\vb{w}}_{T}}\right\|^2 \le \frac{1}{\sqrt{T}} \left[ \frac{4 L \Delta}{\sqrt{E}}  + \frac{S \sigma^2}{2\sqrt{E}} \right] + \frac{E G^2}{T} \left[4 + (1 - \beta)^2 + 8\left(1 - \dfrac{1}{\beta}\right)^2 \right]
    \end{align*}
    where we define $\Delta = \F{\overbar{\vb{w}}_{0, 0}} - \Finfimum$ and $S \eqdef \sum_{i = 1}^C p_{i}^2$. Moreover, $\widehat{\vb{w}}_{T}$ is uniformly chosen from $\{ \, \overbar{\vb{w}}_{t, k} \, \}_{t, k}$ for any combination of $0 \le t \le T - 1$ and $0 \le k \le E - 1$.
\end{restatable}

Furthermore, as \cite{parallel_sgd_guide_optimization} did in the analysis of local stochastic gradient descent for \textsc{FedAvg}, we might question if there is an optimal $E_{\mathrm{opt}}$ depending on known $T$ that further minimizes the error complexity. The following result affirmatively answers our inquiry.

\begin{restatable}{corollary}{optimalNumberLocalStepsNonConvex}
   Consider Theorem~\ref{theorem:convergence_our_algorithm_nonconvex}, and choose a number of local steps $E = \mathcal{O}(T^{1 / 3})$. Then, the error asymptotically decreases as $\mathcal{O}(T^{-2 / 3})$.
\end{restatable}

\section{Discussion}
\label{sec:convergence_analysis_discussion}

We point out the main insights from the analysis of our algorithm and our baseline \textsc{FedProx}.

\paragraph{The antithetic role of $\alpha$ and $\beta$ in the strongly convex case}
\label{sec:contraction_rate_discussion}

We inspect the contraction rate of \textsc{FedProx} and our algorithm from Theorems \ref{theorem:convergence_fedprox_strongly_convex} and \ref{theorem:convergence_our_algorithm_strongly_convex}, respectively. Interestingly, we have
\begin{align}
    \underbrace{\left(1 - \frac{\mu}{(\beta + 2)L}\right)^t}_{\textrm{Ours}} \le \underbrace{\left(1 - \frac{\mu}{3L}\right)^t}_{\textsc{FedAvg}} \le \underbrace{\left(1 - \frac{\mu}{3(\alpha + L)}\right)^t}_{\textsc{FedProx}}.
    \label{eq:contraction_rate_upper_bound_comparison}
\end{align}
We attain the same contraction rate when $\beta = 1$ and $\alpha = 0$, that is the \textsc{FedAvg} base case. Curiously, this suggests that any choice of $\alpha > 0$ would worsen the contraction rate for \textsc{FedProx}, making it inevitably larger than \textsc{FedAvg}'s one. On the other hand, choosing positive $\beta < 1$, would improve the contraction rate for our algorithm as it would be smaller than $(1 - \mu/(3L))^t$. Does this mean that our algorithm has generally a better convergence rate than \textsc{FedAvg} and \textsc{FedProx}? Not exactly, since decreasing $\beta$ boosts the contraction factor but dramatically aggravates the asymptotic error due to the term depending on $1/\beta^2$, being an undeniable theoretical drawback of our algorithm. Contrarily, for \textsc{FedProx}, increasing $\alpha$ would shrink its asymptotic error because of factor $L/(L + \alpha)$.

\paragraph{Attaining $\epsilon$-accuracy in the nonconvex scenario}
\label{sec:nonconvex_scenario_discussion}

Considering Theorems \ref{theorem:convergence_fedprox_nonconvex} and \ref{theorem:convergence_our_algorithm_nonconvex}, we question how the minimum number of iterations $T_{\epsilon}$ changes across the studied algorithms in order to achieve an $\epsilon$-accuracy, namely $\expected \, \|\nabla \F{\widehat{\vb{w}}_{T_{\epsilon}}}\|^2 \le \epsilon$. In this respect, the iteration complexity of \textsc{FedProx} is
\begin{align*}
    \mathcal{O}\left(\frac{64 L_{\alpha}^2 \Delta^2}{E\epsilon^2}\right) +  \mathcal{O}\left(\frac{ L^2 S^2 \sigma^4}{L_{\alpha}^2 E \epsilon^2}\right) +  \mathcal{O}\left(\frac{\alpha^2 E^{3} G^4}{L_{\alpha}^2 \epsilon^2} + \frac{2 L^2 E G^2}{L_{\alpha}^2 \epsilon} + \frac{\alpha^{4/3} L^{2/3} E^{1/3} G^{4/3}}{16^{2/3} L_{\alpha}^2 \epsilon^{2 / 3}}\right).
\end{align*}
The existence of terms $\mathcal{O}(G^4 / \epsilon^2)$ (potentially harmful) and $\mathcal{O}(G^{4/3} / \epsilon^{2 / 3})$ is uniquely motivated by parameter $\alpha > 0$. Indeed, these additive contributions are absent in \textsc{FedAvg}. Contrarily, our algorithm has the following complexity to satisfy the same $\epsilon$-accuracy.
\begin{align*}
    \mathcal{O}\left(\frac{16 L^2 \Delta^2}{E \epsilon^2}\right)  + \mathcal{O}\left(\frac{S^2 \sigma^4}{4 E \epsilon^2}\right) + \mathcal{O}\left(\frac{E G^2}{\epsilon} \left[4 + (1 - \beta)^2 + 8\left(1 - \dfrac{1}{\beta}\right)^2 \right]\right)
\end{align*}
Under the hypothesis of a limited magnitude of the last term for a given value of $\beta \in (0, 1)$, we observe that the iteration complexity is asymptotically inferior compared with $\textsc{FedProx}$. However, we remark that any $\beta$ close to $0$ would degrade the bound with a rate of $\mathcal{O}(G^2 / (\beta^2\epsilon))$.

\section{Experimentation}
\label{sec:experimentation}

By conducting multiple experiments on different datasets, we aim to empirically measure both the convergence speed and generalization capacity of our algorithm and chosen baseline on unseen data.
\begin{table}[t]
    \caption{Numerical results on FEMNIST and CIFAR10. We compare \textsc{FedAvg} (\textsc{FedProx} with parameter $\alpha = 0$), and our algorithm $\textrm{Ours}_{\{\beta\}}$ with parameter $\beta \in \{ \, 0.5, 0.7, 0.9 \, \}$.}
    \renewcommand\arraystretch{1}
    \setlength{\tabcolsep}{5pt}
    \centering
    \small
    \begin{tabular}[t]{lccccc}
        \toprule
        & \textbf{Convex} & \multicolumn{4}{c}{\textbf{Testing accuracy (\%) / Rounds to converge (speedup)}} \\
        \midrule
        & & \multicolumn{2}{c}{\textbf{CIFAR10}} & \multicolumn{2}{c}{\textbf{FEMNIST}} \\
        \cmidrule{3-6}
        & & \textbf{Balanced} & \textbf{Imbalanced} & \textbf{Balanced} & \textbf{Imbalanced}
        \\
        \midrule
        \multirow{2}{*}{\textsc{FedAvg}} & Yes & \textbf{38.70} / 11 $\textcolor{gray}{(1.0\times)}$ & 36.53 / 20 $\textcolor{gray}{(1.0\times)}$ & 84.42 / 28 $\textcolor{gray}{(1.0\times)}$ & 80.20 / 80 $\textcolor{gray}{(1.0\times)}$ \\
        & No & 36.14 / 68 $\textcolor{gray}{(1.0\times)}$ & 34.04 / 74 $\textcolor{gray}{(1.0\times)}$ & 86.07 / 44 $\textcolor{gray}{(1.0\times)}$ & 80.94 / 94 $\textcolor{gray}{(1.0\times)}$ \\
        \midrule
        \multirow{2}{*}{$\textrm{Ours}_{\{0.9\}}$} & Yes & 38.68 / 11 $\textcolor{gray}{(1.0\times)}$ & 36.67 / 19 $\textcolor{gray}{(1.0\times)}$ & 84.48 / 27 $\textcolor{gray}{(1.0\times)}$ & 80.40 / 79 $\textcolor{gray}{(1.0\times)}$ \\
        & No & 36.15 / 68 $\textcolor{gray}{(1.0\times)}$ & 34.26 / 72 $\textcolor{gray}{(1.0\times)}$ & 86.15 / 43 $\textcolor{gray}{(1.0\times)}$ & 81.05 / 87 $\textcolor{gray}{(1.1\times)}$ \\
        \cmidrule{2-6}
        \multirow{2}{*}{$\textrm{Ours}_{\{0.7\}}$} & Yes & \textbf{38.70} / 11 $\textcolor{gray}{(1.0\times)}$ & 36.98 / 18 $\textcolor{gray}{(1.1\times)}$ & 84.54 / 24 $\textcolor{gray}{(1.2\times)}$ & 80.81 / 67 $\textcolor{gray}{(1.2\times)}$ \\
        & No & \textbf{36.16} / 68 $\textcolor{gray}{(1.0\times)}$ & 34.87 / 68 $\textcolor{gray}{(1.1\times)}$ & 86.25 / 41 $\textcolor{gray}{(1.1\times)}$ & 81.21 / 80 $\textcolor{gray}{(1.2\times)}$ \\
        \cmidrule{2-6}
        \multirow{2}{*}{$\textrm{Ours}_{\{0.5\}}$} & Yes & 38.67 / 11 $\textcolor{gray}{(1.0\times)}$ & \textbf{37.49} / \textbf{15}  $\textcolor{gray}{(1.3\times)}$ & \textbf{84.57} / \textbf{23} $\textcolor{gray}{(1.2\times)}$ & \textbf{81.24} / \textbf{50} $\textcolor{gray}{(1.6\times)}$ \\
        & No & 36.13 / 68 $\textcolor{gray}{(1.0\times)}$ & \textbf{35.37} / \textbf{63} $\textcolor{gray}{(1.2\times)}$ & \textbf{86.26} / \textbf{39} $\textcolor{gray}{(1.1\times)}$ & \textbf{81.43} / \textbf{80} $\textcolor{gray}{(1.2\times)}$ \\
        \bottomrule
    \end{tabular}
    \label{tab:algorithms_experiments_comparison}
\end{table}
\begin{figure}[b]
    \centering
    \includegraphics[width=1.0 \textwidth]{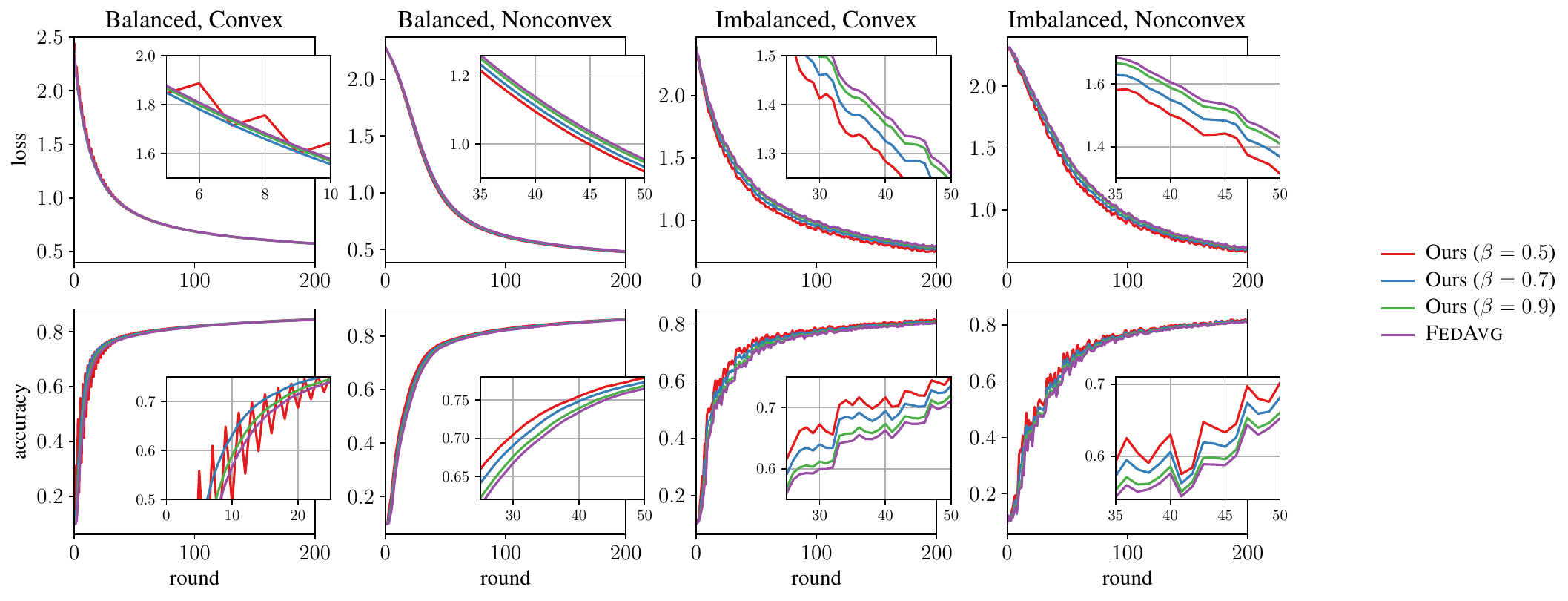}
    \caption{Simulation on balanced and imbalanced FEMNIST with the convex and nonconvex loss.}
    \label{fig:femnist-simulation}
\end{figure}

\paragraph{Settings}

We repeat our main experiment on two datasets. Specifically, we run $T = 200$ rounds on $C = 100$ clients where each participates and undertakes $E = 10$ local optimization epochs on its dataset. In addition, we leverage a local and fixed step size $\gamma = 10^{-3}$ and a $L_2$ penalty coefficient $\lambda = 10^{-4}$. No gradient clipping is applied, even though this would ensure that Assumption \ref{ass:bounded_stochastic_norm} holds. Finally, we implement the \textsc{Adjacency} scheme to aggregate the updates from the clients. This scheme defines clients' weights $p_i$ and similarities $p_{ij}$ as in Subsection \ref{sec:graph_based_model} and Algorithm \ref{alg:our_algorithm}.

\paragraph{Datasets}

Our tested datasets are FEMNIST from \cite{leaf} and CIFAR10 from \cite{cifar10}. We partition each dataset into training and testing subgroups (approximately $80$:$20$ split). We further divide each subgroup among $C = 100$ clients. We employ training clients to learn the model while testing clients for its evaluation on local subsets. In the balanced case, all clients hold the same amount of samples per class and the same portion of data. In the imbalanced scenario, as established by \cite{measuring_niid_data}, for each client, we sample the number of images of a specific class from a Dirichlet distribution and the dataset size from a log-normal distribution.

\paragraph{Loss objective}

We employ a multinomial logistic regression model with $L_2$ penalty as strongly convex and smooth loss. In a nonconvex scenario, we choose an elementary neural network with a single hidden layer composed of $128$ neurons using ReLU activation whose weights are $L_2$ penalized.

\paragraph{Validation metrics}

To measure how well each algorithm generalizes on unseen (testing) clients, we compute the accuracy of the learned model as the weighted average of each client's accuracy using the local amount of samples as weights. Additionally, to evaluate the convergence speed, we count the required number of rounds such that the accuracy exceeds 75\% (30\%) on FEMNIST (CIFAR10).

\subsection{Results}
\label{sec:results}

In this section, we analyze the results of our simulations in relation to our theoretical claims.

\paragraph{Decreasing $\beta$ does improve convergence}

Table \ref{tab:algorithms_experiments_comparison} shows that diminishing $\beta$ positively and consistently affects convergence on unseen data. Indeed, having $\beta = 0.5$ quite significantly hastens convergence by 30 rounds and yields a $1.6$ times speedup compared with \textsc{FedAvg}. As pictured in Figure \ref{fig:femnist-simulation} and \ref{fig:femnist-varying-E-simulation}, such an empirical outcome corroborates our theoretical result concerning strongly convex objectives. As already discussed, concerning the exponentially decaying term, our algorithm has a faster contraction factor than \textsc{FedAvg} ($\alpha = 0$) or \textsc{FedProx} ($\alpha > 0$) for comparable step sizes, which explains why it accelerates as we increase the perturbation by decreasing $\beta$.
\begin{figure}[h]
    \centering
    \includegraphics[width=1.0 \textwidth]{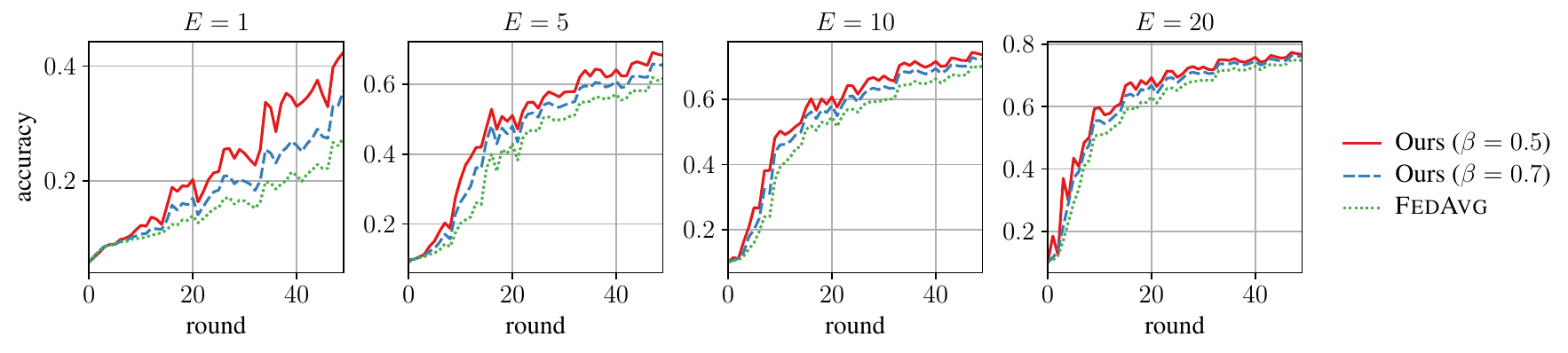}
    \caption{Strongly convex simulation varying the number of local epochs $E$ on imbalanced FEMNIST.}
    \label{fig:femnist-varying-E-simulation}
\end{figure}

\paragraph{Our algorithm is more effective on heterogeneous data}

Contrarily to the imbalanced case, Table \ref{tab:algorithms_experiments_comparison} evinces that the relative improvement shown by our algorithm in the analyzed balanced scenario is marginal, and the gain in accuracy over unseen clients is almost nonexistent after $T$ rounds. Figure \ref{fig:femnist-simulation} highlights that the trajectory of the convex loss is relatively stable and smooth as far as $\beta \in \{ \, 0.7, 0.9 \,\}$. Instead, when $\beta = 0.5$, the same curve becomes comparatively unstable. Reducing $\beta$ magnifies error term $\mathcal{O}(G^2 / \beta^2)$ (see \ref{theorem:convergence_our_algorithm_strongly_convex}). Why is this episode less relevant in the imbalanced case? If we adapt our empirical result to the theory, we might expect heterogeneity term $\mathcal{O}(\Gamma)$ to significantly outweigh $\mathcal{O}(G^2 / \beta^2)$ in imbalanced settings. Contrarily, when $\Gamma \approx 0$ in balanced contexts, $\mathcal{O}(G^2 / \beta^2)$ becomes dominant for smaller $\beta$, and our algorithm possibly outshoots at every step, which could explain the repeated trajectory correction that generates visible oscillations in the testing accuracy plot.

\section{Limitations}
\label{sec:limitations}

We already argued that the major theoretical defect of our method is given by term $\mathcal{O}(G^2 / \beta^2)$, which worsens the convergence rates from Theorems \ref{theorem:convergence_our_algorithm_strongly_convex} and \ref{theorem:convergence_our_algorithm_nonconvex}. It is also worth mentioning that our procedure has a higher communication cost than \textsc{FedAvg} or \textsc{FedProx} since the amount of data exchanged over the network includes the current global iterate $\overbar{\vb{w}}_{t, 0} \in \mathbb{R}^{D}$ plus an additional component $\vb{u}_{t}^i \in \mathbb{R}^{D}$ related to neighboring information, thus the expense still increases linearly with the dimensionality of the data. Practically, exchanging an initial message $\vb{m}_i$ summarizing the statistical nature of the local dataset might be problematic in relation to privacy constraints and data leak risks. Accordingly, advanced strategies should be undertaken to preserve confidentiality in real-world applications.

\section{Conclusions and future research}

We formalize a possible way to define a federated network of clients as a graph whose connections quantify the statistical affinity between clients' datasets, and we leverage this graph representation as the foundation of our novel federated algorithm, which forces every client to perform an inexact gradient-based optimization step. To implement the inexact update rule, we compute the gradient in a shifted coordinate, which minimizes the local variability of each client relative to other statistically similar clients. Our strongly convex analysis shows that our algorithm provides faster exponential contraction than the chosen baseline as we vary the parameter $\beta$ that controls the perturbation degree. The 
empirical results consistently support our claims, where our algorithm exhibits the expected gain in convergence speed. Future lines of research involve studying other forms of perturbation induced by mutual similarities and finding a better balance between perturbation and convergence stability.

\begingroup
\setstretch{1.0}

\endgroup

\addtocontents{toc}{\protect\setcounter{tocdepth}{2}}

\newpage

\appendix

\renewcommand{\contentsname}{\Huge{Appendices}\vspace{25pt}}

{
    \hypersetup{hidelinks}
    \tableofcontents
}

\newpage

\section{Proof of Theorem \ref{theorem:convergence_fedprox_strongly_convex} and \ref{theorem:convergence_fedprox_nonconvex}}

\label{appendix:analysis_of_fed_prox}

This appendix is dedicated to our analysis of \textsc{FedProx}. In this regard, we show all the related results and provide our deferred proofs.

\subsection{Preliminary results}

The optimization procedure of \textsc{FedProx} differs from plain \textsc{FedAvg} as follows. The update rule on iterate  $\vb{w}_{t, k}^i$ tries to solve inexactly the local problem
\begin{align}
    \argmin_{\vb{w} \, \in \, \mathbb{R}^D} \, \left\{ \, \Fi{\vb{w}}{i} + \frac{\alpha}{2}\left\|\vb{w} - \overbar{\vb{w}}_{t, 0}\right\|^2 \, \right\}
\end{align}
to compute the next iterate $\vb{w}_{t, k + 1}^i$. Therefore, by updating $\vb{w}_{t, k}^i$ in the negative direction of the quantity $\grad{\vb{w}_{t, k}^i}{i} + \alpha (\vb{w}_{t, k}^i - \overbar{\vb{w}}_{t, 0})$, the update rule is defined as
\begin{align}
    \vb{w}_{t, k + 1}^i &= (1 - \alpha \gamma_{t} )\vb{w}_{t, k}^i + \alpha \gamma_{t} \overbar{\vb{w}}_{t, 0} - \gamma_{t}  \grad{\vb{w}_{t, k}^i}{i}
    \label{eq:fedprox_update_rule_problem}
\end{align}
Parameter $\alpha > 0$ controls the proximal term, which should counteract the local divergence effect of each client. When $\alpha = 0$, we obtain the formulation of \textsc{FedAvg}. Under the premise of full participation, the average iterate sequence becomes
\begin{align}
    \overbar{\vb{w}}_{t, k + 1} = (1 - \alpha \gamma_{t} )\overbar{\vb{w}}_{t, k} + \alpha \gamma_{t} \overbar{\vb{w}}_{t, 0} - \gamma_{t} \sum_{i = 1}^C p_{i} \grad{\vb{w}_{t, k}^i}{i}
    \label{eq:fed_prox_update_rule}
\end{align}
The following lemmas help us obtain the main outcomes of our analysis. In particular, Lemma \ref{lemma:single_round_local_deviation} bounds the deviation of local iterate $\vb{w}_{t, k}^i$ for each client $i$ from the initial global iterate $\overbar{\vb{w}}_{t, 0}$ at the beginning of each round $t$.
\begin{lemma}[Single round local deviation of \textsc{FedProx}]
    \label{lemma:single_round_local_deviation}
    Assuming that $\gamma_{t} \le 1/\alpha$, and \ref{ass:full_participation} to \ref{ass:smoothness} hold, then the local deviation in one global round satisfies
    \begin{align}
        \expected \, \left\|\vb{w}_{t, k}^i - \overbar{\vb{w}}_{t, 0} \right\|^2 \le \gamma_{t}^2 E^2 G^2
    \end{align}
\end{lemma}

\begin{proof}
    \label{lemma:fedprox_single_round_local_deviation_proof}
    From the definition of the local update rule
    \begin{align*}
        \vb{w}_{t, k + 1}^i - \overbar{\vb{w}}_{t, 0} = (1 - \alpha \gamma_{t} )\left(\vb{w}_{t, k}^i - \overbar{\vb{w}}_{t, 0}\right) - \gamma_{t}  \grad{\vb{w}_{t, k}^i}{i}
    \end{align*}
    we use $d_{t, k}$ to denote $\vb{w}_{t, k}^i - \overbar{\vb{w}}_{t, 0}$, and we obtain by recursion
    \begin{align*}
        d_{t, k + 1}  &= (1 - \alpha \gamma_{t})d_{t, k} - \gamma_{t} \grad{\vb{w}_{t, k}^i}{i} \\
        &= (1 - \alpha \gamma_{t})\left((1 - \alpha \gamma_{t})d_{t, k - 1} - \gamma_{t} \grad{\vb{w}_{t, k - 1}^i}{i} \right) - \gamma_{t} \grad{\vb{w}_{t, k}^i}{i} \\
        &\ldots \\
        &= (1 - \alpha\gamma_{t})^{k + 1} d_{t, 0} -\gamma_{t} \sum_{m = 0}^{k} (1 - \alpha \gamma_{t})^m   \grad{\vb{w}_{t, k - m}^i}{i}
    \end{align*}
    Since $\vb{w}_{t, 0}^i = \overbar{\vb{w}}_{t, 0}$, by definition $d_{t, 0} = \vb{0}$.  Therefore,
    \begin{align}
        \left\|\vb{w}_{t, k}^i - \overbar{\vb{w}}_{t, 0} \right\|^2 &= \gamma_{t}^2 \left\| \sum_{m = 0}^{k - 1} (1 - \alpha \gamma_{t})^m   \grad{\vb{w}_{t, k - 1 - m}^i}{i} \right\|^2 \\
        &\le \gamma_{t}^2 \left(\sum_{m = 0}^{k - 1} (1 - \alpha \gamma_{t})^m \right) \sum_{m = 0}^{k - 1} (1 - \alpha \gamma_{t})^{m} \left\| \grad{\vb{w}_{t, k - 1 - m}^i}{i} \right\|^2 \label{eq:lemma_local_deviation_jensen} \\
        &\le \gamma_{t}^2 k \sum_{m = 0}^{k - 1} \left\| \grad{\vb{w}_{t, k - 1 - m}^i}{i} \right\|^2 \label{eq:lemma_local_deviation_approx_with_geometric}
    \end{align}
    where we apply Jensen's inequality in equation \eqref{eq:lemma_local_deviation_jensen}, and we notice that $(1 - \alpha\gamma_{t})^m \le 1$ in \eqref{eq:lemma_local_deviation_approx_with_geometric}. Finally, we bound norms of gradients recalling Assumption \ref{ass:bounded_stochastic_norm}.
    \begin{align*}
        \expected \, \left\|\vb{w}_{t, k}^i - \overbar{\vb{w}}_{t, 0} \right\|^2 &\le \gamma_{t}^2 k \sum_{m = 0}^{k - 1} \expected \, \left\| \grad{\vb{w}_{t, k - 1 - m}^i}{i} \right\|^2 \\
        &\le \gamma_{t}^2 k^2 G^2 \\
        &\le \gamma_{t}^2 E^2 G^2
    \end{align*}
    Taking total expectation concludes our proof.
\end{proof}

On the other hand, the following result from Lemma \ref{lemma:single_round_local_divergence} describes how client drift is bounded during a single round of training. A similar result was found by \cite{parallel_sgd_guide_optimization} concerning parallel stochastic gradient descent with periodic averaging, alias \textsc{FedAvg}. As in Lemma \ref{lemma:single_round_local_deviation}, the choice of the step size $\gamma_{t}$ depends on parameter $\alpha$ which controls the proximal regularization term.

\begin{lemma}[Single round local divergence of \textsc{FedProx}]
    \label{lemma:single_round_local_divergence}
    Assuming that $\gamma_{t} \le 1/\alpha $, and \ref{ass:full_participation} to \ref{ass:smoothness} hold, then the local divergence in one global round is bounded as
    \begin{align}
        \expected \, \left\|\overbar{\vb{w}}_{t, k} - \vb{w}_{t, k}^i \right\|^2 \le 4\gamma_{t}^2 E^2 G^2
    \end{align}
\end{lemma}

\begin{proof}
    \label{lemma:fedprox_single_round_local_divergence_proof}
    We apply Jensen's inequality for $\|\cdot\|^2$.
    \begin{align}
        \expected \, \left\|\overbar{\vb{w}}_{t, k + 1} - \vb{w}_{t, k + 1}^i \right\|^2 &= \expected \, \left\| \sum_{j = 1}^C p_{i} \vb{w}_{t, k + 1}^j - \vb{w}_{t, k + 1}^i \right\|^2 \\
        &\le \sum_{j = 1}^C p_{i} \, \expected \, \left\| \vb{w}_{t, k + 1}^j - \vb{w}_{t, k + 1}^i \right\|^2
        \label{eq:initial_jensen_inequality_local_divergence}
    \end{align}
    Under \ref{ass:full_participation}, we denote $\delta_{t, k}^{ij} = \vb{w}_{t, k}^j - \vb{w}_{t, k}^i$ and $\Delta g_{t, k}^{ij} = \grad{\vb{w}_{t, k}^i}{i} - \grad{\vb{w}_{t, k}^j}{j}$. Thus, from the definition of the local update rule, we have
    \begin{align*}
        \delta_{t, k + 1}^{ij} = (1 - \alpha \gamma_{t} )\delta_{t, k}^{ij} - \gamma_{t}  \Delta g_{t, k}^{ij}
    \end{align*}
    Similarly to proof of Lemma \ref{lemma:single_round_local_deviation}, we obtain by recursion
    \begin{align*}
        \delta_{t, k + 1}  &= (1 - \alpha \gamma_{t})\delta_{t, k} - \gamma_{t} \Delta g_{t, k}^{ij} \\
        &= (1 - \alpha \gamma_{t})\left((1 - \alpha \gamma_{t})\delta_{t, k - 1} - \gamma_{t} \Delta g_{t, k - 1}^{ij} \right) - \gamma_{t} \Delta g_{t, k}^{ij} \\
        &\ldots \\
        &= (1 - \alpha\gamma_{t})^{k + 1} \delta_{t, 0} - \gamma_{t} \sum_{m = 0}^{k} (1 - \alpha \gamma_{t})^m   \Delta g_{t, k - m}^{ij}
    \end{align*}
    Since $\vb{w}_{t, 0}^i = \vb{w}_{t, 0}^j = \overbar{\vb{w}}_{t, 0}$, by definition $d_{t, 0} = \vb{0}$. Hence,
    \begin{align}
        \delta_{t, k} &= \gamma_{t}^2 \left\| \sum_{m = 0}^{k - 1} (1 - \alpha \gamma_{t})^m   \Delta g_{t, k - 1 - m}^{ij} \right\|^2 \\
        &\le \gamma_{t}^2 \left(\sum_{m = 0}^{k - 1} (1 - \alpha \gamma_{t})^m \right) \sum_{m = 0}^{k - 1} (1 - \alpha \gamma_{t})^{m} \left\| \Delta g_{t, k - 1 - m}^{ij} \right\|^2 \label{lemma_local_divergence_jensen} \\
        &= \gamma_{t}^2 k \sum_{m = 0}^{k - 1} \left\| \Delta g_{t, k - 1 - m}^{ij} \right\|^2 \label{lemma_local_divergence_geometric_approximation}
    \end{align}
    In equation \eqref{lemma_local_divergence_jensen}, we recall Jensen's inequality, while, in equation \eqref{lemma_local_divergence_geometric_approximation}, we leverage the fact that $(1 - \alpha\gamma_{t})^m \le 1 $.
    Therefore, we bound $\|\Delta g_{t, k - 1 - m}^{ij} \|^2$ using Young's inequality and Assumption \ref{ass:bounded_stochastic_norm}.
    \begin{align*}
        \expected \, \left\|\vb{w}_{t, k + 1}^j - \vb{w}_{t, k + 1}^i \right\|^2
        &\le \gamma_{t}^2 k \sum_{m = 0}^{k - 1} \expected \, \left\|\Delta g_{t, k - 1 - m}^{ij} \right\|^2 \\
        &\le 2\gamma_{t}^2 k \sum_{m = 0}^{k - 1} \ev{ \left\|\grad{\vb{w}_{t, k}^i}{i}\right\|^2 + \left\|\grad{\vb{w}_{t, k}^j}{j}\right\|^2 } \\
        &\le 4\gamma_{t}^2 k^2 G^2 \\
        &\le 4\gamma_{t}^2 E^2 G^2
    \end{align*}
    After combining into \ref{eq:initial_jensen_inequality_local_divergence} and taking total expectation, we conclude our proof.
\end{proof}

\subsection{Main results for strongly convex analysis}

Lemma \ref{lemma:single_round_global_progress} expresses the global progress made in a single round of communication. This result is limited to the scenario in which the local objectives are strongly convex. We feel inspired by \cite{fed_optimization_guide} to provide this intermediate lemma to aid the development of further statements. The following result highlights the issue related to the choice of the step size $\gamma_{t}$. Indeed, a larger step size will rapidly shrink the previous distance from the global minimum $\vb{w}_{\star}$ while amplifying the effect of term $A$, which reflects all the pathological properties of the federated setting, namely the statistical heterogeneity of the network and the stochastic gradient behavior.

\begin{lemma}[Single round progress of \textsc{FedProx} for strongly convex loss]
    \label{lemma:single_round_global_progress}
    Assume that
    \begin{align}
        \label{lemma:one_step_progress_step_decay}
        \gamma_{t} \le \min{\left\{ \frac{1}{2L}, \frac{1}{\alpha + \mu} \right\}}
    \end{align}
    and \ref{ass:full_participation} to \ref{ass:strong_convexity} hold, then the progress in one global round satisfies
    \begin{align}
        \expected \, \left\|\overbar{\vb{w}}_{t + 1, 0} - \vb{w}_{\star}\right\|^2 \le \kappa \, \expected \, \left\|\overbar{\vb{w}}_{t, 0} - \vb{w}_{\star}\right\|^2 + A
    \end{align}
    where we define $\kappa = \dfrac{\alpha + \mu \left( 1 - \gamma_{t}(\alpha + \mu) \right)^E }{\alpha + \mu} \le 1 - \gamma_{t}\mu$ and
    \begin{align}
        A = \gamma_{t}^2 E \sigma^2 \sum_{i = 1}^C p_{i}^2 + 6\gamma_{t}^2 LE\Gamma + 8 \gamma_{t}^2 E^3 G^2
    \end{align}
\end{lemma}
\begin{proof}
    We denote $\left\|\overbar{\vb{w}}_{t, k} - \vb{w}_{\star}\right\|^2$ as $D_{t, k}$. We aim to arrive at the following inequality.
    \begin{align}
        \expected \, D_{t, k + 1} = a \, \expected \, D_{t, k} + b \, \expected \, D_{t, 0} + c
        \label{eq:recursion_one_progress_folded}
    \end{align}
    where $a, b, c$ are problem-related coefficients. By definition of update rule \eqref{eq:fed_prox_update_rule}, we have that
    \begin{align*}
        \overbar{\vb{w}}_{t, k + 1} - \vb{w}_{\star} &= (1 - \alpha \gamma_{t} )\left(\overbar{\vb{w}}_{t, k} - \vb{w}_{\star}\right) + \alpha \gamma_{t} \left(\overbar{\vb{w}}_{t, 0} - \vb{w}_{\star}\right) - \gamma_{t} \sum_{i = 1}^C p_{i} \grad{\vb{w}_{t, k}^i}{i} \\
        &= (1 - \alpha \gamma_{t} )\left(\overbar{\vb{w}}_{t, k} - \vb{w}_{\star}\right) + \alpha \gamma_{t} \left(\overbar{\vb{w}}_{t, 0} - \vb{w}_{\star}\right) - \gamma_{t} \sum_{i = 1}^C p_{i} \nabla \Fi{\vb{w}_{t, k}^i}{i}  + \\ &\fake{=} \underbrace{\gamma_{t} \sum_{i = 1}^C p_{i} \nabla \Fi{\vb{w}_{t, k}^i}{i} - \gamma_{t} \sum_{i = 1}^C p_{i} \grad{\vb{w}_{t, k}^i}{i}}_{\vb{v}}
    \end{align*}
    We have that $\expected \, \vb{v} = 0$ due to Assumption \ref{ass:bounded_variance} on the unbiased stochastic gradient, thus all mixed products of nature $\ev{2 \vb{v}^{\transpose} \vb{u}}$ are erased under expectation. Therefore
     \begin{align*}
        \expected \, D_{t, k + 1} &= (1 - \alpha \gamma_{t} )^2 \, \expected \, D_{t, k} + (\alpha \gamma_{t})^2 \, \expected \,  D_{t, 0} + \\
        &\fake{=} \ev{\underbrace{2\alpha \gamma_{t}(1 - \alpha \gamma_{t})\left(\overbar{\vb{w}}_{t, k} - \vb{w}_{\star}\right)^{\transpose}\left(\overbar{\vb{w}}_{t, 0} - \vb{w}_{\star}\right)}_{a_1}} + \\
        &\fake{=} \ev{\underbrace{-2\gamma_{t}(1 - \alpha \gamma_{t})\left(\overbar{\vb{w}}_{t, k} - \vb{w}_{\star}\right)^{\transpose}\left[\sum_{i = 1}^C p_{i} \nabla \Fi{\vb{w}_{t, k}^i}{i}\right]}_{a_2}} + \\
        &\fake{=} \ev{\underbrace{-2\alpha\gamma_{t}^2\left(\overbar{\vb{w}}_{t, 0} - \vb{w}_{\star}\right)^{\transpose}\left[\sum_{i = 1}^C p_{i} \nabla \Fi{\vb{w}_{t, k}^i}{i}\right]}_{a_3}} + \\
        &\fake{=} \ev{\underbrace{\gamma_{t}^2 \left\| \sum_{i = 1}^C p_{i} \nabla \Fi{\vb{w}_{t, k}^i}{i}\right\|^2 + \gamma_{t}^2 \left\|\sum_{i = 1}^C p_{i} \nabla \Fi{\vb{w}_{t, k}^i}{i} - \gamma_{t} \sum_{i = 1}^C p_{i} \grad{\vb{w}_{t, k}^i}{i}\right\|^2}_{a_4}}
    \end{align*}
    First, we bound term $a_1$ using the law $2\vb{u}^{\transpose}\vb{v} = \|\vb{u}\|^2 + \|\vb{v}\|^2 - \|\vb{u} - \vb{v}\|^2$. Therefore,
    \begin{align*}
        a_1 &= \alpha \gamma_{t}(1 - \alpha \gamma_{t})\left[ D_{t, k} + D_{t, 0} - \left\|\overbar{\vb{w}}_{t, k} - \overbar{\vb{w}}_{t, 0}\right\|^2 \right] \\
        &\le \alpha \gamma_{t}(1 - \alpha \gamma_{t})\left[ D_{t, k} + D_{t, 0} \right]
    \end{align*}
    We bound term $a_2$ by adding and subtracting term $\vb{w}_{t, k}^i$.
    \begin{align*}
        a_2 &= -2\gamma_{t}(1 - \alpha \gamma_{t}) \left[\sum_{i = 1}^C p_{i} \left(\overbar{\vb{w}}_{t, k} - \vb{w}_{t, k}^i\right)^{\transpose}\nabla \Fi{\vb{w}_{t, k}^i}{i} + \sum_{i = 1}^C p_{i} \left(\vb{w}_{t, k}^i - \vb{w}_{\star}\right)^{\transpose}\nabla \Fi{\vb{w}_{t, k}^i}{i} \right]
    \end{align*}
    We apply Peter-Paul's inequality on the first term of the sum and strong convexity on the second term.
    \begin{align*}
        a_2 &\le \gamma_{t}(1 - \alpha \gamma_{t}) \sum_{i = 1}^C p_{i} \left[\frac{1}{\gamma_{t}} \left\|\overbar{\vb{w}}_{t, k} - \vb{w}_{t, k}^i\right\|^2 + \gamma_{t} \left\|\nabla \Fi{\vb{w}_{t, k}^i}{i}\right\|^2\right] + \\
        &\fake{\le} \gamma_{t}(1 - \alpha \gamma_{t}) \sum_{i = 1}^C p_{i} \left[2\left(\Fi{\vb{w}_{\star}}{i} - \Fi{\vb{w}_{t, k}^i}{i}\right) - \mu\left\| \vb{w}_{t, k}^i - \vb{w}_{\star} \right\|^2 \right]
    \end{align*}
    By convexity of the squared norm, we have
    \begin{align*}
        -\mu\gamma_{t}(1 - \alpha \gamma_{t}) \sum_{i = 1}^C p_{i} \left\| \vb{w}_{t, k}^i - \vb{w}_{\star} \right\|^2 \le -\mu\gamma_{t}(1 - \alpha \gamma_{t}) \left\| \overbar{\vb{w}}_{t, k} - \vb{w}_{\star} \right\|^2
    \end{align*}
    and by smoothness of $\Fi{\cdot}{i}$, we have the following.
    \begin{align*}
        \gamma_{t}^2(1 - \alpha \gamma_{t}) \sum_{i = 1}^C p_{i} \left\|\nabla \Fi{\vb{w}_{t, k}^i}{i}\right\|^2 \le 2L\gamma_{t}^2(1 - \alpha \gamma_{t}) \sum_{i = 1}^C p_{i} \left(\Fi{\vb{w}_{t, k}^i}{i} - \Fi{\vb{w}_{\star}^i}{i}\right)
    \end{align*}
    Thus, considering that $\sum_{i = 1}^C p_{i} \Fi{\cdot}{i} = \F{\cdot}$, we obtain
    \begin{align*}
        a_2 &\le (1 - \alpha \gamma_{t}) \sum_{i = 1}^C p_{i} \left\|\overbar{\vb{w}}_{t, k} - \vb{w}_{t, k}^i\right\|^2 + 2L\gamma_{t}^2(1 - \alpha \gamma_{t}) \sum_{i = 1}^C p_{i} \left(\Fi{\vb{w}_{t, k}^i}{i} - \Fi{\vb{w}_{\star}^i}{i}\right) + \\
        &\fake{\le} 2\gamma_{t}(1 - \alpha \gamma_{t}) \sum_{i = 1}^C p_{i} \left(\Fi{\vb{w}_{\star}}{i} - \Fi{\vb{w}_{t, k}^i}{i}\right) - \mu\gamma_{t}(1 - \alpha\gamma_{t}) D_{t, k}
    \end{align*}
    We bind the term $a_3$ in a similar fashion to $a_2$.
    \begin{align*}
        a_3 &\le \alpha\gamma_{t}^2 \sum_{i = 1}^C p_{i} \left[\frac{1}{\gamma_{t}} \left\|\overbar{\vb{w}}_{t, 0} - \vb{w}_{t, k}^i\right\|^2 + \gamma_{t} \left\|\nabla \Fi{\vb{w}_{t, k}^i}{i}\right\|^2\right] + \\
        &\fake{\le} \alpha\gamma_{t}^2 \sum_{i = 1}^C p_{i} \left[2\left(\Fi{\vb{w}_{\star}}{i} - \Fi{\vb{w}_{t, k}^i}{i}\right) - \mu\left\| \vb{w}_{t, k}^i - \vb{w}_{\star} \right\|^2 \right] \\
        &= \alpha \gamma_{t} \sum_{i = 1}^C p_{i} \left\|\overbar{\vb{w}}_{t, k} - \vb{w}_{t, k}^i\right\|^2 + 2L\alpha\gamma_{t}^3 \sum_{i = 1}^C p_{i} \left(\Fi{\vb{w}_{t, k}^i}{i} - \Fi{\vb{w}_{\star}^i}{i}\right) + \\
        &\fake{\le} 2\alpha \gamma_{t}^2 \sum_{i = 1}^C p_{i} \left(\Fi{\vb{w}_{\star}}{i} - \Fi{\vb{w}_{t, k}^i}{i}\right) - \mu\alpha\gamma_{t}^2 D_{t, k}
    \end{align*}
    We bound term $a_4$ directly under expectation.
    \begin{align*}
        \expected \, a_4 &= \gamma_{t}^2 \, \expected \underbrace{\left\|\sum_{i = 1}^C p_{i} \nabla \Fi{\vb{w}_{t, k}^i}{i} - \gamma_{t} \sum_{i = 1}^C p_{i} \grad{\vb{w}_{t, k}^i}{i}\right\|^2}_{a_{41}} + \gamma_{t}^2 \, \expected \underbrace{\left\| \sum_{i = 1}^C p_{i} \nabla \Fi{\vb{w}_{t, k}^i}{i}\right\|^2}_{a_{42}}
    \end{align*}
    To bound term $\|a_{41}\|^2$, again, we apply Assumption \ref{ass:bounded_variance} to erase the dot products between terms in $a_{41}$ and bound the squared norms. Thus, we have
    \begin{align*}
        \expected \, \|a_{41}\|^2 &= \sum_{i = 1}^C p_{i}^2 \, \expected \, \left\|\grad{\vb{w}_{t, k}^i}{i} - \nabla \Fi{\vb{w}_{t, k}^i}{i}\right\|^2 \le \sigma^2 \sum_{i = 1}^C p_{i}^2
    \end{align*}
    Now, to bound term $\|a_{42}\|^2$, we utilize Jensen's inequality since $\|\cdot\|^2$ is convex.
    \begin{align*}
        \expected \, \|a_{42}\|^2 &= \left\|\sum_{i = 1}^C p_{i} \nabla \Fi{\vb{w}_{t, k}^i}{i}\right\|^2 \le \sum_{i = 1}^C p_{i} \, \expected \, \left\| \nabla \Fi{\vb{w}_{t, k}^i}{i}\right\|^2
    \end{align*}
    We combine these results to establish a bound for $a_4$. Finally we apply smoothness on $\| \nabla \Fi{\vb{w}_{t, k}^i}{i}\|^2$.
    \begin{align*}
        \expected \, a_4 &\le \gamma_{t}^2 \sigma^2 \sum_{i = 1}^C p_{i}^2 + \gamma_{t}^2 \sum_{i = 1}^C p_{i} \, \expected \, \left\| \nabla \Fi{\vb{w}_{t, k}^i}{i} \right\|^2 \\
        &\le \gamma_{t}^2 \sigma^2 \sum_{i = 1}^C p_{i}^2 + 2L\gamma_{t}^2 \sum_{i = 1}^C p_{i} \, \ev{\Fi{\vb{w}_{t, k}^i}{i} - \Fi{\vb{w}_{\star}^i}{i}}
    \end{align*}
    Combining all the bounds in the main equation under expectation leads to
    \begin{align*}
        \expected \, D_{t, k + 1} &\le \left( 1 - \gamma_{t}(\alpha + \mu) \right) \expected \, D_{t, k} + \alpha \gamma_{t} \, \expected \, D_{t, 0} + \gamma_{t}^2 \sigma^2 \sum_{i = 1}^C p_{i}^2 + \\
        &\fake{\le} (1 - \alpha \gamma_{t}) \sum_{i = 1}^C p_{i} \, \expected \, \left\|\overbar{\vb{w}}_{t, k} - \vb{w}_{t, k}^i\right\|^2 + \alpha \gamma_{t} \sum_{i = 1}^C p_{i} \, \expected \, \left\|\overbar{\vb{w}}_{t, 0} - \vb{w}_{t, k}^i\right\|^2 + \\
        &\fake{\le} \ev{\underbrace{4L\gamma_{t}^2 \sum_{i = 1}^C p_{i} \left(\Fi{\vb{w}_{t, k}^i}{i} - \Fi{\vb{w}_{\star}^i}{i}\right) + 2\gamma_{t} \sum_{i = 1}^C p_{i} \left(\Fi{\vb{w}_{\star}}{i} -  \Fi{\vb{w}_{t, k}^i}{i} \right)}_{b_1}}
    \end{align*}
    We introduce $(4L\gamma_{t}^2/C) \sum_{i = 1}^C  \Fi{\vb{w}_{\star}}{i}$ in term $b_1$ as
    \begin{align}
        b_1 &= -2\gamma_{t}(1 - 2 L \gamma_{t}) \sum_{i = 1}^C p_{i} \left(\Fi{\vb{w}_{t, k}^i}{i} - \Fi{\vb{w}_{\star}}{i}\right) + 4L\gamma_{t}^2 \sum_{i = 1}^C p_{i} \left(\Fi{\vb{w}_{\star}}{i} - \Fi{\vb{w}_{\star}^i}{i}\right) \\
        &= -2\gamma_{t}(1 - 2 L \gamma_{t}) \sum_{i = 1}^C p_{i} \left(\Fi{\vb{w}_{t, k}^i}{i} - \Fi{\vb{w}_{\star}}{i}\right) + 4\gamma_{t}^2L\Gamma  \label{eq:lemma_fedprox_progress_heterogeneity_substitution}
    \end{align}
    where, in \eqref{eq:lemma_fedprox_progress_heterogeneity_substitution}, we use $\sum_{i = 1}^C p_i \Fi{\cdot}{i} = \F{\cdot}$, and we exploit the Definition \ref{ass:statistical_heterogeneity} of statistical heterogeneity. Adding and subtracting $\Fi{\overbar{{\vb{w}}}_{t, k}}{i}$ in the summation from \eqref{eq:lemma_fedprox_progress_heterogeneity_substitution}, we have
    \begin{align}
        b_1 &= - 2\gamma_{t}(1 - 2 L \gamma_{t}) \sum_{i = 1}^C p_{i} \left(\Fi{\overbar{{\vb{w}}}_{t, k}}{i} - \Fi{\vb{w}_{\star}}{i}\right) + 4\gamma_{t}^2L\Gamma  + \\
        &\fake{=} -2\gamma_{t}(1 - 2 L \gamma_{t}) \sum_{i = 1}^C p_{i} \left(\Fi{\vb{w}_{t, k}^i}{i} - \Fi{\overbar{{\vb{w}}}_{t, k}}{i}\right) \\
        &= - 2\gamma_{t}(1 - 2 L \gamma_{t}) \left(\F{\overbar{{\vb{w}}}_{t, k}} - \F{\vb{w}_{\star}}\right) + 4\gamma_{t}^2L\Gamma  + \\
        &\fake{=} 2\gamma_{t}(1 - 2 L \gamma_{t}) \sum_{i = 1}^C p_{i} \left(\Fi{\overbar{{\vb{w}}}_{t, k}}{i} - \Fi{\vb{w}_{t, k}^i}{i}\right) \\
        &\le - 2\gamma_{t}(1 - 2 L \gamma_{t}) \left(\F{\overbar{{\vb{w}}}_{t, k}} - \F{\vb{w}_{\star}}\right) + 4\gamma_{t}^2L\Gamma  + \\
        &\fake{=} 2\gamma_{t}(1 - 2 L \gamma_{t}) \sum_{i = 1}^C p_{i} \nabla \Fi{\overbar{{\vb{w}}}_{t, k}}{i}^{\transpose}\left(\overbar{{\vb{w}}}_{t, k} - {\vb{w}}_{t, k}^i\right) \label{eq:lemma_fedprox_progress_convexity_usage} \\
        &\le - 2\gamma_{t}(1 - 2 L \gamma_{t}) \left(\F{\overbar{{\vb{w}}}_{t, k}} - \F{\vb{w}_{\star}}\right) + 4\gamma_{t}^2L\Gamma  + \\
        &\fake{=} 2\gamma_{t}(1 - 2 L \gamma_{t}) \sum_{i = 1}^C p_{i} \left[\frac{\gamma_{t}}{2} \left\|\nabla \Fi{\overbar{{\vb{w}}}_{t, k}}{i}\right\|^2 + \frac{1}{2\gamma_{t}} \left\|\overbar{{\vb{w}}}_{t, k} - {\vb{w}}_{t, k}^i\right\|^2 \right] \label{eq:lemma_fedprox_progress_peter_paul_in_B} \\
        &\le - 2\gamma_{t}(1 - 2 L \gamma_{t}) \left(\F{\overbar{{\vb{w}}}_{t, k}} - \F{\vb{w}_{\star}}\right) + 4\gamma_{t}^2L\Gamma  + \\
        &\fake{=} 2\gamma_{t}(1 - 2 L \gamma_{t}) \sum_{i = 1}^C p_{i} \left[L\gamma_{t}\left(\Fi{\overbar{{\vb{w}}}_{t, k}}{i} - \Fi{\vb{w}_{\star}^i}{i}\right) + \frac{1}{2\gamma_{t}} \left\|\overbar{{\vb{w}}}_{t, k} - {\vb{w}}_{t, k}^i\right\|^2 \right] \label{eq:lemma_fedprox_progress_smoothness_in_B}
    \end{align}
    We use convexity in \eqref{eq:lemma_fedprox_progress_convexity_usage}, Peter-Paul's inequality in \eqref{eq:lemma_fedprox_progress_peter_paul_in_B}, and smoothness of $\|\nabla \Fi{\overbar{{\vb{w}}}_{t, k}}{i}\|^2$ in \eqref{eq:lemma_fedprox_progress_smoothness_in_B}.
    \begin{align}
        b_1 &\le -2\gamma_{t}(1 - 2L\gamma_{t})(1 - L\gamma_{t}) \left(\F{\overbar{{\vb{w}}}_{t, k}} - \F{\vb{w}_{\star}}\right) + 4\gamma_{t}^2L\Gamma +  \\
        &\fake{\le} 2L\gamma_{t}^2(1 - 2L\gamma_{t})\left(\F{\vb{w}_{\star}} - \sum_{i = 1}^C p_{i} \Fi{\vb{w}_{\star}^i}{i}\right) + \sum_{i = 1}^C p_{i} \left\|\overbar{{\vb{w}}}_{t, k} - {\vb{w}}_{t, k}^i\right\|^2 \label{eq:lemma_fedprox_progress_1_2L_approx}  \\
        &= -2\gamma_{t}(1 - 2L\gamma_{t})(1 - L\gamma_{t}) \left(\F{\overbar{{\vb{w}}}_{t, k}} - \F{\vb{w}_{\star}}\right) + 2L\Gamma\gamma_{t}^2(3 - 2L\gamma_{t}) + \\
        &\fake{=} \sum_{i = 1}^C p_{i} \left\|\overbar{{\vb{w}}}_{t, k} - {\vb{w}}_{t, k}^i\right\|^2 \label{eq:lemma_fedprox_progress_heterogeneity_substitution_2} \\
        &\le 6L\Gamma\gamma_{t}^2 + \sum_{i = 1}^C p_{i} \left\|\overbar{{\vb{w}}}_{t, k} - {\vb{w}}_{t, k}^i\right\|^2 \label{eq:lemma_fedprox_progress_bound_B_end}
    \end{align}
    In expression \eqref{eq:lemma_fedprox_progress_1_2L_approx}, we recall that $1 - 2L\gamma_{t} \le 1$, and we note that $(1 - 2L\gamma_{t})(1 - L\gamma_{t}) \ge 0$ in equation \eqref{eq:lemma_fedprox_progress_bound_B_end}. These facts follow from Assumption \ref{lemma:one_step_progress_step_decay} on the step size. Eventually, we reuse the definition of heterogeneity in \eqref{eq:lemma_fedprox_progress_heterogeneity_substitution_2}. Therefore, we replace $b_1$ into our main bound.
    \begin{align*}
        \expected \, D_{t, k + 1} &\le \left( 1 - \gamma_{t}(\alpha + \mu) \right) \, \expected \, D_{t, k} + \alpha \gamma_{t} \, \expected \,  D_{t, 0} + \frac{\gamma_{t}^2 \sigma^2}{C} + 6L\Gamma\gamma_{t}^2 + \\
        &\fake{\le} (2 - \alpha \gamma_{t}) \sum_{i = 1}^C p_{i} \, \expected \, \left\|\overbar{\vb{w}}_{t, k} - \vb{w}_{t, k}^i\right\|^2 + \alpha \gamma_{t} \sum_{i = 1}^C p_{i} \, \expected \, \left\|\overbar{\vb{w}}_{t, 0} - \vb{w}_{t, k}^i\right\|^2
    \end{align*}
    Now, we use lemmas \ref{lemma:single_round_local_deviation} and \ref{lemma:single_round_local_divergence} to bound our main term, and we approximate $8 - 3\alpha\gamma_{t} \le 8$.
    \begin{align*}
        \expected \, D_{t, k + 1} &\le
        \left( 1 - \gamma_{t}(\alpha + \mu) \right) \expected \, D_{t, k} + \alpha \gamma_{t} \, \expected \, D_{t, 0} + \gamma_{t}^2 \sigma^2 \sum_{i = 1}^C p_{i}^2 + 6L\Gamma\gamma_{t}^2 + (8 - 3\alpha\gamma_{t} ) \gamma_{t}^2 E^2 G^2 \\
        &\le \left(1 - \gamma_{t}(\alpha + \mu) \right) \expected \, D_{t, k} + \alpha \gamma_{t} \, \expected \, D_{t, 0} + \gamma_{t}^2 \sigma^2 \sum_{i = 1}^C p_{i}^2 + 6L\Gamma\gamma_{t}^2 + 8 \gamma_{t}^2 E^2 G^2
    \end{align*}
    We are interested in relating $\expected \, D_{t + 1, 0} = \expected \, D_{t, E}$ to $\expected \, D_{t, 0}$. Accordingly, we define the constants
    \begin{align*}
        a &= \left( 1 - \gamma_{t}(\alpha + \mu) \right) \\
        b &= \alpha\gamma_{t} \\
        c &= \gamma_{t}^2 \sigma^2 \sum_{i = 1}^C p_{i}^2 + 6L\Gamma\gamma_{t}^2 + 8 \gamma_{t}^2 E^2 G^2
    \end{align*}
    Using parameters $a$, $b$ and $c$, we have an expression of the form \eqref{eq:recursion_one_progress_folded}. The application of recursion leads to the following result. We use a coarser approximation for the summation that multiplies $c$ to preserve factor $\gamma_{t}^2$ within term $c$.
    \begin{align}
        \expected \, D_{t, k + 1} &\le a^{k + 1} \, \expected \, D_{t, 0} + \left(b \sum_{m = 0}^{k} a^m\right) \expected \, D_{t, 0} + c \sum_{m = 0}^{k} a^m \\
        &\le \left[a^{k + 1} + b \frac{1 - a^{k + 1}}{1 - a}\right] \expected \,  D_{t, 0} + c (k + 1) \\
        &= \frac{b + (1 - a - b)a^{k + 1}}{1 - a} \, \expected \, D_{t, 0} + c (k + 1)
    \end{align}
    For the sake of our proof, we replace $k + 1 = E$ since we are interested in one round.
    \begin{align*}
        \expected \, D_{t, E} \le \frac{\alpha + \mu \left( 1 - \gamma_{t}(\alpha + \mu) \right)^E }{\alpha + \mu} \, \expected \, D_{t, 0} + cE
    \end{align*}
    We attain our expected result when substituting $c$ and taking total expectation.
\end{proof}

Here, we present the convergence guarantees of \textsc{FedProx} in case of full participation. Note that by setting proximal parameter $\alpha = 0$ in the following, we recover the optimality gap of \textsc{FedAvg}.

\ConvergenceFedproxStronglyConvex*

\begin{proof}
    To prove the statement, we first apply the principle of recursion on the result of Lemma \ref{lemma:single_round_global_progress} in equation \eqref{eq:convergence_fedprox_constant_step_size_lemma_replacement}. Namely, if we denote $\expected \, \left\|\overbar{\vb{w}}_{t, 0} - \vb{w}_{\star}\right\|^2$ as $D_{t}$, we have
    \begin{align}
        D_{t} &\le \kappa D_{t - 1} + A \label{eq:convergence_fedprox_constant_step_size_lemma_replacement} \\
        &\le \kappa  \left(\kappa D_{t - 2} + A\right) + A \\
        &\ldots \\
        &\le \kappa^{t} D_{0} + A \sum_{m = 0}^{t - 1} \kappa^{m} \\
        &\le \kappa^{t} D_{0} + A \sum_{m = 0}^{t - 1} \frac{1 - \kappa^{t}}{1 - \kappa} \\
        &\le \kappa^{t} D_{0} + A \frac{1 - \kappa^{t}}{1 - (1 - \gamma\mu)} \label{eq:convergence_fedprox_constant_step_size_contraction_replacement} \\
        &\le \kappa^{t} \left(D_{0} - \frac{A}{\gamma\mu}\right) + \frac{A}{\gamma\mu}
    \end{align}
    where we use the coarser but simpler approximation $\kappa \le 1 - \gamma\mu$ in equation \eqref{eq:convergence_fedprox_constant_step_size_contraction_replacement}. Finally, under total expectation, we invoke smoothness for $\F{\overbar{\vb{w}}_{t, 0}}$.
    \begin{align*}
        \expected \, \F{\overbar{\vb{w}}_{t, 0}} - \F{\vb{w}_{\star}} &\le \frac{L}{2} D_{t} \le \frac{L}{2}\kappa^{t} \left(D_{0} -  \frac{A}{\gamma\mu}\right) + \frac{L}{2\gamma\mu}A
    \end{align*}
    In addition, using strong convexity, we have $D_{0} \le 2\left(\F{\overbar{\vb{w}}_{0, 0}} - \F{\vb{w}_{\star}}\right)/\mu$. We now bound term $\kappa$ by replacing step size $\gamma$ with its definition.
    \begin{align}
        \kappa &= \frac{\alpha + \mu \left( 1 - (\alpha + \mu)/(2E (\alpha + L)) \right)^E }{\alpha + \mu} \\
            &= \frac{\alpha + \mu \left[\left( 1 - (\alpha + \mu)/(2E (\alpha + L)) \right)^{-2E(\alpha + L) / (\alpha + \mu)}\right]^{-\frac{\alpha + \mu}{2(\alpha + L)}} }{\alpha + \mu} \\
            &\le \frac{\alpha + \mu e^{-\frac{\alpha + \mu}{2(\alpha + L)}}}{\alpha + \mu} \label{eq:convergence_fedprox_constant_step_size_e_bound} \\
            &\le 1 - \frac{\mu}{3\alpha + 2L + \mu} \label{eq:convergence_fedprox_constant_step_size_e_bound_2}
    \end{align}
    We use fact $(1 + 1/x)^x \le e$ in \eqref{eq:convergence_fedprox_constant_step_size_e_bound}, and $e^{-x} \le 1/(x + 1)$ for any $x > -1$ in \eqref{eq:convergence_fedprox_constant_step_size_e_bound_2}. We further notice that $3\alpha + 2L + \mu \le 3(\alpha + L)$, and we discard the negative term $-A/(\gamma \mu)$ within the bound.
\end{proof}

\subsection{Main results for nonconvex analysis}

As we already did for strongly convex analysis, we begin by formulating the global progress made in a single round in a nonconvex scenario. In contrast to convex analysis, we state the convergence rate in terms of the average of squared gradients computed in each iteration.

\begin{lemma}[Single round progress of \textsc{FedProx} for nonconvex loss]
    \label{lemma:fedprox_nonconvex_single_round_global_progress}
    Assume that
    \begin{align}
        \gamma_{t} \le \min\left\{\frac{1}{L}, \frac{1}{2\alpha}\right\}
        \label{ass:fedprox_nonconvex_single_round_global_progress_step_size_assumption}
    \end{align}
    and Assumptions \ref{ass:full_participation} to \ref{ass:smoothness} and \ref{ass:lower_bounded_objective} hold. Then the global progress in a round  satisfies
    \begin{align}
        \frac{1}{E} \sum_{k = 0}^{E - 1} \expected \, \left\|\nabla \F{\overbar{\vb{w}}_{t, k}}\right\|^2 \le \frac{4}{\gamma_{t} E} \, \ev{\F{\overbar{\vb{w}}_{t, 0}} - \F{\overbar{\vb{w}}_{t + 1, 0}}} + A
    \end{align}
    where we define $A = 2\gamma_{t} L \sigma^2 \sum_{i = 1}^C p_i^2 + 2 \gamma_{t} \alpha E^2 G^2 + 8 \gamma_{t}^2 L^2 E^2 G^2 + \dfrac{\gamma_{t}^3 \alpha^2 L E^2 G^2}{2}$.
\end{lemma}

\begin{proof}
    Since the only property that we can exploit in the nonconvex analysis is smoothness, we apply its first-order characterization on iterates $\overbar{\vb{w}}_{t, k + 1}$ and $\overbar{\vb{w}}_{t, k}$.
    \begin{align}
        \F{\overbar{\vb{w}}_{t, k + 1}} - \F{\overbar{\vb{w}}_{t, k}} \le \nabla \F{\overbar{\vb{w}}_{t, k}}^{\transpose}\left(\overbar{\vb{w}}_{t, k + 1} - \overbar{\vb{w}}_{t, k}\right) + \frac{L}{2} \left\|\overbar{\vb{w}}_{t, k + 1} - \overbar{\vb{w}}_{t, k}\right\|^2
    \end{align}
    Leveraging the definition of the update rule, we have
    \begin{align}
        \overbar{\vb{w}}_{t, k + 1} - \overbar{\vb{w}}_{t, k} = \alpha\gamma_{t} \left(\overbar{\vb{w}}_{t, 0} - \overbar{\vb{w}}_{t, k} \right) - \gamma_{t} \sum_{i = 1}^C p_i \grad{\vb{w}_{t, k}^i}{i}
    \end{align}
    which we substitute to obtain
    \begin{align}
        \F{\overbar{\vb{w}}_{t, k + 1}} - \F{\overbar{\vb{w}}_{t, k}} &\le \underbrace{\alpha\gamma_{t} \nabla \F{\overbar{\vb{w}}_{t, k}}^{\transpose} \left(\overbar{\vb{w}}_{t, 0} - \overbar{\vb{w}}_{t, k} \right)}_{a_1} + \\
        &\fake{\le} \underbrace{-\gamma_{t} \left[ \sum_{i = 1}^C p_i \nabla \Fi{\vb{w}_{t, k}^i}{i} \right]^{\transpose} \nabla \F{\overbar{\vb{w}}_{t, k}}}_{a_2} + \\
        &\fake{\le} \underbrace{-\gamma_{t} \left[ \sum_{i = 1}^C p_i \left(\grad{\vb{w}_{t, k}^i}{i} - \nabla \Fi{\vb{w}_{t, k}^i}{i}\right) \right]^{\transpose} \nabla \F{\overbar{\vb{w}}_{t, k}}}_{\widetilde{a}_2} + \\
        &\fake{\le} \underbrace{\frac{L}{2}\left\|\alpha\gamma_{t} \left(\overbar{\vb{w}}_{t, 0} - \overbar{\vb{w}}_{t, k} \right) - \gamma_{t} \sum_{i = 1}^C p_i \grad{\vb{w}_{t, k}^i}{i}\right\|^2}_{a_3}
    \end{align}
    We add and subtract $ \sum_{i = 1}^C p_i \nabla \Fi{\vb{w}_{t, k}^i}{i}$ in terms $\widetilde{a}_2$ and $a_2$. When taking expectation over the previous expression, term $\widetilde{a}_2$ is erased because of Assumption \ref{ass:bounded_variance}. We use Peter-Paul's inequality to bound $a_1$.
    \begin{align}
        a_1 &\le \frac{\alpha \gamma_{t}^2}{2} \left\|\nabla \F{\overbar{\vb{w}}_{t, k}}\right\|^2 + \frac{\alpha}{2}\left\|\overbar{\vb{w}}_{t, 0} - \overbar{\vb{w}}_{t, k}\right\|^2
    \end{align}
    Leveraging the law $2\vb{u}^{\transpose}\vb{v} = \|\vb{u}\|^2 + \|\vb{v}\|^2 - \|\vb{u} - \vb{v}\|^2$, we rewrite $a_2$ as follows.
    \begin{align}
        a_2 &= -\gamma_{t} \left[ \sum_{i = 1}^C p_i \nabla \Fi{\vb{w}_{t, k}^i}{i} \right]^{\transpose} \nabla \F{\overbar{\vb{w}}_{t, k}} \\
        &\le \frac{\gamma_{t}}{2}\left\|\sum_{i = 1}^C p_i \nabla \Fi{\vb{w}_{t, k}^i}{i} - \nabla \F{\overbar{\vb{w}}_{t, k}}\right\|^2 + \frac{\gamma_{t}}{2}\left\|\sum_{i = 1}^C p_i \nabla \Fi{\vb{w}_{t, k}^i}{i}\right\|^2 \\
        &\fake{\le} -\frac{\gamma_{t}}{2} \left\|\nabla \F{\overbar{\vb{w}}_{t, k}}\right\|^2 \\
        &= \frac{\gamma_{t}}{2}\left\|\sum_{i = 1}^C p_i \left( \nabla \Fi{\vb{w}_{t, k}^i}{i} - \nabla \Fi{\overbar{\vb{w}}_{t, k}}{i}\right)\right\|^2 - \frac{\gamma_{t}}{2}\left\|\sum_{i = 1}^C p_i \nabla \Fi{\vb{w}_{t, k}^i}{i}\right\|^2 + \\
        &\fake{\le} -\frac{\gamma_{t}}{2} \left\|\nabla \F{\overbar{\vb{w}}_{t, k}}\right\|^2 \label{eq:lemma_fedprox_progress_non_convex_average_gradient_decomposition} \\
        &\le \frac{\gamma_{t}}{2}\sum_{i = 1}^C p_i \left\|\nabla \Fi{\vb{w}_{t, k}^i}{i} - \nabla \Fi{\overbar{\vb{w}}_{t, k}}{i}\right\|^2 - \frac{\gamma_{t}}{2}\left\|\sum_{i = 1}^C p_i \nabla \Fi{\vb{w}_{t, k}^i}{i}\right\|^2 + \\
        &\fake{\le} -\frac{\gamma_{t}}{2} \left\|\nabla \F{\overbar{\vb{w}}_{t, k}}\right\|^2 \label{eq:lemma_fedprox_progress_non_convex_average_gradient_jensen} \\
        &\le \frac{\gamma_{t} L^2}{2}\sum_{i = 1}^C p_i \left\|\vb{w}_{t, k}^i - \overbar{\vb{w}}_{t, k}\right\|^2 - \frac{\gamma_{t}}{2}\left\|\sum_{i = 1}^C p_i \nabla \Fi{\vb{w}_{t, k}^i}{i}\right\|^2 + \\
        &\fake{\le} - \frac{\gamma_{t}}{2} \left\|\nabla \F{\overbar{\vb{w}}_{t, k}}\right\|^2
        \label{eq:lemma_fedprox_progress_non_convex_lipschitz_gradient} 
    \end{align}
    We consider that $\nabla \F{\cdot} = \sum_{i = 1}^C p_i \nabla \Fi{\cdot}{i}$ in \eqref{eq:lemma_fedprox_progress_non_convex_average_gradient_decomposition}, and we leverage Jensen's inequality in \eqref{eq:lemma_fedprox_progress_non_convex_average_gradient_jensen}.
    Additionally, we use the Lipschitz gradient property in equation \eqref{eq:lemma_fedprox_progress_non_convex_lipschitz_gradient} due to the smoothness of the objectives.
    We bound $a_3$ under expectation by applying Peter-Paul's inequality. We use the same strategy from the proof of Lemma \ref{lemma:single_round_global_progress} to bound the squared sum of local stochastic gradients. On the other hand, $a_3$ is bounded in expectation.
    \begin{align}
        \expected \, a_3 &\le \frac{\gamma_{t}^2 L}{2} \, \expected \, \left\| \sum_{i = 1}^C p_i \grad{\vb{w}_{t, k}^i}{i}\right\|^2 + \frac{\gamma_{t}^2 \alpha^2 L}{8} \, \expected \, \left\|\overbar{\vb{w}}_{t, 0} - \overbar{\vb{w}}_{t, k}\right\|^2 \\
        &=  \frac{\gamma_{t}^2 L}{2} \, \expected \, \left\| \sum_{i = 1}^C p_i \left(\grad{\vb{w}_{t, k}^i}{i} - \Fi{\vb{w}_{t, k}^i}{i}\right) + \sum_{i = 1}^C p_i \Fi{\vb{w}_{t, k}^i}{i}\right\|^2 + \\ &\fake{\le} \frac{\gamma_{t}^2 \alpha^2 L}{8} \, \expected \, \left\|\overbar{\vb{w}}_{t, 0} - \overbar{\vb{w}}_{t, k}\right\|^2 \\
        &= \frac{\gamma_{t}^2 L}{2} \, \ev{\left\| \sum_{i = 1}^C p_i \left(\grad{\vb{w}_{t, k}^i}{i} - \Fi{\vb{w}_{t, k}^i}{i}\right) \right\|^2 + \left\| \sum_{i = 1}^C p_i \Fi{\vb{w}_{t, k}^i}{i}\right\|^2} + \\ &\fake{\le} \frac{\gamma_{t}^2 \alpha^2 L}{8} \, \expected \, \left\|\overbar{\vb{w}}_{t, 0} - \overbar{\vb{w}}_{t, k}\right\|^2 \\
        &=  \frac{\gamma_{t}^2 L}{2} \, \ev{\sum_{i = 1}^C p_i^2 \left\| \grad{\vb{w}_{t, k}^i}{i} - \Fi{\vb{w}_{t, k}^i}{i} \right\|^2 + \left\| \sum_{i = 1}^C p_i \Fi{\vb{w}_{t, k}^i}{i}\right\|^2} + \\ &\fake{\le} \frac{\gamma_{t}^2 \alpha^2 L}{8} \, \expected \, \left\|\overbar{\vb{w}}_{t, 0} - \overbar{\vb{w}}_{t, k}\right\|^2 \\
        &\le  \frac{\gamma_{t}^2 L \sigma^2}{2} \sum_{i = 1}^C p_i^2 + \frac{\gamma_{t}^2 L}{2} \, \expected  \left\| \sum_{i = 1}^C p_i \nabla \Fi{\vb{w}_{t, k}^i}{i} \right\|^2 + \frac{\gamma_{t}^2 \alpha^2 L}{8} \, \expected \, \left\|\overbar{\vb{w}}_{t, 0} - \overbar{\vb{w}}_{t, k}\right\|^2
    \end{align}
    Under expectation, we combine all the bounds in the main expression.
    \begin{align}
        \ev{\F{\overbar{\vb{w}}_{t, k + 1}} - \F{\overbar{\vb{w}}_{t, k}}} &\le -\frac{\gamma_{t}(1 - L\gamma_{t})}{2} \, \expected \, \left\| \sum_{i = 1}^C p_i \nabla \Fi{\vb{w}_{t, k}^i}{i} \right\|^2 + \\
        &\fake{\le} - \frac{\gamma_{t}(1 - \alpha\gamma_{t})}{2}\, \expected \, \left\|\nabla \F{\overbar{\vb{w}}_{t, k}}\right\|^2 + \\
        &\fake{\le} \underbrace{\frac{\alpha}{2}\left(1 + \frac{\alpha L \gamma_{t}^2}{4} \right) \expected \, \left\|\overbar{\vb{w}}_{t, 0} - \overbar{\vb{w}}_{t, k}\right\|^2}_{b_1} + \\
        &\fake{\le} \underbrace{\frac{\gamma_{t} L^2}{2} \sum_{i = 1}^C p_i \, \expected \, \left\|\overbar{\vb{w}}_{t, k} - \vb{w}_{t, k}^i\right\|^2}_{b_2} + \frac{\gamma_{t}^2 L \sigma^2}{2} \sum_{i = 1}^C p_i^2
    \end{align}
    Due to Assumption \eqref{ass:fedprox_nonconvex_single_round_global_progress_step_size_assumption}, we have that $-\gamma_{t}(1 - L\gamma_{t}) \le 0$, and $-\gamma_{t}(1 - \alpha\gamma_{t}) \le -1/2$. Furthermore, to bound $b_1$, we apply Jensen's inequality on $\|\cdot\|^2$, and we use Lemma \ref{lemma:single_round_local_deviation}.
    \begin{align}
        b_1 &= \frac{\alpha}{2}\left(1 + \frac{\gamma_{t}^2 \alpha L}{4} \right) \expected \, \left\|\overbar{\vb{w}}_{t, 0} - \sum_{i = 1}^C p_i \vb{w}_{t, k}^i\right\|^2 \\
        &\le \frac{\alpha}{2}\left(1 + \frac{\gamma_{t}^2 \alpha L}{4} \right) \sum_{i = 1}^C p_i \, \expected \, \left\|\overbar{\vb{w}}_{t, 0} - \vb{w}_{t, k}^i\right\|^2 \\
        &\le \frac{\gamma_{t}^2 \alpha  E^2 G^2}{2}\left(1 + \frac{\gamma_{t}^2 \alpha L}{4} \right)
    \end{align}
    Concerning $b_2$, we use Lemma \ref{lemma:single_round_local_divergence}, therefore $b_2 \le 2\gamma_{t}^3 L^2 E^2 G^2$. Eventually, we attain
    \begin{align}
        \ev{\F{\overbar{\vb{w}}_{t, k + 1}} - \F{\overbar{\vb{w}}_{t, k}}} &\le 
        \underbrace{\frac{\alpha \gamma_{t}^2 E^2 G^2}{2}\left(1 + \frac{\alpha L \gamma_{t}^2}{4} \right) + 2\gamma_{t}^3 L^2 E^2 G^2 + \frac{\gamma_{t}^2 L \sigma^2}{2} \sum_{i = 1}^C p_i^2}_{c} + \\
        &\fake{\le} - \frac{\gamma_{t}}{4} \, \expected \, \left\|\nabla \F{\overbar{\vb{w}}_{t, k}}\right\|^2
    \end{align}
    Therefore, we swap the terms and sum over $k$ from $0$ to $E - 1$.
    \begin{align}
        \sum_{k = 0}^{E - 1} \, \expected \, \left\|\nabla \F{\overbar{\vb{w}}_{t, k}}\right\|^2 \le \frac{4}{\gamma_{t}}\, \ev{\F{\overbar{\vb{w}}_{t, 0}} - \F{\overbar{\vb{w}}_{t, E}} } + \frac{4Ec}{\gamma_{t}}
    \end{align}
    We highlight that $\overbar{\vb{w}}_{t, E} \equiv \overbar{\vb{w}}_{t + 1, 0}$. Dividing by $E$ (local steps) concludes our proof.
\end{proof}

The following lemma presents the general convergence guarantee when adopting a fixed step size. As we can observe, choosing the latter is fundamental to balancing the magnitude of the two additive terms in the bound.

\ConvergenceFedproxNonConvex*

\begin{proof}
    In the first place, using fixed step size $\gamma$, we leverage the result of lemma \ref{lemma:fedprox_nonconvex_single_round_global_progress} by summing both sides for $t = 0, 1, \ldots, T - 1$ and dividing by $T$. Then, we use Assumption \ref{ass:lower_bounded_objective} to state that $\F{\overbar{\vb{w}}_{0, 0}} - \F{\overbar{\vb{w}}_{T, 0}} \le \F{\overbar{\vb{w}}_{0, 0}} - \Finfimum$. Finally, we only replace $\gamma$ with the chosen step size. The definition of $\widehat{\vb{w}}_{T}$ ensures that
        \begin{align}
            \expected \, \left\|\nabla \F{\widehat{\vb{w}}_{T}}\right\|^2 &= \frac{1}{T E} \sum_{t = 0}^{T - 1} \sum_{k = 0}^{E - 1} \expected \, \left\|\nabla \F{\overbar{\vb{w}}_{t, k}}\right\|^2
        \end{align}
    which concludes the proof.
\end{proof}

\section{Proof of Theorem \ref{theorem:convergence_our_algorithm_strongly_convex} and \ref{theorem:convergence_our_algorithm_nonconvex}}
\label{appendix:analysis_of_our_algorithm}

We include in this appendix all the results and missing proofs related to the study of our algorithm.

\subsection{Preliminary results}

This first technical fact will support us in stating future claims on our algorithm.

\OurAlgorithmAveragePerturbedIterate*

\begin{proof}
    We recall definition \eqref{eq:our_algorithm_perturbed_iterate_definition} and the fact that $p_j \eqdef \sum_{i = 1}^C p_{ij}$ to prove our statement.
    \begin{align}
        \sum_{i = 1}^C p_{i} \widetilde{\vb{w}}_{t, k}^i &= \beta \overbar{\vb{w}}_{t, k} + (1 - \beta) \sum_{i = 1}^C \sum_{j = 1}^C p_{ij} \vb{w}_{t - 1, E}^j \\
        &= \beta \overbar{\vb{w}}_{t, k} + (1 - \beta) \sum_{j = 1}^C \sum_{i = 1}^C p_{ij} \vb{w}_{t - 1, E}^j \\
        &= \beta \overbar{\vb{w}}_{t, k} + (1 - \beta) \sum_{j = 1}^C p_{j} \vb{w}_{t - 1, E}^j \\
        &= \beta \overbar{\vb{w}}_{t, k} + (1 - \beta) \overbar{\vb{w}}_{t, 0}
    \end{align}
    This last statement concludes the proof.
\end{proof}

Furthermore, we delimitate the difference between the iterates used to perturb the local computation of stochastic gradients.

\begin{lemma}
    \label{lemma:our_algorithm_lemma_u_difference}
    At round $t$, the deviation between $\vb{u}_{t}^i$ and $\vb{u}_{t}^j$ follows the rule
    \begin{align}
        \expected \, \left\|\vb{u}_{t}^i - \vb{u}_{t}^j \right\|^2 \le \mathds{1}_{t \, \ge \, 1} 4 \gamma_{t - 1}^2 E^2 G^2
    \end{align}
    for any pair of clients $i, j \in \mathcal{C}$. In addition, assume \ref{ass:full_participation} to \ref{ass:smoothness} hold.
\end{lemma}

\begin{proof}
    The first case, when $t = 0$ is trivial, since $\vb{u}_{t}^i  = \overbar{\vb{w}}_{0, 0}$ for every client $i$. Therefore, the deviation between $\vb{u}_{t}^i$ and $\vb{u}_{t}^j$ would be zero. For $t \ge 1$, we introduce variable $\overline{\vb{w}}_{t - 1, 0}$ and we indicate the aforementioned deviation as $\Delta u_{t}^{ij}$.
    \begin{align}
        \Delta u_{t}^{ij} &= \left\|\vb{u}_{t}^i - \vb{u}_{t}^j \right\|^2 \\
        &= \left\|\vb{u}_{t}^i - \overline{\vb{w}}_{t - 1, 0} + \overline{\vb{w}}_{t - 1, 0} -\vb{u}_{t}^j \right\|^2 \\
        &\le 2\left\|\vb{u}_{t}^i - \overline{\vb{w}}_{t - 1, 0}\right\|^2 + 2\left\|\vb{u}_{t}^j - \overline{\vb{w}}_{t - 1, 0}\right\|^2 \label{eq:lemma_u_u_peter_paul} \\
        &\le 2\left\|\frac{1}{p_i} \sum_{l = 1}^C p_{il} \vb{w}_{t - 1, E}^l - \overline{\vb{w}}_{t - 1, 0}\right\|^2 + 2\left\|\frac{1}{p_j} \sum_{l = 1}^C p_{jl} \vb{w}_{t - 1, E}^l - \overline{\vb{w}}_{t - 1, 0}\right\|^2 \label{eq:lemma_u_u_def_replacement} \\
        &\le \frac{2}{p_i} \sum_{l = 1}^C p_{il} \left\|\vb{w}_{t - 1, E}^l - \overline{\vb{w}}_{t - 1, 0}\right\|^2 + \frac{2}{p_j} \sum_{l = 1}^C p_{jl} \left\|\vb{w}_{t - 1, E}^l - \overline{\vb{w}}_{t - 1, 0}\right\|^2 \label{eq:lemma_u_u_convexity_1}
    \end{align}
    where we use Young's inequality in equation \eqref{eq:lemma_u_u_peter_paul}, definition \eqref{eq:our_algorithm_perturbed_iterate_definition} in \eqref{eq:lemma_u_u_def_replacement}, and Jensen's inequality in \eqref{eq:lemma_u_u_convexity_1}. We replace $\vb{w}_{t - 1, E}^l - \overline{\vb{w}}_{t - 1, 0}$ in equation \eqref{eq:lemma_u_u_recursion} using recursion.
    \begin{align}
        \Delta u_{t}^{ij} &\le \frac{2}{p_i} \sum_{l = 1}^C p_{il} \left\|
        -\gamma_{t - 1} \sum_{k = 0}^{E - 1} \grad{\widetilde{\vb{w}}_{t, k}^l}{l}
        \right\|^2 + \frac{2}{p_j} \sum_{l = 1}^C p_{jl} \left\|
        -\gamma_{t - 1} \sum_{k = 0}^{E - 1} \grad{\widetilde{\vb{w}}_{t, k}^l}{l}
        \right\|^2 \label{eq:lemma_u_u_recursion} \\
        &\le \frac{2 \gamma_{t - 1}^2 E}{p_i} \sum_{l = 1}^C p_{il}
        \sum_{k = 0}^{E - 1}  \left\|\grad{\widetilde{\vb{w}}_{t, k}^l}{l}
        \right\|^2 + \frac{2 \gamma_{t - 1}^2 E}{p_j} \sum_{l = 1}^C p_{jl}
        \sum_{k = 0}^{E - 1}  \left\|\grad{\widetilde{\vb{w}}_{t, k}^l}{l}
        \right\|^2 \label{eq:lemma_u_u_convexity_2}
    \end{align}
    We leverage Jensen's inequality in \eqref{eq:lemma_u_u_convexity_2}. To conclude, we have $\expected \, \Delta u_{t}^{ij} \le 4 \gamma_{t - 1}^2 E^2 G^2$ under expectation using Assumption \ref{ass:bounded_stochastic_norm}.
\end{proof}

Such a result lets us upper bound the deviation between the average iterate and the local perturbed one for each client.

\begin{lemma}
    \label{lemma:our_algorithm_perturbed_iterates_difference}
    The deviation between $\overbar{\vb{w}}_{t, k}$ and $\widetilde{\vb{w}}_{t, k}^i$ is bounded as
    \begin{align}
        \expected \, \left\|\overbar{\vb{w}}_{t, k} - \widetilde{\vb{w}}_{t, k}^i\right\|^2 \le 4 \gamma_{t}^2 E^2 G^2 \left[4 + (1 - \beta)^2 + \mathds{1}_{t \, \ge \, 1} \frac{8\gamma_{t - 1}^2}{\gamma_{t}^2} \left(1 - \frac{1}{\beta}\right)^2 \right]
    \end{align}
    for any client $i \in \mathcal{C}$ at step $k$ of round $t$. Moreover, assume \ref{ass:full_participation} to \ref{ass:smoothness} hold.
\end{lemma}

\begin{proof}
    Denoting $\widetilde{D}_{t, k}^i = \left\|\overbar{\vb{w}}_{t, k} - \widetilde{\vb{w}}_{t, k}^i\right\|^2$, we use our Lemma \ref{lemma:our_algorithm_average_perturbed_iterate} to replace $\overbar{\vb{w}}_{t, k}$ in \eqref{eq:lemma_w_wilde_w_t_k_replacement}.
    \begin{align}
        \widetilde{D}_{t, k}^i &= \left\|\frac{1}{\beta} \sum_{j = 1}^C p_{j} \left(\widetilde{\vb{w}}_{t, k}^j - \widetilde{\vb{w}}_{t, k}^i\right) + \left(1 - \frac{1}{\beta}\right) \left(\overbar{\vb{w}}_{t, 0} - \widetilde{\vb{w}}_{t, k}^i \right) \right\|^2 \label{eq:lemma_w_wilde_w_t_k_replacement} \\
        &\le \frac{2}{\beta^2}\left\|\sum_{j = 1}^C p_{j} \left(\widetilde{\vb{w}}_{t, k}^j - \widetilde{\vb{w}}_{t, k}^i\right)\right\|^2 +  2\left(1 - \frac{1}{\beta}\right)^2 \left\| \sum_{j = 1}^C p_{j} \widetilde{\vb{w}}_{t, 0}^j - \widetilde{\vb{w}}_{t, k}^i \right\|^2 \label{eq:lemma_w_wilde_w_t_0_replacement} \\
        &\le \frac{2}{\beta^2} \sum_{j = 1}^C p_{j} \underbrace{\left\|\widetilde{\vb{w}}_{t, k}^j - \widetilde{\vb{w}}_{t, k}^i\right\|^2}_{A} + 2\left(1 - \frac{1}{\beta}\right)^2 \sum_{j = 1}^C p_{j} \underbrace{\left\| \widetilde{\vb{w}}_{t, 0}^j - \widetilde{\vb{w}}_{t, k}^i \right\|^2}_{B}
    \end{align}
    Again, using \ref{lemma:our_algorithm_average_perturbed_iterate} with $k = 0$, we leverage the fact that $\overbar{\vb{w}}_{t, 0} = \sum_{j = 1}^C p_{j} \widetilde{\vb{w}}_{t, 0}^j$ to rewrite $\overbar{\vb{w}}_{t, 0}$ in expression \eqref{eq:lemma_w_wilde_w_t_0_replacement}. We first bound term $A$ as
    \begin{align}
        A &= \left\| \beta \left(\vb{w}_{t, k}^j - \vb{w}_{t, k}^i\right) + (1 - \beta)\left(\vb{u}_{t}^j - \vb{u}_{t}^i\right) \right\|^2 \\
        &= \left\| -\gamma_{t} \beta \sum_{m = 0}^{k - 1} \left(\grad{\widetilde{\vb{w}}_{t, m}^j}{j} - \grad{\widetilde{\vb{w}}_{t, m}^i}{i}\right) + (1 - \beta)\left(\vb{u}_{t}^j - \vb{u}_{t}^i\right) \right\|^2 \label{eq:lemma_w_tilde_A_bound_recursion} \\
        &\le 2 \gamma_{t}^2 \beta^2 \left\| \sum_{m = 0}^{k - 1} \left(\grad{\widetilde{\vb{w}}_{t, m}^j}{j} - \grad{\widetilde{\vb{w}}_{t, m}^i}{i}\right) \right\|^2 + 2(1 - \beta)^2 \left\|\vb{u}_{t}^j - \vb{u}_{t}^i\right\|^2 \label{eq:lemma_w_tilde_A_bound_young_1} \\
        &\le 2 \gamma_{t}^2 \beta^2 k \sum_{m = 0}^{k - 1} \left\|\grad{\widetilde{\vb{w}}_{t, m}^j}{j} - \grad{\widetilde{\vb{w}}_{t, m}^i}{i} \right\|^2 + 2(1 - \beta)^2 \left\|\vb{u}_{t}^j - \vb{u}_{t}^i\right\|^2 \label{eq:lemma_w_tilde_A_bound_jensen} \\
        &\le 4 \gamma_{t}^2 \beta^2 k \sum_{m = 0}^{k - 1} \left[ \left\|\grad{\widetilde{\vb{w}}_{t, m}^j}{j}\right\|^2 + \left\|\grad{\widetilde{\vb{w}}_{t, m}^i}{i} \right\|^2 \right] + 2(1 - \beta)^2 \left\|\vb{u}_{t}^j - \vb{u}_{t}^i\right\|^2 \label{eq:lemma_w_tilde_A_bound_young_2}
    \end{align}
    using recursion on the update rule in expression \eqref{eq:lemma_w_tilde_A_bound_recursion}, Young's inequality in \eqref{eq:lemma_w_tilde_A_bound_young_1}, Jensen's inequality in \eqref{eq:lemma_w_tilde_A_bound_jensen}, Young's inequality again in \eqref{eq:lemma_w_tilde_A_bound_young_2}. Eventually, we recall Assumption \ref{ass:bounded_stochastic_norm} and the result of Lemma \ref{lemma:our_algorithm_lemma_u_difference} as well as the fact that $k \le E$ to bound $\expected \, A \le 8 \gamma_{t}^2 \beta^2 E^2 G^2 + \mathds{1}_{t \, \ge \, 1} 8 \gamma_{t - 1}^2 (1 - \beta)^2 E^2 G^2$.
    Let us focus on term $B$.
    \begin{align}
        B &= \left\| \gamma_{t} \beta \sum_{m = 0}^{k - 1} \grad{\widetilde{\vb{w}}_{t, m}^i}{i} + (1 - \beta)\left( \vb{u}_{t}^j - \vb{u}_{t}^i \right) 
        \right\|^2 \label{eq:lemma_w_tilde_B_bound_updare_rule_recursion} \\
        &\le 2\gamma_{t}^2 \beta^2  \left\| \sum_{m = 0}^{k - 1} \grad{\widetilde{\vb{w}}_{t, m}^i}{i} \right\|^2 + 2(1 - \beta)^2 \left\| \vb{u}_{t}^j - \vb{u}_{t}^i 
        \right\|^2 \label{eq:lemma_w_tilde_B_bound_young} \\
        &\le 2\gamma_{t}^2 \beta^2 k \sum_{m = 0}^{k - 1} \left\| \grad{\widetilde{\vb{w}}_{t, m}^i}{i} \right\|^2 + 2(1 - \beta)^2 \left\| \vb{u}_{t}^j - \vb{u}_{t}^i 
        \right\|^2 \label{eq:lemma_w_tilde_B_bound_jensen_2}
    \end{align}
    In equation \eqref{eq:lemma_w_tilde_B_bound_updare_rule_recursion}, using definition \eqref{eq:our_algorithm_perturbed_iterate_definition}, we exploit the fact that
    \begin{align}
        \widetilde{\vb{w}}_{t, 0}^j &= \beta \overbar{\vb{w}}_{t, 0} + (1 - \beta) \vb{u}_{t}^j \\
        \widetilde{\vb{w}}_{t, k}^i &= \beta \overbar{\vb{w}}_{t, 0} - \gamma_{t} \beta \sum_{m = 0}^{k - 1} \grad{\widetilde{\vb{w}}_{t, m}^i}{i} + (1 - \beta) \vb{u}_{t}^i
    \end{align}
    In addition, we use Young's inequality in equation \eqref{eq:lemma_w_tilde_B_bound_young}, again Jensen's inequality in \eqref{eq:lemma_w_tilde_B_bound_jensen_2}, and fact $k \le E$. Finally, using Assumption \ref{ass:bounded_stochastic_norm} and result of Lemma \ref{lemma:our_algorithm_lemma_u_difference}, we are able to bound term $\expected \, B \le 2 \gamma_{t}^2 \beta^2 E^2 G^2 + \mathds{1}_{t \, \ge \, 1} 8 \gamma_{t - 1}^2 (1 - \beta)^2 E^2 G^2$. Combining the bounds on $A$ and $B$ together in the main expression gives
    \begin{align}
        \expected \, \widetilde{D}_{t, k}^i &\le 16 \gamma_{t}^2 E^2 G^2 + \frac{\mathds{1}_{t \, \ge \, 1} 16 \gamma_{t - 1}^2 E^2 G^2}{\beta^2} \left[ (1 - \beta)^2 + (1 - \beta)^4 \right] + 4 \gamma_{t}^2 (1 - \beta)^2 E^2 G^2
    \end{align}
    Using approximation $(1 - \beta)^4 \le (1 - \beta)^2$ since $\beta \in (0, 1)$, we obtain the desired result.
\end{proof}

\subsection{Main results for strongly convex analysis}

The subsequent lemma eventually presents the progress made by our algorithm in a single round of communication in a strongly convex scenario. In this respect, parameter $\beta$ heavily impacts the contraction of the distance measure. Moreover, it also controls the growth of term $A$ since the choice $\beta = 1$ (as in \textsc{FedAvg}) nullifies two potentially significant terms depending on it. 

\begin{lemma}[Single round progress of our algorithm for strongly convex loss]
    \label{lemma:our_algorithm_single_round_global_progress}
        Assume
    \begin{align}
        \label{lemma:our_algorithm_one_step_progress_step_decay}
        \gamma_{t} \le \min{\left\{ \frac{1}{2L}, \frac{1}{\beta\mu} \right\}}
    \end{align}
    and \ref{ass:full_participation} to \ref{ass:smoothness} hold, then the progress in one global round satisfies
    \begin{align}
        \expected \, \left\|\overbar{\vb{w}}_{t + 1, 0} - \vb{w}_{\star}\right\|^2 \le \kappa \, \expected \, \left\|\overbar{\vb{w}}_{t, 0} - \vb{w}_{\star}\right\|^2 + A
    \end{align}
    where $ \kappa = 1 - \dfrac{1}{\beta} + \dfrac{1}{\beta}\left(1 - \beta\mu\gamma_{t}\right)^E \le 1 - \mu\gamma_{t}$, and
    \begin{align*}
        A &=  8 \gamma_{t}^2 E^3 G^2 \left[4 + (1 - \beta)^2 + \frac{\mathds{1}_{t \ge 1}8\gamma_{t - 1}^2(1 - \beta)^2}{\gamma_{t}^2 \beta^2} \right] + \mu\gamma_{t}^3\beta(1 - \beta)E^3 G^2 + \gamma_{t}^2 E S \sigma^2 + 6 \gamma_{t}^2 L E \Gamma
    \end{align*}
\end{lemma}

\begin{proof}
    To begin our proof, we denote $\left\|\overbar{\vb{w}}_{t, k} - \vb{w}_{\star}\right\|^2$ as $D_{t, k}$, and we recall the definition of update rule for the average sequence in the following equation \eqref{eq:our_algorithm_single_round_global_progress_average_rule_replacement}.
    \begin{align}
        \overbar{\vb{w}}_{t, k + 1} - \vb{w}_{\star} &= \overbar{\vb{w}}_{t, k} - \vb{w}_{\star} - \gamma_{t} \sum_{i = 1}^C p_{i} \grad{\widetilde{\vb{w}}_{t, k}^i}{i} \label{eq:our_algorithm_single_round_global_progress_average_rule_replacement} \\
        &= \overbar{\vb{w}}_{t, k} - \vb{w}_{\star} - \gamma_{t} \sum_{i = 1}^C p_{i} \nabla \Fi{\widetilde{\vb{w}}_{t, k}^i}{i} + \underbrace{\gamma_{t} \sum_{i = 1}^C p_{i} \left[ \nabla \Fi{\widetilde{\vb{w}}_{t, k}^i}{i} - \grad{\widetilde{\vb{w}}_{t, k}^i}{i} \right]}_{\vb{v}}
    \end{align}
    Since $\expected \, \vb{v} = 0$ because of Assumption \ref{ass:bounded_variance}, when we take expectation over the squared norm $D_{t, k}$, all mixed products $2 \vb{v}^{\transpose}\vb{u}$ are erased in expectation. Hence, we have
    \begin{align}
        \expected \, D_{t, k + 1} &= \expected \, D_{t, k} + \ev{\underbrace{- 2\gamma_{t} \left(\overbar{\vb{w}}_{t, k} - \vb{w}_{\star}\right)^{\transpose}\left(\sum_{i = 1}^C p_{i}  \nabla \Fi{\widetilde{\vb{w}}_{t, k}^i}{i}\right)}_{a_1}} + \\
        &\fake{=} \ev{\underbrace{\gamma_{t}^2 \left\|\sum_{i = 1}^C p_{i} \nabla \Fi{\widetilde{\vb{w}}_{t, k}^i}{i}\right\|^2 + \gamma_{t}^2 \left\| \sum_{i = 1}^C p_{i} \left[ \nabla \Fi{\widetilde{\vb{w}}_{t, k}^i}{i} - \grad{\widetilde{\vb{w}}_{t, k}^i}{i} \right] \right\|^2}_{a_2}}
    \end{align}
    First we bound term $a_{1}$ in expectation
    \begin{align}
        a_{1} &= - 2\gamma_{t} \left(\overbar{\vb{w}}_{t, k} - \vb{w}_{\star}\right)^{\transpose}\left(\sum_{i = 1}^C p_{i} \nabla \Fi{\widetilde{\vb{w}}_{t, k}^i}{i}\right) \label{eq:lemma_progress_our_algorithm_use_of_unbiased_gradient} \\
        &= - 2\gamma_{t} \sum_{i = 1}^C p_{i} \nabla \Fi{\widetilde{\vb{w}}_{t, k}^i}{i}^{\transpose} \left(\overbar{\vb{w}}_{t, k} - \vb{w}_{\star}\right) \\
        &= - 2\gamma_{t} \sum_{i = 1}^C p_{i} \nabla \Fi{\widetilde{\vb{w}}_{t, k}^i}{i}^{\transpose} \left(\overbar{\vb{w}}_{t, k} - \widetilde{\vb{w}}_{t, k}^i\right) + \\
        &\fake{=} - 2\gamma_{t} \sum_{i = 1}^C p_{i} \nabla \Fi{\widetilde{\vb{w}}_{t, k}^i}{i}^{\transpose} \left( \widetilde{\vb{w}}_{t, k}^i - \vb{w}_{\star}\right) \\
        &\le \gamma_{t} \sum_{i = 1}^C p_{i} \left[ \gamma_{t}\left\|\nabla \Fi{\widetilde{\vb{w}}_{t, k}^i}{i}\right\|^2 + \frac{1}{\gamma_{t}}  \left\|\overbar{\vb{w}}_{t, k} - \widetilde{\vb{w}}_{t, k}^i\right\|^2 \right] + \label{eq:lemma_progress_our_algorithm_use_of_peter_paul} \\
        &\fake{\le} 2\gamma_{t} \sum_{i = 1}^C p_{i} \left[ \Fi{\vb{w}_{\star}}{i} - \Fi{\widetilde{\vb{w}}_{t, k}^i}{i} - \frac{\mu}{2}\left\|\widetilde{\vb{w}}_{t, k}^i - \vb{w}_{\star}\right\|^2 \right] \label{eq:lemma_progress_our_algorithm_use_of_convexity} \\
        &\le 2L\gamma_{t}^2 \sum_{i = 1}^C p_{i} \left(\Fi{\widetilde{\vb{w}}_{t, k}^i}{i} - \Fi{\vb{w}_{\star}^i}{i}\right) + \underbrace{\sum_{i = 1}^C p_{i} \left\|\overbar{\vb{w}}_{t, k} - \widetilde{\vb{w}}_{t, k}^i\right\|^2}_{a_{12}} + \label{eq:lemma_progress_our_algorithm_use_of_smoothness} \\
        &\fake{\le} 2\gamma_{t} \sum_{i = 1}^C p_{i} \left(\Fi{\vb{w}_{\star}}{i} - \Fi{\widetilde{\vb{w}}_{t, k}^i}{i}\right) \underbrace{- \mu\gamma_{t} \sum_{i = 1}^C p_{i}\left\|\widetilde{\vb{w}}_{t, k}^i - \vb{w}_{\star}\right\|^2}_{a_{11}}
    \end{align}
    where we use Assumption \ref{ass:bounded_variance} in equation \eqref{eq:lemma_progress_our_algorithm_use_of_unbiased_gradient}, Peter-Paul's inequality in equation \eqref{eq:lemma_progress_our_algorithm_use_of_peter_paul}, strong convexity in equation \eqref{eq:lemma_progress_our_algorithm_use_of_convexity}, and smoothness in  \eqref{eq:lemma_progress_our_algorithm_use_of_smoothness}. Addressing $a_{11}$, we use Jensen's inequality in equation \eqref{eq:lemma_our_algorithm_progress_jensen_0}, the result of Lemma \ref{lemma:our_algorithm_average_perturbed_iterate} in equation \eqref{eq:lemma_our_algorithm_progress_lemma_average_perturbed_replacement}, and the rule $-2\vb{u}^{\transpose}\vb{v} = \left\|\vb{u} - \vb{v}\right\|^2 - \left\|\vb{u}\right\|^2 - \left\|\vb{v}\right\|^2$ in equation \eqref{eq:lemma_our_algorithm_progress_parallelogram_law}.
    \begin{align}
        a_{11} &\le -\mu\gamma_{t} \left\| \sum_{i = 1}^C p_{i} \widetilde{\vb{w}}_{t, k}^i - \vb{w}_{\star} \right\|^2 \label{eq:lemma_our_algorithm_progress_jensen_0} \\
        &\le -\mu\gamma_{t} \left\| \beta \overbar{\vb{w}}_{t, k} + (1 - \beta) \overbar{\vb{w}}_{t, 0} - \vb{w}_{\star} \right\|^2 \label{eq:lemma_our_algorithm_progress_lemma_average_perturbed_replacement}\\
        &= -\mu\gamma_{t} \left\| \beta \left( \overbar{\vb{w}}_{t, k}  - \vb{w}_{\star} \right) + (1 - \beta)  \left(\overbar{\vb{w}}_{t, 0} - \vb{w}_{\star} \right) \right\|^2 \\
        &= -\mu\gamma_{t}\beta^2 D_{t, k} - \mu\gamma_{t}(1 - \beta)^2 D_{t, 0} - 2\mu\gamma_{t}\beta(1 - \beta) \left( \overbar{\vb{w}}_{t, k}  - \vb{w}_{\star} \right)^{\transpose} \left( \overbar{\vb{w}}_{t, 0}  - \vb{w}_{\star} \right) \\
        &= -\mu\gamma_{t}\beta^2 D_{t, k} - \mu\gamma_{t}(1 - \beta)^2 D_{t, 0} + \\
        &\fake{=} - \mu\gamma_{t}\beta(1 - \beta) D_{t, k} - \mu\gamma_{t}\beta(1 - \beta) D_{t, 0} +  \mu\gamma_{t}\beta(1 - \beta) \left\| \overbar{\vb{w}}_{t, k}  - \overbar{\vb{w}}_{t, 0} \right\|^2 \label{eq:lemma_our_algorithm_progress_parallelogram_law} \\
        &= -\mu\gamma_{t}\beta D_{t, k} - \mu\gamma_{t}(1 - \beta) D_{t, 0} + \mu\gamma_{t}\beta(1 - \beta) \left\| \overbar{\vb{w}}_{t, k}  - \overbar{\vb{w}}_{t, 0} \right\|^2
    \end{align}
    Using the definition of average iterate, we rewrite $\overbar{\vb{w}}_{t, k}  - \overbar{\vb{w}}_{t, 0}$ in \eqref{eq:lemma_our_algorithm_progress_average_update_rule} by applying recursion.
    \begin{align}
        a_{11} &\le -\mu\gamma_{t}\beta D_{t, k} - \mu\gamma_{t}(1 - \beta) D_{t, 0} + \mu\gamma_{t}\beta(1 - \beta) \left\| \gamma_{t} \sum_{m = 0}^{k - 1} \sum_{i = 1}^C p_{i} \grad{\widetilde{\vb{w}}_{t, m}^i}{i} \right\|^2 \label{eq:lemma_our_algorithm_progress_average_update_rule} \\
        &\le -\mu\gamma_{t}\beta D_{t, k} - \mu\gamma_{t}(1 - \beta) D_{t, 0} + \mu\gamma_{t}^3\beta(1 - \beta)k \sum_{m = 0}^{k - 1} \sum_{i = 1}^C p_{i} \left\| \grad{\widetilde{\vb{w}}_{t, m}^i}{i} \right\|^2 \label{eq:lemma_our_algorithm_progress_average_convexity}
    \end{align}
    Here, we recall Jensen's inequality in equation \eqref{eq:lemma_our_algorithm_progress_average_convexity}, then we use Assumption \ref{ass:bounded_stochastic_norm} in \eqref{eq:lemma_our_algorithm_progress_average_bounded_norm}. Under expectation, we have that $k \le E - 1$ in equation \eqref{eq:lemma_our_algorithm_progress_average_bounded_steps_count}.
    \begin{align}
        \expected \, a_{11} &\le -\mu\gamma_{t}\beta D_{t, k} - \mu\gamma_{t}(1 - \beta) D_{t, 0} + \mu\gamma_{t}^3\beta(1 - \beta)k^2 G^2 \label{eq:lemma_our_algorithm_progress_average_bounded_norm} \\
        &\le -\mu\gamma_{t}\beta D_{t, k} - \mu\gamma_{t}(1 - \beta) D_{t, 0} + \mu\gamma_{t}^3\beta(1 - \beta)E^2 G^2 \label{eq:lemma_our_algorithm_progress_average_bounded_steps_count}
    \end{align}
    Finally, after bounding $\expected \, a_{12}$ using the result of Lemma \ref{lemma:our_algorithm_perturbed_iterates_difference}, we bound $A_1$ in expectation.
    \begin{align}
        \expected \, a_1 &\le -\mu\gamma_{t}\beta \, \expected \, D_{t, k} - \mu\gamma_{t}(1 - \beta) \, \expected \,  D_{t, 0} + \\
        &\fake{\le} 4 \gamma_{t}^2 E^2 G^2 \left[4 + (1 - \beta)^2 + \mathds{1}_{t \ge 1} \frac{8\gamma_{t - 1}^2}{\gamma_{t}^2} \left(1 - \frac{1}{\beta}\right)^2 \right] + \mu\gamma_{t}^3\beta(1 - \beta)E^2 G^2 + \\
        &\fake{\le} \ev{2L\gamma_{t}^2 \sum_{i = 1}^C p_{i} \left(\Fi{\widetilde{\vb{w}}_{t, k}^i}{i} - \Fi{\vb{w}_{\star}^i}{i}\right) + 2\gamma_{t} \sum_{i = 1}^C p_{i} \left(\Fi{\vb{w}_{\star}}{i} - \Fi{\widetilde{\vb{w}}_{t, k}^i}{i}\right)}
    \end{align}
    Now, we bound term $a_{2}$.
    \begin{align*}
        \expected \, a_2 &= \gamma_{t}^2 \, \expected \underbrace{\left\| \sum_{i = 1}^C p_{i} \left[ \grad{\widetilde{\vb{w}}_{t, k}^i}{i} - \nabla \Fi{\widetilde{\vb{w}}_{t, k}^i}{i} \right] \right\|^2}_{a_{21}} + \gamma_{t}^2 \, \expected \underbrace{\left\|\sum_{i = 1}^C p_{i} \nabla \Fi{\widetilde{\vb{w}}_{t, k}^i}{i}\right\|^2}_{a_{22}}
    \end{align*}
    To bound term $a_{21}$, we recall Assumption \ref{ass:bounded_variance} to nullify the dot products between terms in $a_{21}$ and bound the squared norms. Hence, we obtain
    \begin{align*}
        \expected \, a_{21} = \sum_{i = 1}^C p_{i}^2 \, \expected \, \left\|\grad{\widetilde{\vb{w}}_{t, k}^i}{i} - \nabla \Fi{\widetilde{\vb{w}}_{t, k}^i}{i}\right\|^2 \le \sigma^2 \sum_{i = 1}^C p_{i}^2
    \end{align*}
    We recall Jensen's inequality on $\|\cdot\|^2$ to bound term $a_{22}$.
    \begin{align}
        \expected \, a_{22} = \expected \, \left\|\sum_{i = 1}^C p_{i} \nabla \Fi{\widetilde{\vb{w}}_{t, k}^i}{i}\right\|^2 \le \sum_{i = 1}^C p_{i} \, \expected \, \left\| \nabla \Fi{\widetilde{\vb{w}}_{t, k}^i}{i}\right\|^2
    \end{align}
    After combining these intermediate results, we eventually use smoothness in \eqref{eq:lemma_our_algorithm_progress_smoothness_gradient_replacement}.
    \begin{align}
        \expected \, a_2 &\le 
        \gamma_{t}^2 \sigma^2 \sum_{i = 1}^C p_{i}^2 + \gamma_{t}^2 \sum_{i = 1}^C p_{i} \, \expected \, \left\| \nabla \Fi{\widetilde{\vb{w}}_{t, k}^i}{i} \right\|^2 \\
        &\le \gamma_{t}^2 \sigma^2 \sum_{i = 1}^C p_{i}^2 + \ev{2L\gamma_{t}^2 \sum_{i = 1}^C p_{i} \left(\Fi{\widetilde{\vb{w}}_{t, k}^i}{i} - \Fi{\vb{w}_{\star}^i}{i}\right)} \label{eq:lemma_our_algorithm_progress_smoothness_gradient_replacement}
    \end{align}
    In expectation, we combine the bounds on $a_1$ and $a_2$ into our main equation.
    \begin{align}
        \expected \, D_{t, k + 1} &\le \left(1 - \mu \gamma_{t} \beta \right) \, \expected \, D_{t, k} - \mu \gamma_{t} (1 - \beta) \, \expected \, D_{t, 0} + \mu\gamma_{t}^3\beta(1 - \beta)E^2 G^2 + \\
        &\fake{\le} \gamma_{t}^2 \sigma^2 \sum_{i = 1}^C p_{i}^2 + 4 \gamma_{t}^2 E^2 G^2 \left[4 + (1 - \beta)^2 + \mathds{1}_{t \ge 1} \frac{8\gamma_{t - 1}^2}{\gamma_{t}^2} \left(1 - \frac{1}{\beta}\right)^2 \right]  + \\
        &\fake{\le} \ev{\underbrace{2\gamma_{t} \sum_{i = 1}^C p_{i} \left(\Fi{\vb{w}_{\star}}{i} - \Fi{\widetilde{\vb{w}}_{t, k}^i}{i}\right) + 4L\gamma_{t}^2 \sum_{i = 1}^C p_{i} \left(\Fi{\widetilde{\vb{w}}_{t, k}^i}{i} - \Fi{\vb{w}_{\star}^i}{i}\right)}_{b_1}}
    \end{align}
    We rewrite term $b_1$ as follows by adding and subtracting $4L\gamma_{t}^2 \sum_{i = 1}^C p_i \Fi{\vb{w}_{\star}}{i}$.
    \begin{align}
        b_1 &= -2\gamma_{t}(1 - 2 L \gamma_{t}) \sum_{i = 1}^C p_{i} \left(\Fi{\widetilde{\vb{w}}_{t, k}^i}{i} - \Fi{\vb{w}_{\star}}{i}\right) + 4L\gamma_{t}^2 \sum_{i = 1}^C p_{i} \left(\Fi{\vb{w}_{\star}}{i} - \Fi{\vb{w}_{\star}^i}{i}\right) \\
        &= -2\gamma_{t}(1 - 2 L \gamma_{t}) \sum_{i = 1}^C p_{i} \left(\Fi{\widetilde{\vb{w}}_{t, k}^i}{i} - \Fi{\vb{w}_{\star}}{i}\right) + 4\gamma_{t}^2L\Gamma  \label{eq:lemma_our_algorithm_progress_heterogeneity_substitution}
    \end{align}
    In equation \eqref{eq:lemma_our_algorithm_progress_heterogeneity_substitution}, we use the fact that $\sum_{i = 1}^C p_i \Fi{\cdot}{i} = \F{\cdot}$, and we recall the definition of statistical heterogeneity $\Gamma$. Now, we introduce term $\Fi{\overbar{{\vb{w}}}_{t, k}}{i}$ in the summation from expression \eqref{eq:lemma_our_algorithm_progress_heterogeneity_substitution}.
    \begin{align}
        b_1 &= - 2\gamma_{t}(1 - 2 L \gamma_{t}) \sum_{i = 1}^C p_{i} \left(\Fi{\overbar{{\vb{w}}}_{t, k}}{i} - \Fi{\vb{w}_{\star}}{i}\right) + 4\gamma_{t}^2L\Gamma  + \\
        &\fake{=} -2\gamma_{t}(1 - 2 L \gamma_{t}) \sum_{i = 1}^C p_{i} \left(\Fi{\widetilde{\vb{w}}_{t, k}^i}{i} - \Fi{\overbar{{\vb{w}}}_{t, k}}{i}\right) \\
        &= - 2\gamma_{t}(1 - 2 L \gamma_{t}) \left(\F{\overbar{{\vb{w}}}_{t, k}} - \F{\vb{w}_{\star}}\right) + 4\gamma_{t}^2L\Gamma  + \\
        &\fake{=} 2\gamma_{t}(1 - 2 L \gamma_{t}) \sum_{i = 1}^C p_{i} \left(\Fi{\overbar{{\vb{w}}}_{t, k}}{i} - \Fi{\widetilde{\vb{w}}_{t, k}^i}{i}\right) \\
        &\le - 2\gamma_{t}(1 - 2 L \gamma_{t}) \left(\F{\overbar{{\vb{w}}}_{t, k}} - \F{\vb{w}_{\star}}\right) + 4\gamma_{t}^2L\Gamma  + \\
        &\fake{=} 2\gamma_{t}(1 - 2 L \gamma_{t}) \sum_{i = 1}^C p_{i} \nabla \Fi{\overbar{{\vb{w}}}_{t, k}}{i}^{\transpose}\left(\overbar{{\vb{w}}}_{t, k} - \widetilde{{\vb{w}}}_{t, k}^i\right) \label{eq:lemma_our_algorithm_progress_convexity_usage} \\
        &\le - 2\gamma_{t}(1 - 2 L \gamma_{t}) \left(\F{\overbar{{\vb{w}}}_{t, k}} - \F{\vb{w}_{\star}}\right) + 4\gamma_{t}^2L\Gamma  + \\
        &\fake{=} 2\gamma_{t}(1 - 2 L \gamma_{t}) \sum_{i = 1}^C p_{i} \left[\frac{\gamma_{t}}{2} \left\|\nabla \Fi{\overbar{{\vb{w}}}_{t, k}}{i}\right\|^2 + \frac{1}{2\gamma_{t}} \left\|\overbar{{\vb{w}}}_{t, k} - \widetilde{{\vb{w}}}_{t, k}^i\right\|^2 \right] \label{eq:lemma_our_algorithm_progress_peter_paul_in_B} \\
        &\le - 2\gamma_{t}(1 - 2 L \gamma_{t}) \left(\F{\overbar{{\vb{w}}}_{t, k}} - \F{\vb{w}_{\star}}\right) + 4\gamma_{t}^2L\Gamma  + \\
        &\fake{=} 2\gamma_{t}(1 - 2 L \gamma_{t}) \sum_{i = 1}^C p_{i} \left[L\gamma_{t}\left(\Fi{\overbar{{\vb{w}}}_{t, k}}{i} - \Fi{\vb{w}_{\star}^i}{i}\right) + \frac{1}{2\gamma_{t}} \left\|\overbar{{\vb{w}}}_{t, k} - \widetilde{{\vb{w}}}_{t, k}^i\right\|^2 \right] \label{eq:lemma_our_algorithm_progress_smoothness_in_B}
    \end{align}
    Peculiarly, we leverage convexity in \eqref{eq:lemma_our_algorithm_progress_convexity_usage}, Peter-Paul's inequality in expression \eqref{eq:lemma_our_algorithm_progress_peter_paul_in_B}, and smoothness in \eqref{eq:lemma_our_algorithm_progress_smoothness_in_B}. From \eqref{eq:lemma_our_algorithm_progress_smoothness_in_B}, we add and subtract $2L\gamma_{t}^2(1 - 2L\gamma_{t})\F{\vb{w}_{\star}}$.
    \begin{align}
        b_1 &\le -2\gamma_{t}(1 - 2L\gamma_{t})(1 - L\gamma_{t}) \left(\F{\overbar{{\vb{w}}}_{t, k}} - \F{\vb{w}_{\star}}\right) + 4\gamma_{t}^2L\Gamma +  \\
        &\fake{\le} 2L\gamma_{t}^2(1 - 2L\gamma_{t})\left(\F{\vb{w}_{\star}} - \sum_{i = 1}^C p_{i} \Fi{\vb{w}_{\star}^i}{i}\right) + \sum_{i = 1}^C p_{i} \left\|\overbar{{\vb{w}}}_{t, k} - \widetilde{{\vb{w}}}_{t, k}^i\right\|^2 \label{eq:lemma_our_algorithm_progress_1_2L_approx}  \\
        &= -2\gamma_{t}(1 - 2L\gamma_{t})(1 - L\gamma_{t}) \left(\F{\overbar{{\vb{w}}}_{t, k}} - \F{\vb{w}_{\star}}\right) + 2L\Gamma\gamma_{t}^2(3 - 2L\gamma_{t}) + \\
        &\fake{=} \sum_{i = 1}^C p_{i} \left\|\overbar{{\vb{w}}}_{t, k} - \widetilde{{\vb{w}}}_{t, k}^i\right\|^2 \label{eq:lemma_our_algorithm_progress_heterogeneity_substitution_2}
    \end{align}
    In expectation, we have once again
    \begin{align}
        \expected \, b_1 &\le 6L\Gamma\gamma_{t}^2 + 4 \gamma_{t}^2 E^2 G^2 \left[4 + (1 - \beta)^2 + \mathds{1}_{t \, \ge \, 1} \frac{8\gamma_{t - 1}^2}{\gamma_{t}^2} \left(1 - \frac{1}{\beta}\right)^2 \right]\label{eq:lemma_our_algorithm_progress_bound_B_end}
    \end{align}
    Again, we use the definition of heterogeneity in \eqref{eq:lemma_our_algorithm_progress_heterogeneity_substitution_2}. On the other hand, in expression \eqref{eq:lemma_our_algorithm_progress_bound_B_end}, we exploit the fact that $(1 - 2L\gamma_{t})(1 - L\gamma_{t}) \ge 0$ since $\gamma_{t} \le 1 / (2L)$ due to Assumption \ref{lemma:our_algorithm_one_step_progress_step_decay}, and also $1 - 2L\gamma_{t} \le 1$ in \eqref{eq:lemma_our_algorithm_progress_1_2L_approx} and \eqref{eq:lemma_our_algorithm_progress_bound_B_end}. Eventually, we use the outcome of Lemma \ref{lemma:our_algorithm_perturbed_iterates_difference} in \eqref{eq:lemma_our_algorithm_progress_bound_B_end}. In expectation, we utilize this bound back into our main expression.
    \begin{align}
        \expected \, D_{t, k + 1} &\le \left(1 - \mu \gamma_{t} \beta \right) \, \expected \, D_{t, k} - \mu \gamma_{t} (1 - \beta) \, \expected \,  D_{t, 0} + \gamma_{t}^2 \sigma^2 \sum_{i = 1}^C p_{i}^2 + 6 \gamma_{t}^2 L \Gamma + \\
        &\fake{\le} 8 \gamma_{t}^2 E^2 G^2 \left[4 + (1 - \beta)^2 + \mathds{1}_{t \, \ge \, 1} \frac{8\gamma_{t - 1}^2}{\gamma_{t}^2} \left(1 - \frac{1}{\beta}\right)^2 \right] + \mu\gamma_{t}^3\beta(1 - \beta)E^2 G^2
    \end{align}
    We define variables
    \begin{align*}
        a &= 1 - \mu \gamma_{t} \beta \\
        b &= - \mu \gamma_{t} (1 - \beta) \\
        c&= 8 \gamma_{t}^2 E^2 G^2 \left[4 + (1 - \beta)^2 + \frac{\mathds{1}_{t \, \ge \, 1} 8\gamma_{t - 1}^2 (1 - \beta)^2}{\gamma_{t}^2 \beta^2} \right] + \mu\gamma_{t}^3\beta(1 - \beta)E^2 G^2 + \gamma_{t}^2  S \sigma^2 + 6 \gamma_{t}^2 L \Gamma
    \end{align*}
    where $S \eqdef \sum_{i = 1}^C p_{i}^2$. Hence, we apply the usual recursion technique to the following expression.
    \begin{align}
        \expected \, D_{t, k + 1} &\le a \, \expected \, D_{t, k} + b \, \expected \, D_{t, 0} + c \\
        &\ldots \\
        &\le \left[a^{k + 1} + b \sum_{m = 0}^{k} a^{m}\right] \expected \, D_{t, 0} + c \sum_{m = 0}^{k} a^{m} \\
        &\le \left[a^{k + 1} + b \frac{1 - a^{k + 1}}{1 - a} \right] \expected \, D_{t, 0} + c (k + 1) \label{eq:lemma_our_algorithm_progress_rough_approximation_geometric} \\
        &= \frac{b + (1 - a - b)a^{k + 1}}{1 - a} \, \expected \, D_{t, 0} + c (k + 1)
    \end{align}
    In expression \eqref{eq:lemma_our_algorithm_progress_rough_approximation_geometric}, we roughly approximate the geometric series that multiplies term $c$ using the fact that $(1 - \mu \gamma_{t} \beta)^{m} \le 1$ to preserve $\gamma_{t}^2$ within $c$. To conclude, after setting $k + 1 = E$ to establish the bound for a round of communication, we replace $a, b$ and $c$.
\end{proof}

We derive the convergence rate for a specific choice of step size.

\OurAlgorithmConvergenceStronglyConvex*

\begin{proof}
    We first apply the principle of recursion on the result of Lemma \ref{lemma:our_algorithm_single_round_global_progress} in equation \eqref{eq:our_algorithm_convergence_constant_step_size_lemma_replacement}. Namely, if we denote $\expected \, \left\|\overbar{\vb{w}}_{t, 0} - \vb{w}_{\star}\right\|^2$ as $D_{t}$, we have
    \begin{align}
        D_{t} &\le \kappa D_{t - 1} + A \label{eq:our_algorithm_convergence_constant_step_size_lemma_replacement} \\
        &\le \kappa  \left(\kappa D_{t - 2} + A\right) + A \\
        &\ldots \\
        &\le \kappa^{t} D_{0} + A \sum_{m = 0}^{t - 1} \kappa^{m} \\
        &= \kappa^{t} D_{0} + A \sum_{m = 0}^{t - 1} \frac{1 - \kappa^{t}}{1 - \kappa} \\
        &\le \kappa^{t} D_{0} + \frac{A}{1 - (1 - \gamma\mu)} \left(1 - \kappa^{t}\right) \label{eq:our_algorithm_convergence_constant_step_size_contraction_replacement} \\
        &\le \kappa^{t} \left(D_{0} - \frac{A}{\gamma\mu}\right) + \frac{A}{\gamma\mu}
    \end{align}
    where we use the coarser but simpler approximation $\kappa \le 1 - \gamma\mu$ in equation \eqref{eq:our_algorithm_convergence_constant_step_size_contraction_replacement}. Finally, under total expectation, we invoke smoothness.
    \begin{align*}
        \expected \, \F{\overbar{\vb{w}}_{t, 0}} - \F{\vb{w}_{\star}} &\le \frac{L}{2} D_{t} \le \frac{L}{2}\kappa^{t} \left(D_{0} - \frac{A}{\gamma\mu}\right) + \frac{L}{2\gamma\mu}A
    \end{align*}
    With strong convexity we have $D_{0} \le 2\left(\F{\overbar{\vb{w}}_{0, 0}} - \F{\vb{w}_{\star}}\right)/\mu$. We bound the contraction factor $\kappa$ using the definition of step size. 
    \begin{align}
        \kappa &= 1 - \frac{1}{\beta} + \frac{1}{\beta}\left(1 - \frac{\beta \mu}{2EL}\right)^E \\
        &= 1 - \frac{1}{\beta} + \frac{1}{\beta}\left[\left(1 - \frac{\beta\mu}{2EL}\right)^{-\frac{2EL}{\beta\mu}}\right]^{-\frac{\beta\mu}{2L}} \\
        &\le 1 - \frac{1}{\beta} + \frac{1}{\beta}e^{-\frac{\beta\mu}{2L}} \label{eq:our_algorithm_convergence_constant_step_size_e_bound} \\
        &\le 1 - \frac{1}{\beta} + \frac{1}{\beta}\frac{2L}{\beta\mu + 2L} \label{eq:our_algorithm_convergence_constant_step_size_e_bound_2} \\
        &= 1 - \frac{\mu}{\beta \mu + 2L}
    \end{align}
    We use fact that $(1 + 1/x)^x \le e$ in equation \eqref{eq:our_algorithm_convergence_constant_step_size_e_bound}, and $e^{-x} \le 1/(x + 1)$ for any $x > -1$ in \eqref{eq:our_algorithm_convergence_constant_step_size_e_bound_2}. Moreover, we replace the chosen $\gamma$ in the error term. Lastly, we approximate $\beta \mu + 2L \le L(\beta + 2)$ and neglect the negative term in the bound.
\end{proof}

\subsection{Main results for nonconvex analysis}

The following lemma quantifies the progress made by our algorithm in a global round for nonconvex loss objectives.

\begin{lemma}[Single round progress of our algorithm for nonconvex loss]
    \label{lemma:our_algorithm_nonconvex_single_round_global_progress}
    Assume that $\gamma_{t} \le 1 / L \label{ass:our_algorithm_nonconvex_single_round_global_progress_step_size_assumption}$ and Assumptions \ref{ass:full_participation} to \ref{ass:smoothness} and \ref{ass:lower_bounded_objective} hold. Then, in a single round, we have
    \begin{align}
        \frac{1}{E} \sum_{k = 0}^{E - 1} \expected \, \left\|\nabla \F{\overbar{\vb{w}}_{t, k}}\right\|^2 \le \frac{2}{\gamma_{t} E} \, \ev{\F{\overbar{\vb{w}}_{t, 0}} - \F{\overbar{\vb{w}}_{t + 1, 0}}} + A
    \end{align}
    where $A = \gamma_{t} L \sigma^2 \sum_{i = 1}^C p_i^2 + 4 \gamma_{t}^2 L^2 E^2 G^2 \left[4 + (1 - \beta)^2 + \mathds{1}_{t \, \ge \, 1}  \dfrac{8\gamma_{t - 1}^2}{\gamma_{t}^2} \left(1 - \dfrac{1}{\beta}\right)^2 \right]$.
\end{lemma}

\begin{proof}
    This proof is very similar to the one of Lemma \ref{lemma:fedprox_nonconvex_single_round_global_progress} with problem-specific adjustments. Therefore, we begin by stating the definition of the update rule for the average iterate.
    \begin{align}
        \overbar{\vb{w}}_{t, k + 1} - \overbar{\vb{w}}_{t, k} = - \gamma_{t} \sum_{i = 1}^C p_i \grad{\widetilde{{\vb{w}}}_{t, k}^i}{i}
    \end{align}
    Iterates $\overbar{\vb{w}}_{t, k + 1}$ and $\overbar{\vb{w}}_{t, k}$ are replaced in the definition of smoothness from \ref{ass:smoothness}. We also define $\delta_{t, k} = \F{\overbar{\vb{w}}_{t, k + 1}} - \F{\overbar{\vb{w}}_{t, k}}$.
    \begin{align}
        \expected \, \delta_{t, k} &\le \ev{\underbrace{-\gamma_{t} \sum_{i = 1}^C p_i \nabla \Fi{\widetilde{{\vb{w}}}_{t, k}^i}{i}^{\transpose} \nabla \F{\overbar{\vb{w}}_{t, k}}}_{a_1}} + \ev{\underbrace{\frac{\gamma_{t}^2 L}{2} \left\| \sum_{i = 1}^C p_i \grad{\widetilde{{\vb{w}}}_{t, k}^i}{i}\right\|^2}_{a_2}} \\
        &\fake{\le} \ev{\underbrace{-\gamma_{t} \, \sum_{i = 1}^C p_i \left[\grad{\widetilde{{\vb{w}}}_{t, k}^i}{i} - \nabla \Fi{\widetilde{{\vb{w}}}_{t, k}^i}{i} \right]^{\transpose} \nabla \F{\overbar{\vb{w}}_{t, k}}}_{\widetilde{a}_1}}
    \end{align}
    Because of Assumption \ref{ass:bounded_variance} on the unbiasedness of the stochastic gradient, we observe that $\expected \, \widetilde{a}_1 = 0$. Let us bound $a_1$ first.
    \begin{align}
        a_1 &\le \frac{\gamma_{t}}{2}\left\|\sum_{i = 1}^C p_i \nabla \Fi{\widetilde{{\vb{w}}}_{t, k}^i}{i} - \nabla \F{\overbar{\vb{w}}_{t, k}}\right\|^2 - \frac{\gamma_{t}}{2}\left\|\sum_{i = 1}^C p_i \nabla \Fi{\widetilde{{\vb{w}}}_{t, k}^i}{i}\right\|^2 \\
        &\fake{\le} -\frac{\gamma_{t}}{2} \left\|\nabla \F{\overbar{\vb{w}}_{t, k}}\right\|^2 \label{eq:lemma_our_algorithm_progress_non_convex_parallelogram_law} \\
        &= \frac{\gamma_{t}}{2}\left\|\sum_{i = 1}^C p_i \left( \nabla \Fi{\widetilde{{\vb{w}}}_{t, k}^i}{i} - \nabla \Fi{\overbar{\vb{w}}_{t, k}}{i}\right)\right\|^2 - \frac{\gamma_{t}}{2}\left\|\sum_{i = 1}^C p_i \nabla \Fi{\widetilde{{\vb{w}}}_{t, k}^i}{i}\right\|^2 + \\
        &\fake{\le} -\frac{\gamma_{t}}{2} \left\|\nabla \F{\overbar{\vb{w}}_{t, k}}\right\|^2 \label{eq:lemma_our_algorithm_progress_non_convex_average_gradient_decomposition} \\
        &\le \frac{\gamma_{t}}{2}\sum_{i = 1}^C p_i \left\|\nabla \Fi{\widetilde{{\vb{w}}}_{t, k}^i}{i} - \nabla \Fi{\overbar{\vb{w}}_{t, k}}{i}\right\|^2 - \frac{\gamma_{t}}{2}\left\|\sum_{i = 1}^C p_i \nabla \Fi{\widetilde{{\vb{w}}}_{t, k}^i }{i}\right\|^2 + \\
        &\fake{\le} -\frac{\gamma_{t}}{2} \left\|\nabla \F{\overbar{\vb{w}}_{t, k}}\right\|^2 \label{eq:lemma_our_algorithm_progress_non_convex_average_gradient_jensen} \\
        &\le \frac{\gamma_{t} L^2}{2}\sum_{i = 1}^C p_i \left\|\widetilde{{\vb{w}}}_{t, k}^i - \overbar{\vb{w}}_{t, k}\right\|^2 - \frac{\gamma_{t}}{2}\left\|\sum_{i = 1}^C p_i \nabla \Fi{\widetilde{{\vb{w}}}_{t, k}^i}{i}\right\|^2 + \\
        &\fake{\le} - \frac{\gamma_{t}}{2} \left\|\nabla \F{\overbar{\vb{w}}_{t, k}}\right\|^2
        \label{eq:lemma_our_algorithm_progress_non_convex_lipschitz_gradient} 
    \end{align}
    In equation \eqref{eq:lemma_our_algorithm_progress_non_convex_parallelogram_law}, we use the fact that $2\vb{u}^{\transpose}\vb{v} = \|\vb{u}\|^2 + \|\vb{v}\|^2 - \|\vb{u} - \vb{v}\|^2$, while we leverage $\nabla \F{\cdot} = \sum_{i = 1}^C p_i \nabla \Fi{\cdot}{i}$ in \eqref{eq:lemma_our_algorithm_progress_non_convex_average_gradient_decomposition}, and we recall Jensen's inequality in \eqref{eq:lemma_our_algorithm_progress_non_convex_average_gradient_jensen}.
    In equation \eqref{eq:lemma_our_algorithm_progress_non_convex_lipschitz_gradient}, we use the Lipschitz characterization of smoothness. To bound term $a_2$, we exploit Assumptions \ref{ass:bounded_variance} as we already did in previous proofs.
    \begin{align}
        a_2 &=  \frac{\gamma_{t}^2 L}{2} \left\| \sum_{i = 1}^C p_i \left(\grad{\vb{w}_{t, k}^i}{i} - \Fi{\vb{w}_{t, k}^i}{i}\right) + \sum_{i = 1}^C p_i \Fi{\vb{w}_{t, k}^i}{i}\right\|^2 \\
        &= \frac{\gamma_{t}^2 L}{2} \, \left[\left\| \sum_{i = 1}^C p_i \left(\grad{\vb{w}_{t, k}^i}{i} - \Fi{\vb{w}_{t, k}^i}{i}\right) \right\|^2 + \left\| \sum_{i = 1}^C p_i \Fi{\vb{w}_{t, k}^i}{i}\right\|^2\right] \\
        &=  \frac{\gamma_{t}^2 L}{2} \, \left[\sum_{i = 1}^C p_i^2 \left\| \grad{\vb{w}_{t, k}^i}{i} - \Fi{\vb{w}_{t, k}^i}{i} \right\|^2 + \left\| \sum_{i = 1}^C p_i \Fi{\vb{w}_{t, k}^i}{i}\right\|^2\right]
    \end{align}
    Taking the expectation, we have
    \begin{align}
        \expected \, a_2 \le  \frac{\gamma_{t}^2 L \sigma^2}{2} \sum_{i = 1}^C p_i^2 + \frac{\gamma_{t}^2 L}{2} \, \expected \, \left\| \sum_{i = 1}^C p_i \nabla \Fi{\vb{w}_{t, k}^i}{i} \right\|^2
    \end{align}
    We join the bounds on $a_1$ and $a_2$. Thus, under expectation, we attain
    \begin{align}
        \expected \, \delta_{t, k} &\le -\frac{\gamma_{t}(1 - L\gamma_{t})}{2} \, \expected  \left\| \sum_{i = 1}^C p_i \nabla \Fi{\widetilde{\vb{w}}_{t, k}^i}{i} \right\|^2 - \frac{\gamma_{t}}{2} \, \expected \, \left\|\nabla \F{\overbar{\vb{w}}_{t, k}}\right\|^2 + \frac{\gamma_{t}^2 L \sigma^2}{2} \sum_{i = 1}^C p_i^2 + \\
        &\fake{\le} \underbrace{\frac{\gamma_{t} L^2}{2} \sum_{i = 1}^C p_i \, \expected \, \left\|\overbar{\vb{w}}_{t, k} - \widetilde{\vb{w}}_{t, k}^i\right\|^2}_{b}
    \end{align}
    We notice that $-\gamma_{t}(1 - L\gamma_{t}) \le 0$ because of assumption \eqref{ass:our_algorithm_nonconvex_single_round_global_progress_step_size_assumption}, and we replace $b$ using the result of Lemma \ref{lemma:our_algorithm_perturbed_iterates_difference}. Eventually, we attain
    \begin{align*}
        \expected \, \delta_{t, k} &\le 
        \underbrace{2 \gamma_{t}^3 L^2 E^2 G^2 \left[4 + (1 - \beta)^2 + \frac{\mathds{1}_{t \, \ge \, 1} 8 \gamma_{t - 1}^2 (1 - \beta)^2}{\gamma_{t}^2 \beta^2} \right] + \frac{\gamma_{t}^2 L \sigma^2}{2} \sum_{i = 1}^C p_i^2}_{c} - \expected \,\frac{\gamma_{t}}{2} \left\|\nabla \F{\overbar{\vb{w}}_{t, k}}\right\|^2
    \end{align*}
    After rearranging the terms, we sum from $k = 0$ to $E - 1$.
    \begin{align}
        \sum_{k = 0}^{E - 1} \expected \, \left\|\nabla \F{\overbar{\vb{w}}_{t, k}}\right\|^2 \le \frac{2}{\gamma_{t}} \, \ev{\F{\overbar{\vb{w}}_{t, 0}} - \F{\overbar{\vb{w}}_{t, E}} } + \frac{2Ec}{\gamma_{t}}
    \end{align}
    We replace $\overbar{\vb{w}}_{t, E} \equiv \overbar{\vb{w}}_{t + 1, 0}$, and we conclude our proof by dividing by $E$.
\end{proof}

We now address the convergence behavior when opting for a constant step size.

\OurAlgorithmConvergenceNonConvex*

\begin{proof}
    We first unroll the inequality from Lemma \ref{lemma:our_algorithm_nonconvex_single_round_global_progress} using recursion as we did in Theorem \ref{theorem:convergence_fedprox_nonconvex}.
    Then we substitute our chosen $\gamma$, and we use the definition of $\vb{\widehat{w}}_{t, k}$.
\end{proof}

Finally, we refine the convergence rate of the previous theorem when using a fixed step size. Specifically, we focus on a choice of $E$ that minimizes the bound.

\optimalNumberLocalStepsNonConvex*

\begin{proof}
   We rewrite the convergence rate from the case $I$ of Theorem \ref{theorem:convergence_our_algorithm_nonconvex} as a function of $E$, namely $r(E) = A / \sqrt{TE} + BE / T$ where 
   \begin{align}
       A = 4 L \Delta  + S \sigma^2 / 2 \quad \mathrm{and} \quad B = G^2 \left[4 + (1 - \beta)^2 + 8\left(1 - \dfrac{1}{\beta}\right)^2 \right]
   \end{align}
   Minimizing $r(E)$ in relation to $E$ leads to the critical point $E_{\mathrm{opt}} = [A/(2B)]^{2 / 3} T^{1 / 3}$. After replacing $E_{\mathrm{opt}}$ in $r(E)$, we obtain $r(E_{\mathrm{opt}}) = 3 / 2^{2/3} A^{2 / 3} B^{1 / 3} T^{-2 / 3}$. To conclude, we ignore the constants that depend on $A$ and $B$ in the $\mathcal{O}(\cdot)$ notation.
\end{proof}

\section{Additional details on empirical results}
\label{sec:appendix_experimentation}

This appendix is dedicated to explaining in detail the experiments of the main part of the paper as well as further results that we have obtained.

\subsection{Datasets}
\label{sec:datasets_detailed}

We illustrate the procedure adopted to split each dataset into training and testing sets, and the generation process on each set to simulate the partitioning among $C = 100$ clients in the context of an arbitrarily heterogeneous federated network.

\paragraph{Splitting in training and testing datasets}

We split each full dataset into training and testing datasets following approximately the $80$:$20$ convention. We keep the partition fixed across all experiments to favor reproducibility.

\paragraph{Preprocessing}

Data samples undergo a standardization step along each channel. For every data sample $i$, we transform each pixel value $\vb{x}_i[m]$ to $\vb{z}_i[m] = (\vb{x}_i[m] - \vb{\mu}[m]) / \vb{\sigma}[m]$, where $\vb{\mu}[m]$ and $\vb{\sigma}[m]$ are the $m$-th pixel mean and standard deviation across all samples, respectively.

\paragraph{Dividing samples among clients}
\label{sec:federated_dataset_generation}

We generate a federated dataset from the centralized version by partitioning the latter into $C$ disjoint subsets of samples, where $C$ is the number of clients. We control the degree of heterogeneity for the generated dataset through two parameters: class imbalance and data imbalance. The class imbalance expresses the disproportion in the number of samples of a specific class assigned across all clients. This measure is parameterized through $a \ge 0$ where $a \approx 0$ corresponds to perfect class balance, while larger values imply a higher imbalance. 
Taking inspiration from \cite{measuring_niid_data}, to achieve this, we sample the number of examples of a specific class assigned to each client from a Dirichlet distribution $\mathrm{Dir}((1 / a) \cdot \vb{1}_{K})$ where $K$ is the number of classes. All previous steps are repeated both for the training and testing datasets.
On the other hand, the data imbalance represents the discrepancy between the number of data points given to each client. We parameterize such a measure through $b > 0$, which is employed to sample the size of each client's dataset from a log-normal distribution $\ln \mathcal{N}(\lfloor N / C \rfloor, b)$, where $N$ is the total amount of samples across all participants. Again, by using $b \approx 0$ we expect all clients to share almost the same number of samples. 
Both these two parameters let us construct accurately a federated dataset by calibrating the extent of imbalance that we aim to introduce in the resulting network of clients. Furthermore, whenever some samples are left in the partitioning phase, these are then divided among clients to approximately ensure the extent of class and data imbalance requested.

The following computer vision datasets are released under the MIT license and are widely employed for federated learning. For instance, \cite{fedprox} used FEMNIST to assess the performance of \textsc{FedProx}, and \cite{parallel_sgd_guide_optimization, fednova} used CIFAR10 to study the behavior of federated algorithms in nonconvex settings.

\paragraph{CIFAR10}

CIFAR10, introduced by \cite{cifar10} in addition to CIFAR100, is a dataset of 60000 RGB images of size $32 \times 32$ belonging to 10 classes, where each class takes exactly 6000 samples. The 10 classes correspond to daily observable objects such as \textit{airplace}, \textit{truck}, and others.
\begin{figure}[thb]
    \centering
    \includegraphics[width=1 \textwidth]{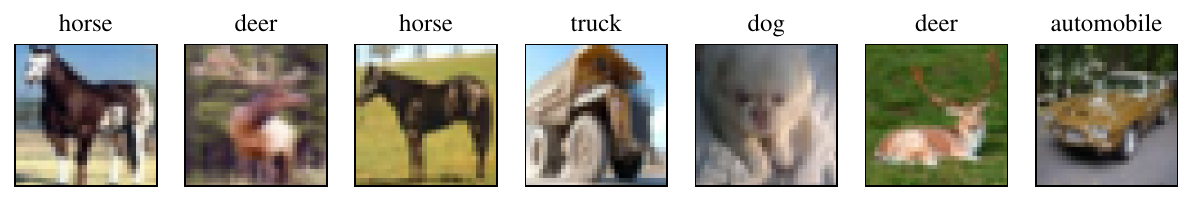}
    \caption{Visualization of samples from CIFAR10 (\cite{cifar10}) dataset.}
    \label{fig:cifar10-samples}
\end{figure}

\paragraph{CIFAR100}

CIFAR100, similarly to CIFAR10, contains RGB 60000 images of the same size but belonging to 100 classes. All 100 labels can be further grouped into 20 more generic superclasses. Nonetheless, we only consider the finer subdivision for our experiments in Subsection \ref{sec:auxiliary_results}.

\paragraph{FEMNIST}

\cite{leaf} published the FEMNIST dataset to deliberately study federated learning and to establish a common benchmark to investigate the performance of distributed optimization algorithms. The original dataset contains $805263$ grayscale images of size $28 \times 28$ subdivided into $62$ classes. We use a subset of $382705$ samples picturing handwritten digits from $0$ to $9$. 
\begin{figure}[h]
    \centering
    \includegraphics[width=1 \textwidth]{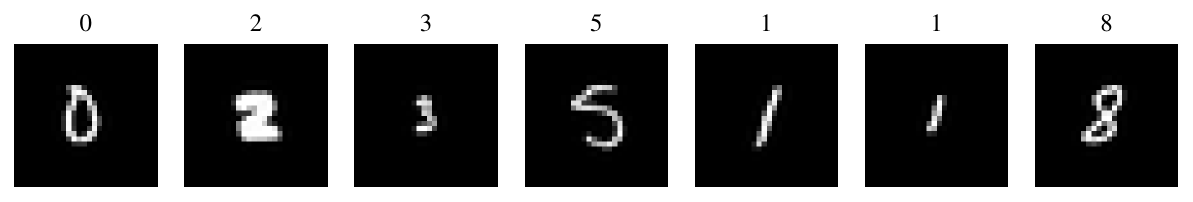}
    \caption{Visualization of samples from FEMNIST (\cite{leaf}) dataset.}
    \label{fig:femnist-samples}
\end{figure}

\subsection{Loss objective}

We describe the two kinds of loss functions that we use in practice according to our convex and nonconvex theoretical analysis.

\paragraph{Strongly convex}

Since part of our theoretical analysis applies to smooth and strongly convex loss objectives, we employ a multinomial logistic regression model with $L_2$ penalty of the parameters $\vb{w} \eqdef \{ \, \vb{v}_k \, \}_{k = 0}^K$ where each $\vb{v}_k \in \mathbb{R}^D$ for $k = 0, \ldots, K$. We denote the number of classes as $K$, and dataset $\mathcal{D}_i \eqdef \{ \, (\vb{x}_n, \vb{y}_n) \, \}_{n = 1}^{N_i}$ for each client $i \in \{ \, 1, 2, \ldots, C \, \}$. Therefore, we define the local loss function for client $i$ as follows.
\begin{align}
    \Fi{\vb{w}}{i} \eqdef \frac{1}{N_i} \sum_{n = 1}^{N_i} \ell(\vb{w}; (\vb{x}_n, \vb{y}_n)) + \frac{\lambda}{2} \sum_{k = 0}^K \left\|\vb{v}_k\right\|^2
\end{align}
Notice that the label vector $\vb{y}$ is a one-hot encoded vector where the true class is $1$ and other entries $0$. In addition, we deliberately decide to include the bias $\vb{v}_0$ in the computation of the linear mappings, where $\vb{e}_i$ is the $i$-th column vector of the canonical basis. Ultimately, we choose the sample loss $\ell(\vb{w}; (\vb{x}, \vb{y}))$ as the cross entropy on the output of the multinomial logistic regression, that is
\begin{align}
    \ell(\vb{w}; (\vb{x}, \vb{y})) \eqdef - \sum_{k = 1}^K y_k \left[z_k - \ln\left[ \, \sum_{s = 1}^K \exp(z_s) \, \right] \right] \quad \textrm{where} \quad z_i \eqdef \vb{v}_i^{\transpose} \vb{x} + \vb{v}_0^{\transpose} \vb{e}_i
\end{align}
We select the predicted class as $\arg\max_{k}(z_k)$.

\paragraph{Nonconvex}

In the nonconvex scenario, we choose an elementary neural network with a single hidden layer of $128$ neurons using ReLU activation. The input layer accepts flattened images, and the output layer emits class probabilities fed to a cross-entropy loss. Moreover, the weights are subject to $L_2$ regularization as in the strongly convex case.

\subsection{Validation metrics}

We use two different metrics to assess the performance of each algorithm on unseen data, namely those $C$ clients generated from the testing set of each dataset.

\paragraph{Accuracy}

The accuracy that we represent in several plots and tables is the mean accuracy across all testing clients. More precisely, each client $i$ computes the local accuracy of the global model $\vb{w}$ on its data samples $\{ \,(\vb{x}_j, \vb{y}_j) \, \}_{j = 1}^{N_i}$, and local accuracies are then averaged, according to the amount of data samples held by each client, to compute the overall accuracy metric.
\begin{align}
    \mathrm{accuracy}(\vb{w}) \eqdef \sum_{i = 1}^C \frac{N_i}{N} \underbrace{\sum_{j = 1}^{N_i} \frac{\mathds{1}\left\{ \, h(\vb{w}; \vb{x}_j) = \vb{y}_j \, \right\}}{N_i}}_{\textrm{accuracy on testing client }i} \quad \textrm{where} \quad N \eqdef \sum_{i = 1}^C N_i
    \label{eq:accuracy}
\end{align}
The expression $\mathds{1}\left\{ \, h(\vb{w}; \vb{x}_j) = \vb{y}_j \, \right\}$ is the 0/1 score on the predicted output $h(\vb{w}; \vb{x}_j)$.

\paragraph{Convergence speed}

The convergence speed is given by the minimum number of rounds that the testing accuracy \eqref{eq:accuracy} needs to surpass the threshold $c$ (see Table \ref{tab:federated_experiments_parameter_grid}), namely $\mathrm{accuracy}(\overbar{\vb{w}}_{t, 0}) \ge c$.

\subsection{Hyper-parameters}

Table \ref{tab:federated_experiments_parameter_grid} contains the set of hyper-parameters employed to schedule the experiments found in \ref{sec:results} and in this appendix.
\begin{table}[t]
    \caption{Grid of parameters used for the convergence simulations on FEMNIST and CIFAR10. We use this grid to compare our algorithm against the baseline methods and to study the effect of the variation of both $E$ and $G$ on our algorithm while keeping the other parameters fixed.}
    \renewcommand\arraystretch{1.25}
    \setlength{\tabcolsep}{5pt}
    \centering
    \small
    \begin{tabular}[t]{llll}
        \toprule
        & & \textbf{FEMNIST} & \textbf{CIFAR10} \\
        \midrule
        $\mathcal{A}$ & Federated algorithm & \textsc{FedAvg}, Ours & \textsc{FedAvg}, Ours \\
        $\mathcal{W}$ & Aggregation scheme & \textsc{Adjacency} & \textsc{Adjacency} \\
        $\alpha$ & Proximal parameter & 0 & 0 \\
        $\beta$ & Our algorithm's parameter & 0.5, 0.7, 0.9 & 0.5, 0.7, 0.9 \\
        $\gamma$ & Local step size & $10^{-3}$ & $10^{-3}$ \\
        $\lambda$ & $L_2$ regularization & $10^{-4}$ & $10^{-4}$ \\
        $B$ & Minibatch size & 256 & 256 \\
        $C$ & Number of clients & 100 & 100 \\
        $E$ & Number of local epochs & 10 & 10 \\
        $T$ & Number of rounds & 200 & 200 \\
        $a$ & Class imbalance & 0, 10 & 0, 100 \\
        $b$ & Data imbalance & 0, 1 & 0, 1 \\
        $c$ & Convergence threshold & 0.75 & 0.30 \\
        $s$ & Random seed & 0 & 0 \\
        \bottomrule
        \label{tab:federated_experiments_parameter_grid}
    \end{tabular}
\end{table}

\subsection{Implementation}
\label{sec:implementation}

We implement the datasets and the algorithmic frameworks using PyTorch $\ge 2.0$ from \cite{pytorch}. We run all experiments on an Ubuntu 22.04 laptop mounting 16GB DDR4 RAM, Intel Core i7-7500U CPU 2.7GHz processor, and Nvidia GeForce 940MX (2GB VRAM) GPU. Each experiment takes approximately 9 to 24 hours on a single worker with the CUDA accelerator. Additionally, to promote reproducibility, we use the same random seed initialization (when not tuned) and the deterministic option offered by PyTorch.

\paragraph{Our technical limitations and possible suggestions}
\label{sec:implementation_limitations}

The limited hardware resources allow for a reduced set of feasible simulations because of the demanding running time and computationally intensive nature of the experiments. We suggest running the convergence simulations enabling the CUDA accelerator whenever possible, especially when using the nonconvex loss.

\subsection{Auxiliary results on the nature of the graph representation of federated networks}
\label{sec:auxiliary_results}

This part contains supplemental information to support our definition of the graph representation of a federated network from paragraph \ref{sec:graph_based_model}. In all results, we fix the number of clients $C = 100$ and change the imbalance parameters as shown in figures and the random seed using $\{ \, 0, 1, 41 \, \}$.

\paragraph{How imbalance affects client misalignment}
\begin{figure}[h]
    \centering
    \includegraphics[width=1 \textwidth]{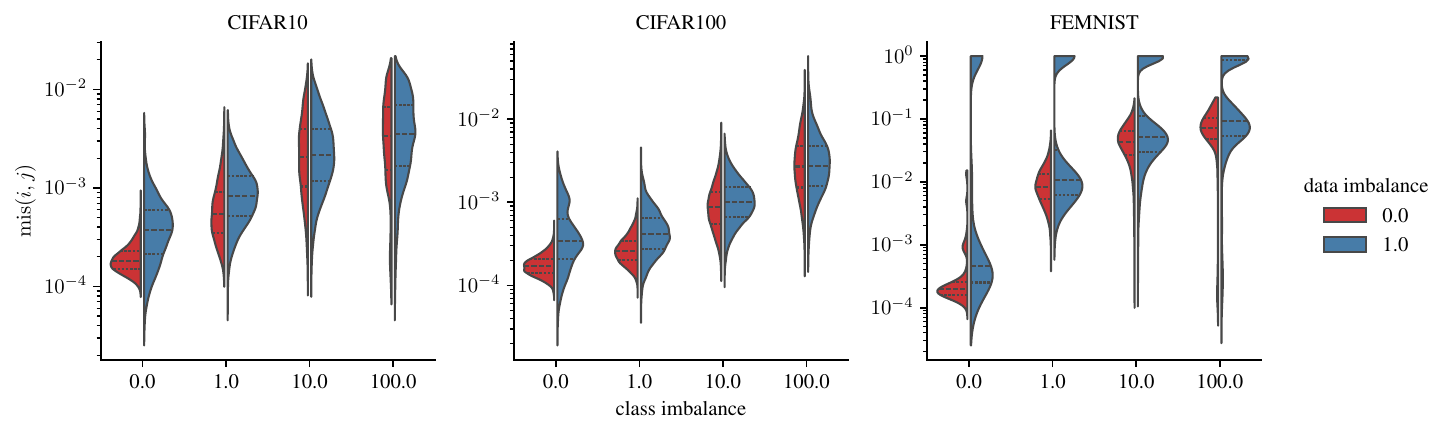}
    \caption{Client misalignment (see \ref{eq:client_misalignment}) for the three datasets under examination. We compute the misalignment distributions for each class imbalance and data imbalance extent. When increasing the imbalance, the distributions move toward a higher degree of misalignment.}
    \label{fig:client-misalignments}
\end{figure}
We apply the definition of client misalignment \ref{eq:client_misalignment} on the chosen federated datasets, generated using the approach depicted in \ref{sec:federated_dataset_generation}. In this regard, following the idea presented in \ref{sec:graph_based_model}, the allocated subsets are interpreted as the clients that constitute the federated network, whereas each one corresponds to a node of the graph and exchanges a message $\vb{m}_i$ corresponding to the first principal component of its dataset (uncentered and unscaled PCA). Therefore, repeating the experiment for multiple choices of imbalance, we compute the misalignment between each pair of clients inside each generated network. We decide to utilize the logarithmic scale to represent all client misalignment distributions. Adopting such a scale allows for better discrimination between different scenarios where the imbalance parameters change. Indeed, by augmenting the class imbalance, we notably affect the misalignment distribution. Specifically, the mean of the distribution grows as we exponentially increase the class imbalance. Likewise, with fixed class imbalance, higher data imbalance negatively affects both the skewness and kurtosis of the misalignment distributions, since these present shifted peaks and longer tails. To conclude, such results suggest that client misalignment $\mathrm{mis}(i, j)$ could be a representative measure of the statistical dissimilarity between different clients, and, therefore, can be reasonably leveraged to interpret federated networks as similarity graphs.

\paragraph{More insights from the graph-laplacian $\vb{L}$}
\begin{figure}[h]
    \centering
    \includegraphics[width=1 \textwidth]{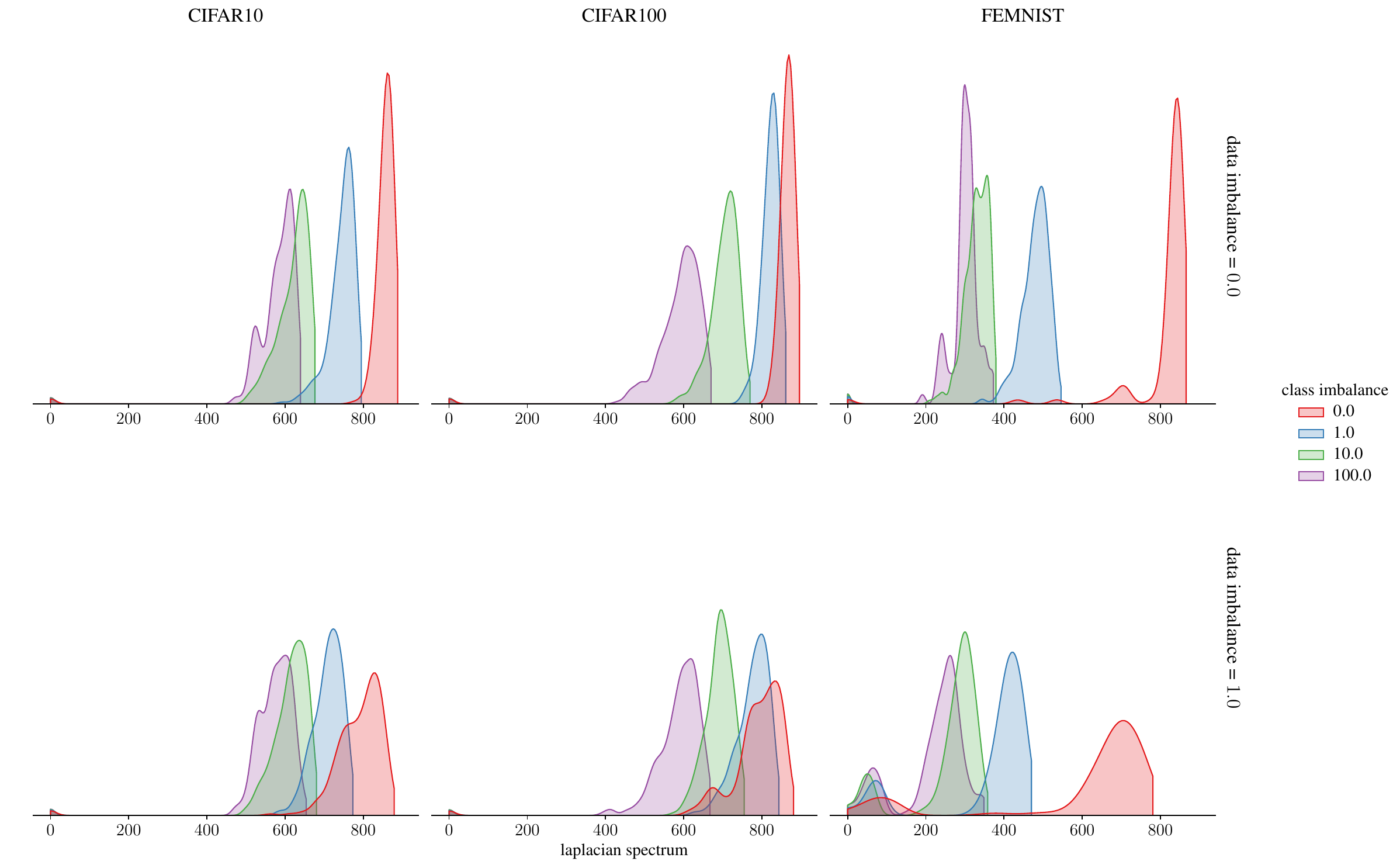}
    \caption{The distribution of the eigenvalues of the laplacian matrix $\vb{L}$ that is associated with the graph representation of the federated network. We inspect the distribution under different conditions of imbalance. The spectrum is highly skewed toward lower eigenvalues as the heterogeneity increases.}
    \label{fig:laplacian-spectrum-distribution}
\end{figure}
Given the adjacency matrix $\vb{A}$ of the graph representation of a federated network from Definition \ref{eq:federated_network_adjacency_matrix}, we construct its laplacian matrix $\vb{L} = \textrm{diag}(\vb{A}\vb{1}) - \vb{A}$. 
On this subject, we purposefully compare the spectra of multiple federated networks' laplacians, conceived with different rates of imbalance. We accurately portray the results of our simulation in Figure \ref{fig:laplacian-spectrum-distribution}. Noticeably, the behavior of the spectrum significantly depends on the degree of class imbalance. Indeed, as we increase it, the spectrum is moved to lower eigenvalues. Additionally, the spectra of highly homogeneous networks ($\text{class imbalance} = 0.0$) are more affected by an increment in data imbalance than highly heterogeneous networks ($\text{class imbalance} \gg 0.0$), and this fact becomes evident through a reduction in the kurtosis of the eigenvalues distribution.
\begin{table}[b]
    \caption{Running times of experiments from Table \ref{tab:algorithms_experiments_comparison}. We report execution times as (mean $\pm$ std) across the balanced and imbalanced cases for each configuration of parameters.}
    \renewcommand\arraystretch{1}
    \newcommand{\NA}{---}
    \newcommand{\MISSING}{\textcolor{red}{TODO}}
    \setlength{\tabcolsep}{5pt}
    \centering
    \small
    \begin{tabular}[t]{lccc}
        \toprule
        & \textbf{Convex} & \multicolumn{2}{c}{\textbf{Running time} (mean $\pm$ std)} \\
        \cmidrule{3-4}
        & & \textbf{CIFAR10} & \textbf{FEMNIST} \\
        \midrule
        \multirow{2}{*}{\textsc{FedAvg}} & Yes & 08h58m $\pm$ 00h02m & 21h34m $\pm$ 00h21m \\
        & No & 09h45m $\pm$ 00h00m & 23h03m $\pm$ 00h04m \\
        \midrule
        \multirow{2}{*}{$\textrm{Ours}_{\{0.9\}}$} & Yes & 09h08m $\pm$ 00h02m & 22h15m $\pm$ 00h14m \\
        & No & 11h11m $\pm$ 00h08m & 23h58m $\pm$ 00h01m \\
        \cmidrule{2-4}
        \multirow{2}{*}{$\textrm{Ours}_{\{0.7\}}$} & Yes & 09h06m $\pm$ 00h00m & 22h22m $\pm$ 00h02m \\
        & No & 11h09m $\pm$ 00h03m & 24h46m $\pm$ 00h11m \\
        \cmidrule{2-4}
        \multirow{2}{*}{$\textrm{Ours}_{\{0.5\}}$} & Yes & 09h09m $\pm$ 00h00m & 22h02m $\pm$ 00h03m \\
        & No & 11h06m $\pm$ 00h03m & 24h10m $\pm$ 00h15m \\
        \bottomrule
    \end{tabular}
    \label{tab:algorithms_experiments_running_times}
\end{table}

\subsection{Experiments on the convergence and robustness of our algorithm}

This part contains the experiments used to evaluate the performance of our approach. In the first place, we aim to assess its behavior across different extents of balancedness and convexity guarantees in comparison with the baseline algorithms \textsc{FedAvg}. In addition, due to the theoretical limitation of our method that we explain in section \ref{sec:analysis_our_algorithm}, we are also interested in understanding its robustness to variations in the number of local optimization epochs $E$ and the norm of the stochastic gradient $G$.

\paragraph{The aggregation scheme}

To enable comparisons across the simulations, we define and employ the \textsc{Adjacency} scheme to average the updated weights from clients. To this end, each client $i$ initially shares the first principal component related to its local dataset $\mathcal{D}_i$ as message $\vb{m}_i$. We define similarity weights and aggregation weights from shared messages accordingly to Algorithm \ref{alg:our_algorithm}.

\paragraph{Running times of our main simulations}

We present the running times on the reference hardware (see \ref{sec:implementation}) of our main simulations from Table \ref{tab:algorithms_experiments_comparison} when tuning the parameters drawn from \ref{tab:federated_experiments_parameter_grid}. For each combination of parameters identifying a row of Table \ref{tab:algorithms_experiments_running_times}, we compute the mean and the standard deviation of running times between the balanced and imbalanced experiments.

\paragraph{The efficacy of our algorithm persists as $E$ changes}
\begin{figure}[tbh]
    \centering
    \includegraphics[width=1 \textwidth]{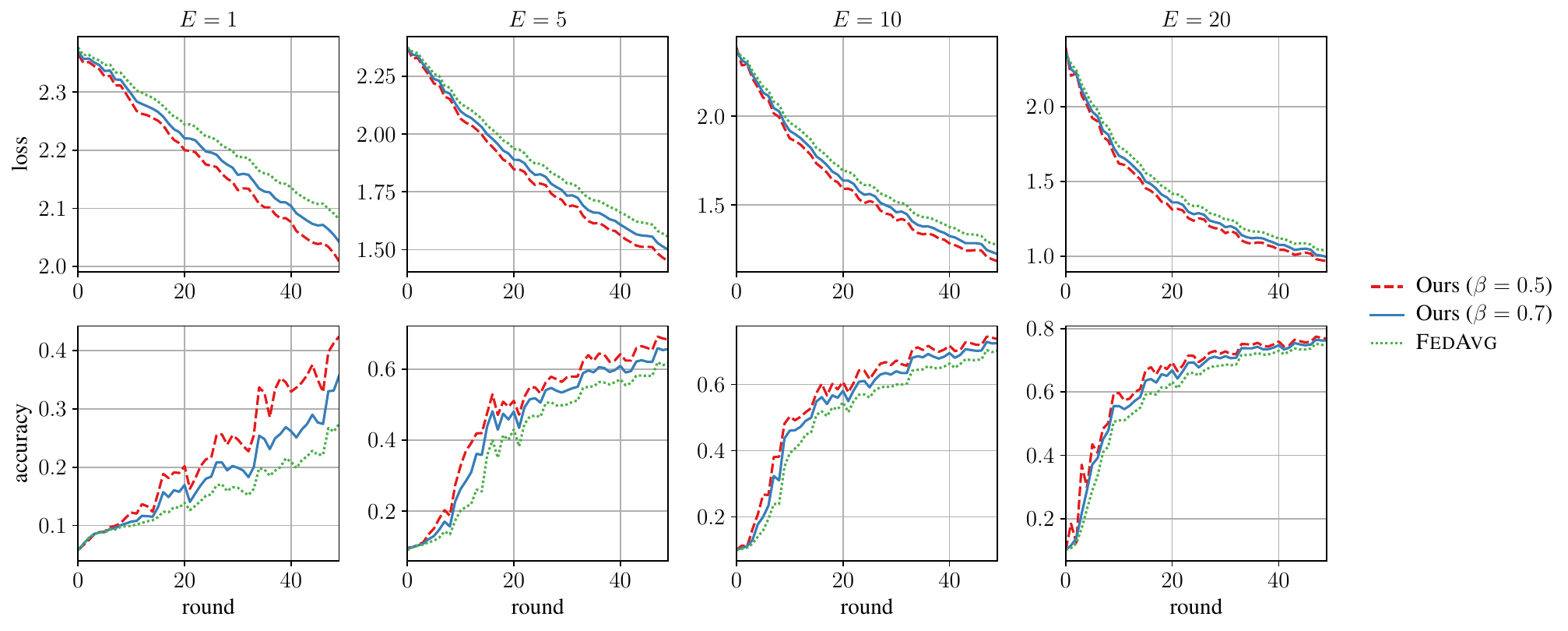}
    \caption{We run \textsc{FedAvg} and our algorithm for multiple values of $\beta$ on the imbalanced FEMNIST, specifically on unseen clients. We vary the number of local epochs $E$ using our strongly convex loss. }
    \label{fig:femnist-imbalance-convex-vary-epochs-extended}
\end{figure}
\begin{figure}[tbh]
    \centering
    \includegraphics[width=1 \textwidth]{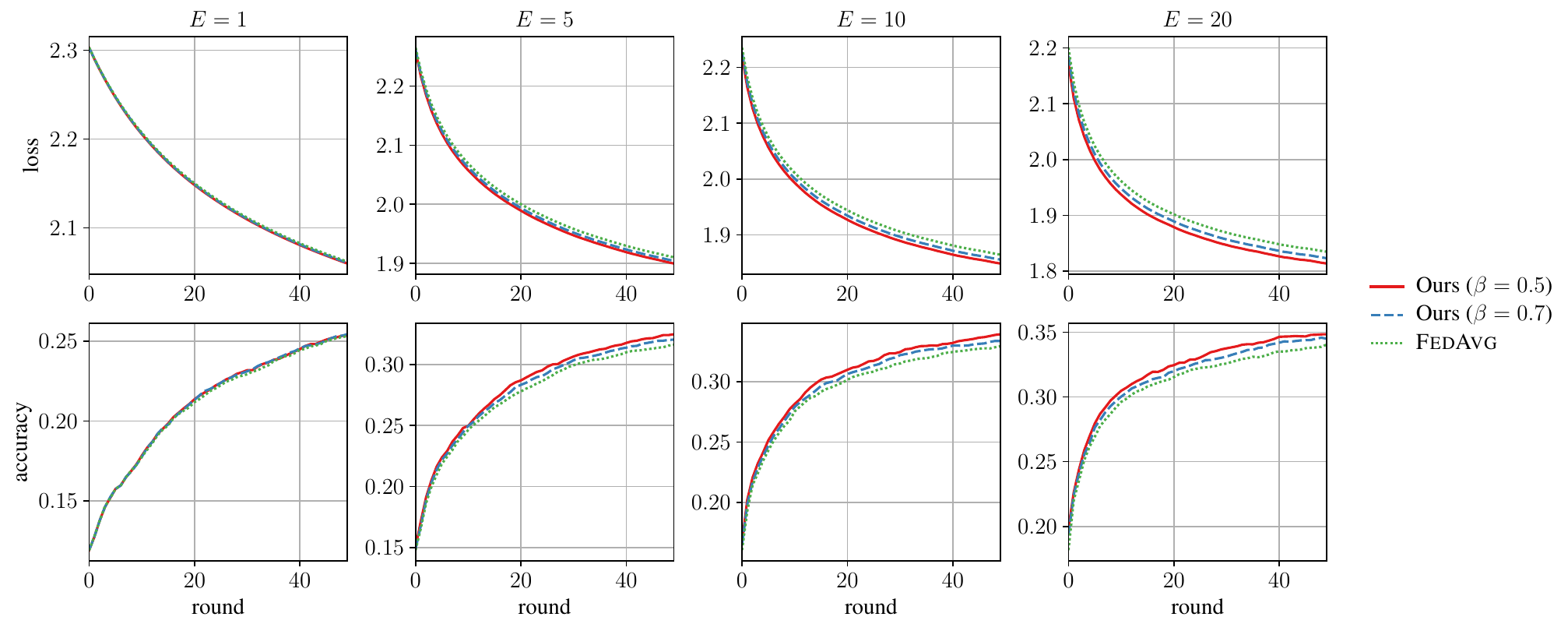}
    \caption{We run the same strongly convex simulation of Figure \ref{fig:femnist-varying-E-simulation} on the imbalanced CIFAR10. }
    \label{fig:cifar10-imbalance-convex-vary-epochs-extended}
\end{figure}
In this set of experiments, we test our algorithm and \textsc{FedAvg} for $T = 50$ rounds on unseen data in imbalanced scenarios when using the multinomial logistic regression as strongly convex loss. We again draw the values of the parameters from Table \ref{tab:federated_experiments_parameter_grid}. We only vary the number of epochs $E \in \{ \, 1, 5, 10, 20 \, \}$ devoted to local optimization on each client, and our algorithm's parameter $\beta \in \{ \, 0.5, 0.7 \, \}$. Differently from our theoretical results where $\gamma \propto 1 / E$, we keep the same step size $\gamma = 10^{-3}$ across all experiments, regardless of the value of $E$. Figure \ref{fig:femnist-imbalance-convex-vary-epochs-extended} and \ref{fig:cifar10-imbalance-convex-vary-epochs-extended} show that our algorithm consistently improves over \textsc{FedAvg} for multiple configurations of $E$.

\paragraph{Studying the effect of gradient clipping on our algorithm}

We now consider studying the performance of our algorithm in comparison with the baseline \textsc{FedAvg} on unseen clients (testing dataset) as we vary the maximum allowed norm $G \in \{ \, 1.0, 10.0, \infty \textrm{ (unbounded)} \, \}$ of stochastic gradients. We accomplish this task by applying the gradient clipping operation implemented by PyTorch. For this simulation, we consider the strongly convex loss, namely the multinomial logistic regression, and we pick $\beta \in \{ \, 0.5, 0.7 \, \}$. We again run the experiments for $T = 50$ rounds on the imbalanced datasets. We choose all the other parameters from Table \ref{tab:federated_experiments_parameter_grid}. From Figure \ref{fig:femnist-imbalance-convex-vary-gradient-norm-extended} and \ref{fig:cifar10-imbalance-convex-vary-gradient-norm-extended}, we discern that our algorithm performs comparably to the baseline if not better.
\begin{figure}[tbh]
    \centering
    \includegraphics[width=0.8 \textwidth]{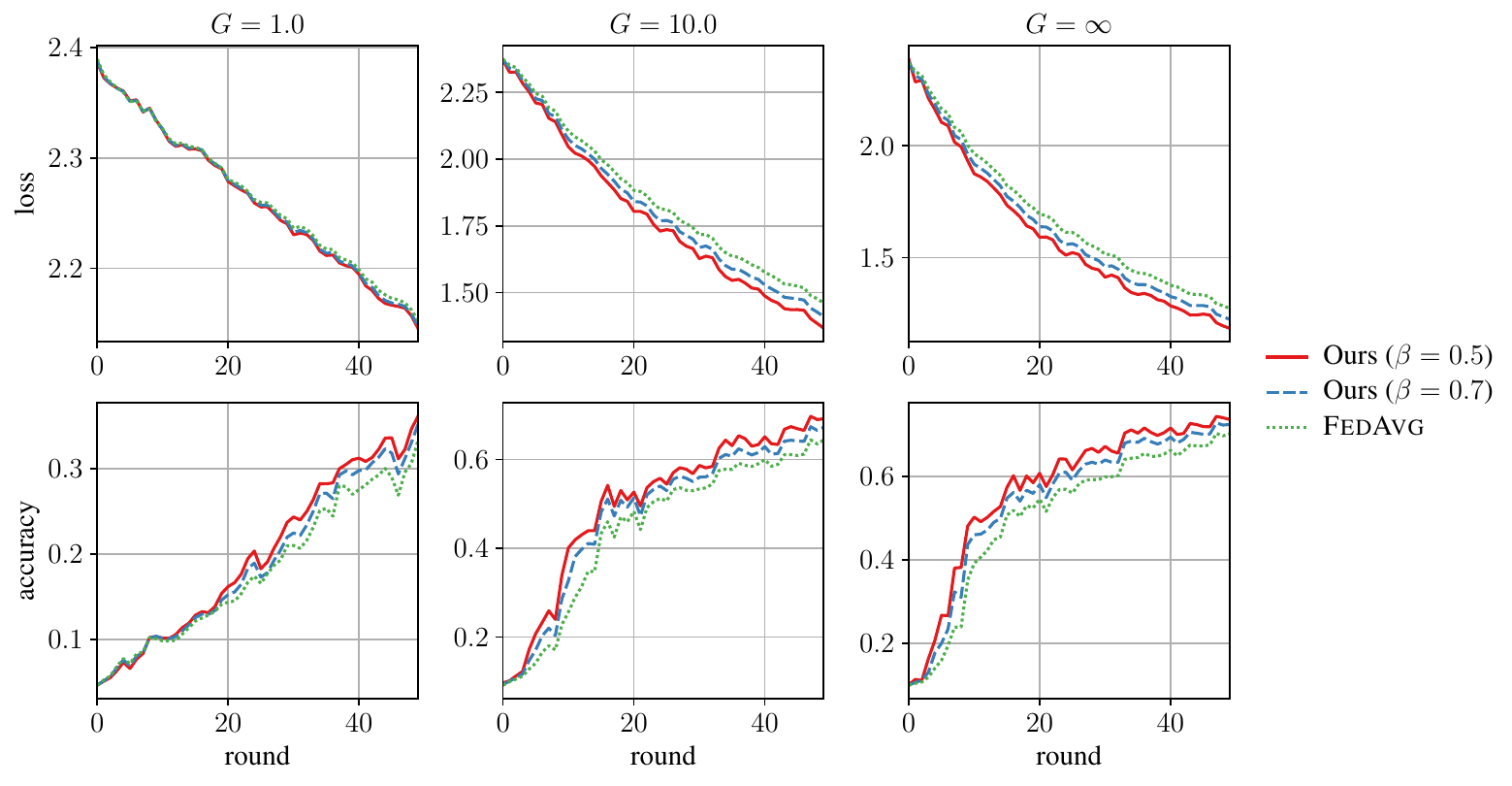}
    \caption{In this strongly convex simulation on the imbalanced FEMNIST, we observe the behavior of the baseline \textsc{FedAvg} and our algorithm when gradient clipping is applied. Specifically, we set the maximum norm of the stochastic gradient as $G$ in each depicted experiment. Having $G = 1.0$ significantly slows down convergence.}
    \label{fig:femnist-imbalance-convex-vary-gradient-norm-extended}
\end{figure}
\begin{figure}[tbh]
    \centering
    \includegraphics[width=0.8 \textwidth]{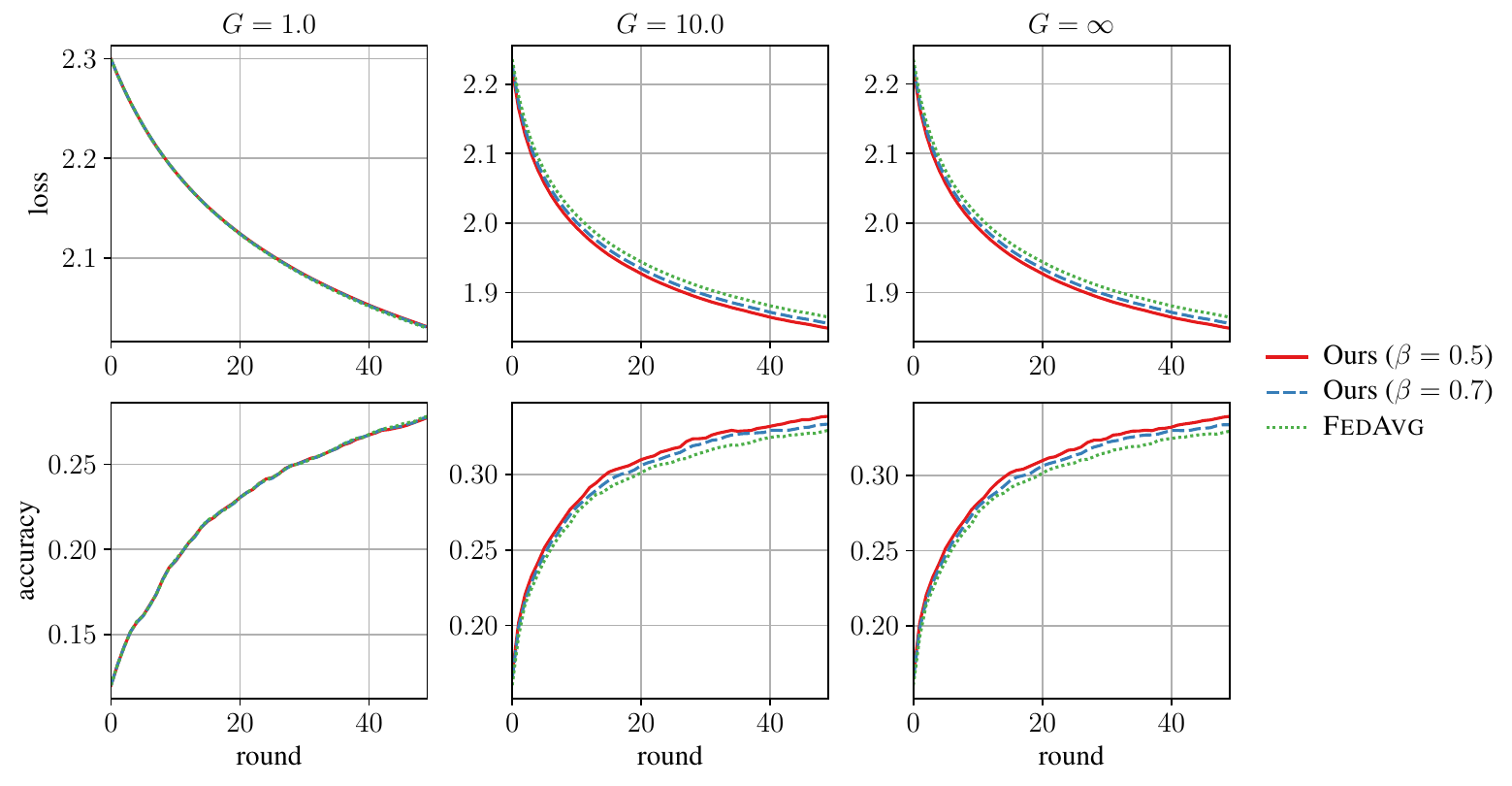}
    \caption{As in figure \ref{fig:femnist-imbalance-convex-vary-gradient-norm-extended}, we run the same strongly convex simulation on the imbalanced CIFAR10. We observe that the consequence of clipping the gradient is practically imperceptible when $G \ge 10$.}
    \label{fig:cifar10-imbalance-convex-vary-gradient-norm-extended}
\end{figure}

\paragraph{How the step size impacts the stability of convergence}

We now assess how the convergence our algorithm for $\beta \in \{ \, 0.5, 0.7, 0.9 \, \}$ is affected when employing a different step size $\gamma \in \{ \, 10^{-3}, 10^{-2}, 10^{-1} \, \}$ in the strongly convex case. Once more, we utilize \textsc{FedAvg} as our baseline, and we run these simulations for $T = 50$ rounds on the imbalanced CIFAR10 and FEMNIST datasets. All other parameters are fixed and chosen from Table \ref{tab:federated_experiments_parameter_grid}. In both Figure \ref{fig:femnist-imbalance-convex-vary-step-size} and \ref{fig:cifar10-imbalance-convex-vary-step-size}, we observe that the combination of the step size $\gamma$ and the perturbation parameter $\beta$ is crucial to guarantee a stable convergence for our algorithm. Precisely, and in line with our theoretical result from Theorem \ref{theorem:convergence_our_algorithm_strongly_convex}, a large step size $\gamma$ and small $\beta$ (high perturbation) imply evident spikes in the (testing) loss and accuracy curves. However, when $\beta$ is sufficiently large (minimal perturbation) and the magnitude of the step size is limited enough, our algorithm visibly performs better than \textsc{FedAvg}.
\begin{figure}[tbh]
    \centering
    \includegraphics[width=0.8 \textwidth]{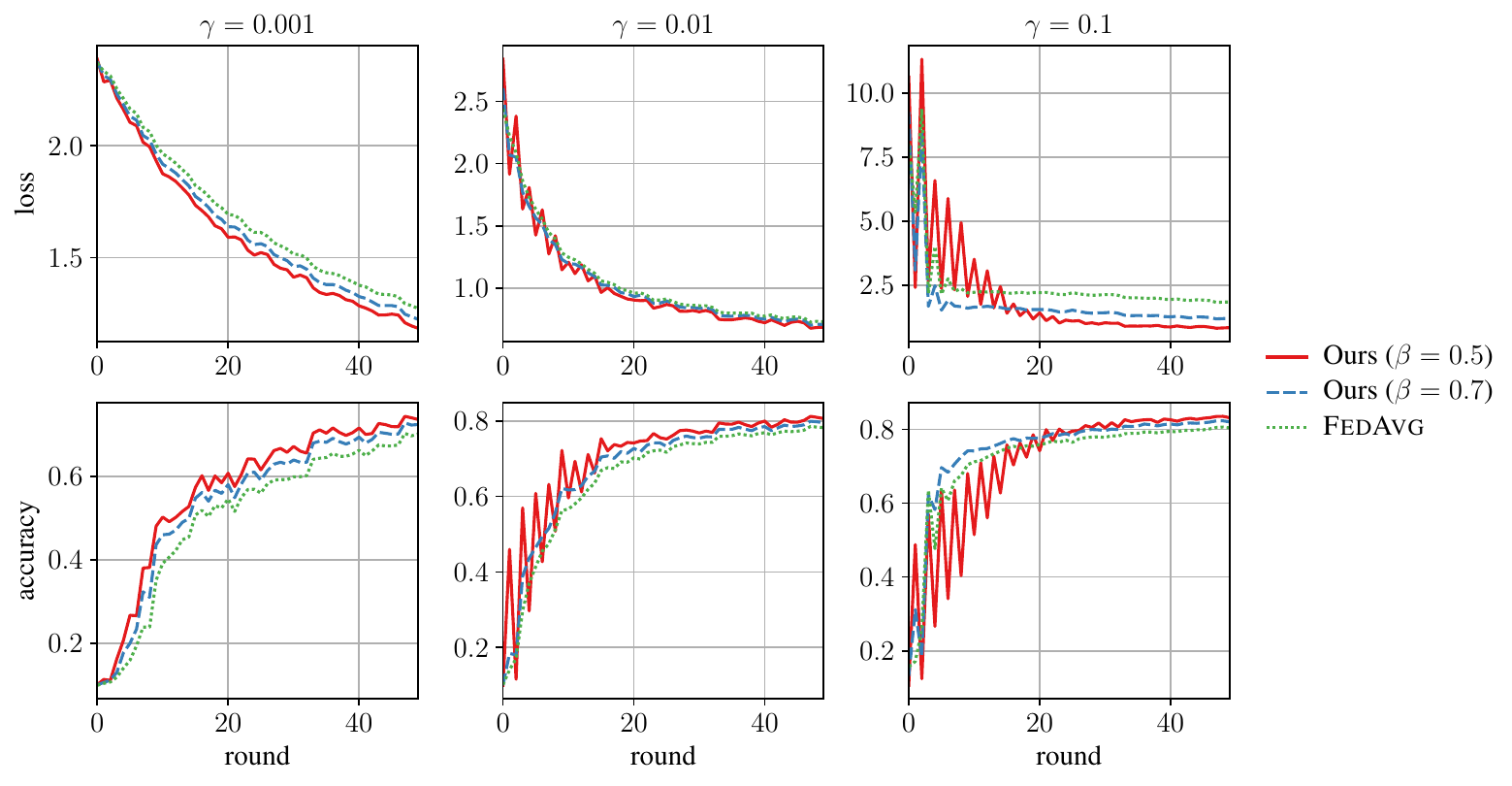}
    \caption{We vary the step size when running our algorithm with a strongly convex loss, namely the multinomial logistic regression, on the imbalanced FEMNIST dataset.}
    \label{fig:femnist-imbalance-convex-vary-step-size}
\end{figure}
\begin{figure}[tbh]
    \centering
    \includegraphics[width=0.8 \textwidth]{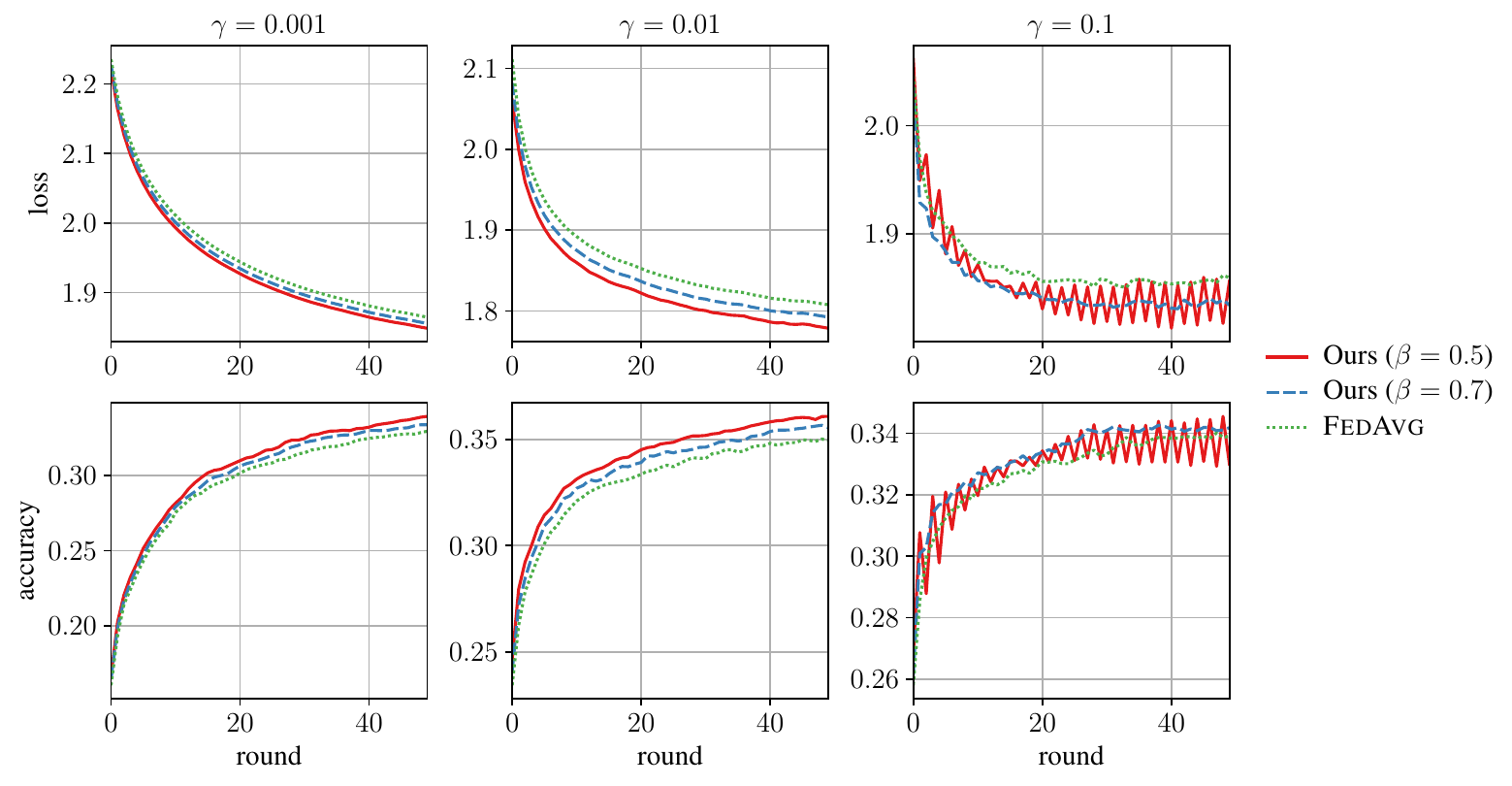}
    \caption{On the imbalanced CIFAR10 dataset, we run the same set of experiments of Figure \ref{fig:femnist-imbalance-convex-vary-step-size}.}
    \label{fig:cifar10-imbalance-convex-vary-step-size}
\end{figure}

\section{Further discussion on the scope and utility of our study}
\label{sec:motivation}

As pointed out by \cite{fed_optimization_guide}, an intrinsic problem of federated learning is represented by the update operation each client independently carries out for multiple local steps. \cite{fedaveraging} introduced this scheme as \textsc{FedAvg}, where each client undertakes more than one stochastic gradient update to reduce the synchronization steps with the server and thus the communication cost. However, when the server aggregates the computed updates from the clients, the whole procedure results in an inexact gradient descent in terms of the average iterate $\overbar{\vb{w}}_{t, k}$ since the clients evaluate local gradients in their respective local iterates in place of $\overbar{\vb{w}}_{t, k}$. We discuss in Section \ref{sec:related_works} how previous works attack this issue and the related client drift phenomenon. 

Nevertheless, our novel approach is different yet elementary regarding how it addresses the previously mentioned problem and corrects the optimization procedure performed on each client's device. 
To better mimic the classic and centralized stochastic gradient descent, we believe it is worth finding a locally perturbed iterate $\widetilde{\vb{w}}_{t, k}^i$ closer than $\vb{w}_{t, k}^i$ to the global average $\overbar{\vb{w}}_{t, k}$. 
Specifically, we realign locally computed gradients through calculated and "personalized" perturbations that carry information about other clients based on statistical affinity. We introduce our framework in detail in Section \ref{sec:our_algorithm_framework_definition}.

Both Theorem \ref{theorem:convergence_our_algorithm_strongly_convex} and \ref{theorem:convergence_our_algorithm_nonconvex} highlight how much we have to pay in terms of expected convergence error when raising the extent of perturbation (by reducing $\beta$). In other words, a higher perturbation implies injecting more mutual similarity information $\vb{u}_t^i$ into the update, to the detriment of $\vb{w}_{t, k}^i$. We achieve this by defining the perturbed iterate as the weighted mean $\beta \vb{w}_{t, k}^i + (1 - \beta)\vb{u}_t^i$.
Although our theoretical results agree that lowering $\beta$ destabilizes the convergence to optimality, the empirical evidence shows that the proposed scheme consistently outperforms the baseline \textsc{FedAvg} across multiple scenarios when making an appropriate choice of $\beta$ (sufficiently large) and $\gamma$ (sufficiently small).

However, it becomes clear that our method has a more general and simple structure that can be useful in developing other federated algorithms. In this respect, there are no limitations on how the perturbed iterate $\widetilde{\vb{w}}_{t, k}^i$ can be defined, and we present a possible and specific way to do so. We hope such a consideration opens up unexplored possibilities for devising algorithms where clients implement more informed optimization steps while complying with the communication and privacy constraints imposed by federated learning.

\end{document}